\documentclass[11pt]{article}

\usepackage{bbm}
\usepackage{geometry}
\usepackage{setspace}
\usepackage{amsmath, amssymb, amsfonts, bm, mathtools, mathrsfs}
\usepackage{amsthm}
\usepackage[dvipsnames]{xcolor}
\definecolor{darkblue}{rgb}{0,0,.5}
\usepackage{graphicx}
\usepackage{subfigure}

\usepackage{mathabx}
\usepackage[utf8]{inputenc} % allow utf-8 input
\usepackage[T1]{fontenc}    % use 8-bit T1 fonts
\usepackage{hyperref}       % hyperlinks
\usepackage{url}            % simple URL typesetting
\usepackage{booktabs}       % professional-quality tables
\usepackage{amsfonts}       % blackboard math symbols
\usepackage{nicefrac}       % compact symbols for 1/2, etc.
\usepackage{microtype}      % microtypography

\allowdisplaybreaks

\usepackage{float}
\usepackage{multirow}
\usepackage{footnote}
\usepackage{dsfont}

\usepackage{algorithm}
\usepackage{algpseudocode}
\usepackage{nicefrac}
\usepackage{tikz}
\usepackage{overpic}

\usepackage{dsfont}
\usepackage{hyperref}
\usepackage[capitalize]{cleveref}
\usepackage{crossreftools}

\makeatletter
\newcommand*{\rom}[1]{\expandafter\@slowromancap\romannumeral #1@}
\makeatother

\newcommand{\ind}{\mathds{1}}
\newcommand{\var}{\mathsf{Var}}

\newcommand{\brc}[1]{\left\{{#1}\right\}}
\newcommand{\prn}[1]{\left({#1}\right)} % parentheses
\newcommand{\brk}[1]{\left[{#1}\right]} % bracket
\newcommand{\norm}[1]{\left\|{#1}\right\|} % norm
\newcommand{\abs}[1]{\left|{#1}\right|} % norm
\newcommand{\what}[1]{\widehat{#1}}
\newcommand{\wtilde}[1]{\widetilde{#1}}

\newcommand{\<}{\langle} % Angle brackets
\renewcommand{\>}{\rangle}

\DeclareMathOperator*{\argmin}{argmin}

\def\gA{{\mathcal{A}}}
\def\gB{{\mathcal{B}}}
\def\gC{{\mathcal{C}}}

\def\gF{{\mathcal{F}}}

\def\gH{{\mathcal{H}}}
\def\gI{{\bm{\mathcal{I}}}}

\def\gL{{\mathcal{L}}}
\def\gM{{\mathcal{M}}}

\def\gP{{\mathcal{P}}}

\def\gS{{\mathcal{S}}}
\def\gT{{\bm{\mathcal{T}}}}
\def\gU{{\bm{\mathcal{U}}}}
\def\gV{{\mathcal{V}}}

\def\sT{{\mathscr{T}}}

\def\sP{{\mathscr{P}}}
\def\sM{{\mathscr{M}}}

\def\sX{{\mathscr{X}}}
\def\sL{{\mathscr{L}}}

\def\tmix{t_{\operatorname{mix}}}
\def\tframe{t_{\operatorname{frame}}}
\def\tepoch{t_{\operatorname{epoch}}}

\def\muframe{\mu_{\operatorname{frame}}}

\def\bA{{\bm{A}}}
\def\bW{{\bm{W}}}
\def\bI{{\bm{I}}}

\def\bU{{\bm{U}}}
\def\bP{{\bm{P}}}
\def\bLamb{{\bm{\Lambda}}}

\def\bDelta{{\bm{\Delta}}}

\newcommand{\inner}[2]{\left\langle #1, #2 \right\rangle}

\def\RB{{\mathbb R}}
\def\EB{{\mathbb E}}

\def\PB{{\mathbb P}}

\def\NB{{\mathbb N}}

\def\QB{{\mathbb Q}}
\def\st{\mathsf{s.t.}}

\def\ie{{\em i.e.\/}}

\newcommand{\TV}{\textup{TV}}

\makeatletter
\long\def\@makecaption#1#2{
  \vskip 0.8ex
  \setbox\@tempboxa\hbox{\small {\bf #1:} #2}
  \parindent 1.5em  %% How can we use the global value of this???
  \dimen0=\hsize
  \advance\dimen0 by -3em
  \ifdim \wd\@tempboxa >\dimen0
  \hbox to \hsize{
    \parindent 0em
    \hfil 
    \parbox{\dimen0}{\def\baselinestretch{0.96}\small
      {\bf #1.} #2
    } 
    \hfil}
  \else \hbox to \hsize{\hfil \box\@tempboxa \hfil}
  \fi
}
\makeatother

\newcommand{\TD}{{\texttt{TD}}}
\newcommand{\NTD}{{\texttt{NTD}}}
\newcommand{\CTD}{{\texttt{CTD}}}
\newcommand{\QTD}{{\texttt{QTD}}}
\newcommand{\VRNTD}{{\texttt{VR-NTD}}}
\newcommand{\VRCTD}{{\texttt{VR-CTD}}}

\usepackage{enumerate}
\usepackage{natbib}

\newtheorem{claim}{Claim}[section]
\newtheorem{lemma}[claim]{Lemma}

\newtheorem{theorem}{Theorem}[section]

\newtheorem{proposition}{Proposition}[section]
\newtheorem{corollary}{Corollary}[section]

\newcounter{remark}
\newcounter{rtaskno}

\usepackage{xr}
\title{Statistical Efficiency of Distributional Temporal Difference  Learning and  Freedman's Inequality in Hilbert Spaces}
\author{
Yang Peng,~\thanks{Yang Peng is with  the School of Mathematical Sciences, Peking University; email: \texttt{pengyang@pku.edu.cn}.} \and
Liangyu Zhang,~\thanks{Liangyu Zhang is with the School of Statistics and Data Science, Shanghai University of Finance and Economics; email: \texttt{zhangliangyu@sufe.edu.cn}.} \and
Zhihua Zhang\thanks{Zhihua Zhang is with the School of Mathematical Sciences, Peking University; email: \texttt{zhzhang@math.pku.edu.cn}.}
}

\begin{document}
\maketitle

\begin{abstract}%
Distributional reinforcement learning (DRL) has achieved empirical success in various domains.
One core task in DRL is distributional policy evaluation, which involves estimating the return distribution $\eta^\pi$ for a given policy $\pi$.
Distributional temporal difference learning has been accordingly proposed, which extends the classic temporal difference learning (TD) in RL.
In this paper, we focus on the non-asymptotic statistical rates of distributional TD.
To facilitate theoretical analysis, we propose non-parametric distributional temporal difference learning (\NTD).
For a $\gamma$-discounted infinite-horizon tabular Markov decision process,
we show that for {\NTD} with a generative model, we need $\wtilde O\prn{\varepsilon^{-2}\mu_{\min}^{-1}(1-\gamma)^{-3}}$ interactions with the environment to achieve an $\varepsilon$-optimal estimator with high probability, when the estimation error is measured by the $1$-Wasserstein distance.
Here $\mu_{\min}$ is a covering coefficient.
This sample complexity bound is minimax optimal up to logarithmic factors.
In addition, we revisit categorical distributional temporal difference learning (\CTD), showing that the same non-asymptotic convergence bounds hold for {\CTD} in the case of the $1$-Wasserstein distance.

We also extend our analysis to the more general setting where the data generating process is Markovian.
In the Markovian setting, we propose variance-reduced variants of {\NTD} and {\CTD}, and show that both can achieve a $\widetilde{O}\left({\varepsilon^{-2} \mu_{\pi,\min}^{-1}(1-\gamma)^{-3}}+{\tmix}{\mu_{\pi,\min}^{-1}(1-\gamma)^{-1}}\right)$ sample complexity bounds in the case of the $1$-Wasserstein distance, which matches the state-of-the-art statistical results for classic policy evaluation.
Here $\mu_{\pi,\min}$ is some covering coefficient and $\tmix$ is the mixing time of the underlying Markov chain.

To achieve the sharp statistical rates, we establish a novel Freedman's inequality in Hilbert spaces.
This new Freedman's inequality would be of independent interest for statistical analysis of various infinite-dimensional online learning problems.
\end{abstract}

\tableofcontents

% \tableofcontents
\section{Introduction}
% -*- Mode: latex -*- %
In high-stake applications of reinforcement learning (RL), such as healthcare \citep{lavori2004dynamic,bock2022superhuman} and finance \citep{ghysels2005there}, only considering the mean of returns is insufficient. 
It is necessary to take risks and uncertainties into consideration. 
Distributional reinforcement learning (DRL) \citep{morimura2010nonparametric,bellemare2017distributional,bdr2022} addresses such issues by modeling the complete distribution of returns instead of their expectations.
One of the most fundamental tasks in DRL is to estimate the return distribution $\eta^\pi$ for a given policy $\pi$, which is referred to as distributional policy evaluation.

In analogy to classic temporal-difference learning \citep{sutton1988learning}, distributional temporal-difference (TD) learning is one of the most widely-used methodologies for solving the distributional policy evaluation problem. 
Despite its empirical success, theoretical understanding is relatively lacking.
Existing theoretical works have primarily focused on asymptotic behaviors of distributional TD. 
In particular, \citet{rowland2018analysis} and \citet{rowland2023analysis} showed the asymptotic convergence of two variants of distributional TD, namely categorical temporal difference learning ({\CTD}) \citep{bellemare2017distributional} and quantile temporal difference learning ({\QTD}) \citep{dabney2018distributional}, respectively. 
A natural question arises: 
\emph{can we precisely depict the statistical efficiency of distributional TD with non-asymptotic rates similar to the classic TD algorithm \citep{li2024q}}?

We would answer this question by exploring near-minimax-optimal non-asymptotic sample complexity bounds for distributional TD.
A large body of existing non-asymptotic analyses on classic TD learning \citep{shah2018q,chen2024lyapunov,li2024q} are based on the stochastic approximation (SA) ~\citep{robbins1951stochastic,sdr2021} approach.
In the current paper, we will also employ the SA approach by viewing distributional TD learning as an instance of infinite-dimensional SA algorithms.
As a key tool for deriving sharp statistical rates for SA algorithms, Freedman's inequality \citep{freedman1975tail} can be viewed as a Bernstein-type inequality for martingales.
It has various generalizations, such as matrix Freedman's inequality \citep{tropp2011freedman}. 
However, to the best of our knowledge, an applicable version of infinite-dimensional Freedman's inequality remains open.
Thus, we are encouraged to develop a Freedman's inequality in Hilbert spaces, with which we obtain near-optimal non-asymptotic statistical rates for distributional TD.
We would believe that the new Freedman inequality has independent merits and can be useful for the statistical analysis of various online learning problems \citep{kushner1985stochastic, dieuleveut2015non, martin2019stochastic}.

\subsection{Contributions}

Our contribution includes three folds.   
First, we affirmatively answer the above question about the statistical efficiency of distributional TD in the generative model setting \citep{kakade2003sample, kearns2002sparse} (also called the synchronous setting).
Specifically, we present the non-parametric distributional TD algorithm ({\NTD}) in Section~\ref{Section_dtd}, which is not practical but aids theoretical understanding.
We show that, in the generative model setting, for {\NTD}, $\wtilde O\prn{\varepsilon^{-2}\mu_{\min}^{-1}(1-\gamma)^{-3}}$
\footnote{In this paper, the notation $f(\cdot)=\tilde{O}\prn{g(\cdot)}$ ($f(\cdot)=\tilde{\Omega}\prn{g(\cdot)}$) means that $f(\cdot)$ is order-wise no larger (smaller) than $g(\cdot)$, ignoring logarithmic factors $\textup{poly}(\log\abs{\gS},\log\abs{\gA},-\log\prn{1-\gamma},\log(\tmix),\log\prn{1/{\mu_{\min}}},\log(1/{\varepsilon}),\log(1/\delta)$, as $\abs{\gS},\abs{\gA},\prn{1-\gamma}^{-1},\tmix,1/{\mu_{\min}},1/\varepsilon,1/{\delta}\to\infty$.} interactions with the environments are sufficient to yield an estimator $\hat\eta^\pi$ such that the $1$-Wasserstein metric between $\hat\eta^\pi$ and $\eta^\pi$ is less than $\varepsilon$ with high probability (Theorem~\ref{thm:samplecomplex_syn_w1}).
Here $\mu_{\min}$ is a parameter that describes how well the generative model can cover the whole state space.
This bound is minimax optimal (Theorem~\ref{thm:minimax_lower_bound_w1}) if we neglect
all logarithmic terms.
The same minimax optimal sample complexity bounds also hold for the more practical {\CTD} algorithm.

Second, we answer the above question about the statistical efficiency of distributional TD affirmatively in the Markovian setting \citep{li2021sample,mou2022optimal2,li2023online, chen2024lyapunov} (also called the asynchronous setting).
To achieve minimax optimal sample complexity bounds, we first devise variance-reduced variations of {\NTD} and {\CTD}, which we call {\VRNTD} and {\VRCTD}.
We show a $\widetilde{O}\left({\varepsilon^{-2} \mu_{\pi,\min}^{-1}(1-\gamma)^{-3}}+{\tmix}{\mu_{\pi,\min}^{-1}(1-\gamma)^{-1}}\right)$ sample complexity bound of both {\VRNTD} and {\VRCTD}, which matches the state-of-the-art sample complexity bounds in TD for classic policy evaluation \citep{li2021sample}.
Here $\tmix$ represents the mixing time of Markov chain induced by the underlying MDP and the policy to be evaluated and ${\mu_{\pi,\min}}$ represents how well the stationary distribution of the induced Markov chain can cover the state space.

Third, we propose a novel Freedman's inequality in Hilbert spaces (Theorem~\ref{thm:freedman_ineq}) and utilize it for statistical analysis of distributional TD.
We believe our new Freedman's inequality is of broad interest beyond this work and can be applied to statistical analysis of many infinite-dimensional online learning problems.

A shortened version of this manuscript was  presented orally at a 
conference \citep{anonymous}.
The conference version includes statistical analysis of distributional TD in the synchronous setting and a preliminary version of Freedman's inequality in Hilbert spaces.
In the current version of this paper, we extend our theoretical analysis in the synchronous setting to the generative model setting and provide a in-depth discussion of the results.
We also analyze statistical efficiency of distributional TD in the more challenging Markovian setting.
We propose variance-reduced variants of existing algorithms and manage to show sharp statistical rates that match the state-of-the-art theoretical results in classic policy evaluations. %(which is easier to deal with).
We also present a more well-rounded discussion about the motivation and implications of our novel Freedman's inequality.
\subsection{Related Work}\label{subsection_related_work}
\paragraph{Theory of Distributional Policy Evaluation}
Recently, there has been an emergence of work focusing on finite-sample/iteration results of the distributional policy evaluations.
\citet{zhang2023estimation} proposed to solve distributional policy evaluation by the model-based approach in the synchronous setting, and derived corresponding sample complexity bounds, namely $\wtilde O\prn{{\varepsilon^{-2p}(1-\gamma)^{-2p-2}}}$ in the case of the $p$-Wasserstein metric, and $\wtilde O\prn{{\varepsilon^{-2}(1-\gamma)^{-4}}}$ in the case of the KS metric and total variation metric under different mild conditions.
\citet{rowland2024near} proposed direct categorical fixed-point computation (\texttt{DCFP}), a model-based version of {\CTD}, in which they constructed the estimator by solving a linear system directly instead of performing an iterative algorithm.
They showed that, in the synchronous setting, the sample complexity of \texttt{DCFP} is $\wtilde O\prn{{\varepsilon^{-2}(1-\gamma)^{-3}}}$ in the case of the $1$-Wasserstein metric.
This upper bound matches existing lower bounds (up to logarithmic factors) and thus having solved an open problem raised in \citep{zhang2023estimation}.
Roughly speaking, their results imply that learning the full return distribution can be as sample-efficient as learning just its expectation \citep{gheshlaghi2013minimax}.
\citet{pmlr-v202-wu23s} studied the offline distributional policy evaluation problem. They solved the problem via fitted likelihood estimation (\texttt{FLE}) inspired by the classic offline policy evaluation algorithm of fitted Q evaluation (\texttt{FQE}), and provided a generalization bound in the case of the $p$-Wasserstein metric.

Among these works, the work of \cite{speedy} is the closest one to this paper as both perform non-asymptotic analysis of distributional TD in the generative model setting.
However, compared to \citep{speedy}, our analysis is much tighter.
We show that even if we do not introduce any acceleration techniques to the original {\CTD} algorithm, it is still possible to attain the near-minimax optimal sample complexity bounds. 
In a word, we manage to show \emph{sharper} bounds based on a \emph{simpler} algorithm in the generative model setting.
Moreover, we also show non-asymptotic bounds in the more challenging setting with Markovian data.

Table~\ref{table:sample_complexity} gives more detailed comparisons of our results with previous works in the generative model setting.
We set $\mu_{\min}=1$ to make our results comparable with other works in which agents have access to all states at any time $t$.
And Table~\ref{table:sample_complexity_markovian} presents comparisons in the Markovian setting.
Note that solving the distributional policy evaluation problem $\varepsilon$-accurately in the supreme $1$-Wasserstein metric is strictly harder than solving classic policy evaluation problem $\varepsilon$-accurately in the $\ell_\infty$ metric, as the absolute difference of value functions is always bounded by the $1$-Wasserstein metric between return distributions (see the proof of Theorem~\ref{thm:minimax_lower_bound_w1} in Appendix~\ref{Appendix_minimax}). 
We have also listed the sample complexity of the policy evaluation task in the tables for comparison.
\begin{table}
\centering
\begin{tabular}{|c|c|c|c|} 
\hline 
 & Sample Complexity & Algorithms & Task \\
\hline
\citep{li2023breaking} & $\widetilde{O}\left(\frac{1}{\varepsilon^2 (1-\gamma)^3}\right)$ & Model-based & PE \\
\hline
\citep{li2024q} & $\widetilde{O}\left(\frac{1}{\varepsilon^2 (1-\gamma)^3}\right)$ & {\TD} (Model-free) & PE \\
\hline
\citep{rowland2018analysis} & Asymptotic & {\CTD} (Model-free) & DPE \\
\hline
\citep{rowland2023analysis} & Asymptotic & {\QTD} (Model-free) &  DPE\\
\hline
\citep{rowland2024near} & $\widetilde{O}\left(\frac{1}{\varepsilon^2 (1-\gamma)^3}\right)$ & \texttt{DCFP} (Model-based) & DPE \\
\hline
\citep{speedy} & $\widetilde{O}\left(\frac{1}{\varepsilon^2 (1-\gamma)^4}\right)$ & \texttt{SCPE} (Model-free) &  DPE \\
\hline
Our Work & $\widetilde{O}\left(\frac{1}{\varepsilon^2 (1-\gamma)^3}\right)$ & {\CTD} (Model-free) & DPE \\
\bottomrule 
\end{tabular}
\caption{Sample complexity bounds of algorithms in the generative model setting for solving policy evaluation (PE) when the error is measured by $\ell_\infty$ norm, and distributional policy evaluation (DPE) when the error is measured by supreme $1$-Wasserstein metric.}
\label{table:sample_complexity}
\end{table}

\begin{table}
\centering
\begin{tabular}{|c|c|c|c|} 
\hline 
 & Sample Complexity & Algorithms & Task \\
\hline
\citep{li2024q} & $\widetilde{O}\left(\frac{1}{\varepsilon^2\mu_{\pi,\min} (1-\gamma)^4}+\frac{\tmix}{\mu_{\pi,\min}(1-\gamma)}\right)$ & {\TD} & PE \\
\hline
\citep{li2021sample} & $\widetilde{O}\left(\frac{1}{\varepsilon^2\mu_{\pi,\min} (1-\gamma)^3}+\frac{\tmix}{\mu_{\pi,\min}(1-\gamma)}\right)$ & \texttt{VR-TD} & PE \\
\hline
Our Work & $\widetilde{O}\left(\frac{\tmix}{\varepsilon^2 \mu_{\pi,\min}(1-\gamma)^3}\right)$ & {\CTD} & DPE \\
\hline
Our Work & $\widetilde{O}\left(\frac{1}{\varepsilon^2\mu_{\pi,\min} (1-\gamma)^3}+\frac{\tmix}{\mu_{\pi,\min}(1-\gamma)}\right)$ & {\VRCTD} & DPE \\
\bottomrule 
\end{tabular}
\caption{Sample complexity bounds of algorithms in the Markovian setting for solving policy evaluation (PE) when the error is measured by the $\ell_\infty$ norm, and distributional policy evaluation (DPE) when the error is measured by the supreme $1$-Wasserstein metric.}
\label{table:sample_complexity_markovian}
\end{table}

\paragraph{Freedman's Inequality}
Freedman's inequality was originally proposed in \citep{freedman1975tail} as a generalization of Bernstein's inequality to martingales.
\citet{tropp2011freedman} extended Freedman’s inequality to matrix martingales.
And our Freedman's inequality is an extension motivated by statistical analysis for infinite-dimensional online algorithms.
To the best of our knowledge, we are the first to present this version (Theorem~\ref{thm:freedman_ineq}) of Freedman's inequality in Hilbert spaces.
The most related works to ours are \citep{tarres2014online, cutkosky2021high, martinez2024empirical}, where they independently proved a special case of our Theorem~\ref{thm:freedman_ineq} depending on an almost surely upper bound of the quadratic variation that could be restrictive. 
In contrast, our Theorem~\ref{thm:freedman_ineq} provides finer depiction of concentration and can produce sharper upper bounds for the case where the quadratic variation is ``not too large'' with high probability.
Actually, in certain infinite-dimensional problems such as the distributional TD learning we aim to investigate, achieving sharp statistical rates is only possible by the general version Theorem~\ref{thm:freedman_ineq}.

The remainder of this paper is organized as follows. 
In Section~\ref{Section_prelim} we introduce some backgrounds of DRL.
In Section~\ref{Section_dtd}, we revisit the distributional TD algorithms, propose {\NTD} to ease theoretical analysis, and devise the variance-reduced variants of existing algorithms that can achieve sharper rates in the Markovian setting.
In Section~\ref{Section_analysis} we analyze the statistical efficiency of distributional TD algorithms in both the generative model setting and the Markovian setting.
In Section~\ref{Section_Freedman} we present Freedman's inequality in Hilbert spaces.
Section~\ref{Section_proof_outlines} presents the proof outlines of our theoretical results. We conclude our work in Section~\ref{Section_discussion}. All the proof details are given in the appendices.

\section{Backgrounds}\label{Section_prelim}
\subsection{Hilbert Space}\label{subsection:Hilbert}
A Hilbert space $\prn{\sX,\inner{\cdot}{\cdot}}$ is a complete vector space equipped with an inner product\footnote{In this paper, we assume that all the Hilbert and Banach spaces we encounter are separable, which can avoid measurability issues, ensure that the Bochner integral can be defined, and guarantee
tightness of any distribution. See \cite{pisier_2016} for more details about probability in Banach space}.
For a thorough review one may check \citep{yosida2012functional}.
Throughout this paper, we denote by $\sL(\sX)$ the space of all bounded linear operators in $\sX$, 
and $\bI\in\sL\prn{\sX}$ the identity operator.
For any $\bm{A}\in\sL(\sX)$, we denote the operator norm of $\bA$ as $\norm{\bm A}:=\sup_{f\in\sX,\norm{f}=1}\norm{\bm A f}$.
For any operators or matrices $\brc{\bm{A}_k}_{k=1}^t$, $\prod_{k=1}^t \bm{A}_k$ is defined as $\bm{A}_t\bm{A}_{t-1}\cdots \bm{A}_1$.

% \paragraph{Markov Decision Processes}
\subsection{Markov Decision Processes}\label{subsection:MDP}
An infinite-horizon tabular Markov decision process (MDP) is defined by a 5-tuple $M=\<\gS,\gA,\gP_R,P,\gamma\>$, where $\gS$ represents a finite state space, $\gA$ a finite action space, ${\gP_R}$ the distributions of rewards, ${P}$  the transition dynamics, (\ie, $\gP_R(\cdot|s,a)\in\Delta\prn{[0,1]}$),
$P(\cdot|s,a)\in\Delta\prn{\gS}$ for any state action pair $(s,a)\in\gS\times\gA$), and $\gamma\in(0,1)$ a discounted factor.
Here we use $\Delta(\cdot)$ to represent the set of all probability distributions over some set.
Given a policy $\pi\colon\gS\to\Delta\prn{\gA}$ and an initial state $s_0= s$, a random trajectory $\brc{\prn{s_t,a_t,r_t,s_{t+1}}}_{t=0}^\infty$ can be sampled from $M$: 
$a_t\sim\pi(\cdot| s_t)$, $r_t\sim \gP_R({\cdot}| s_t,a_t)$, ${s_{t+1}}\sim P({\cdot}| {s_t,a_t})$ for any $t\in\NB$. 
Given a trajectory, we define the return by $G^\pi(s):=\sum_{t=0}^\infty \gamma^t r_t\in\brk{0,(1-\gamma)^{-1}}.$
We denote by $\eta^\pi(s)$  the probability distribution of $G^\pi(s)$, and ${\bm{\eta}}^\pi:=\prn{\eta^\pi(s)}_{s\in\gS}$.
If the initial state $s_0$ is sampled from $\rho\in\Delta(\gS)$, then the return distribution is given by $\eta^\pi(\rho)=\EB_{s_0\sim\rho}\eta^\pi(s_0)$.
The expected return $V^\pi(s)=\EB G^\pi(s)$ is called value function in conventional RL. 

\subsection{Distributional Bellman Equation and Operator}

Recall the classic policy evaluation task, where one need to evaluate the value functions ${\bm{V}}^\pi=\prn{V^\pi(s)}_{s\in\gS}$ for some given $\pi$.
It is known that the value functions satisfy the Bellman equation. That is, for any $s\in\gS$,
\begin{equation}\label{Equation_Bellman_equation}
            V^\pi(s)=\brk{{\bm{T}}^\pi({\bm{V}}^\pi)}(s):=\EB\brk{r_0+\gamma V^\pi(s_1)\mid s_0=s}.
\end{equation}
The operator ${\bm{T}}^\pi\colon \RB^\gS\to \RB^\gS$ is called the Bellman operator and ${\bm{V}}^\pi$ is a fixed point of ${\bm{T}}^\pi$.

The task of distribution policy evaluation is to find the whole return distributions ${\bm{\eta}}^\pi$ given some fixed policy $\pi$. 
Here
${\bm{\eta}}^\pi$ satisfies a distributional version of the Bellman equation~\eqref{Equation_Bellman_equation}, that is, for any $s\in\gS$
\begin{equation*}\label{Equation_distributional_Bellman_equation}
\begin{aligned}
        \eta^\pi(s)&=\brk{\gT^\pi{\bm{\eta}}^\pi}(s)&=\EB\brk{\prn{b_{r_0,\gamma}}_\#\eta^\pi(s_1)\mid s_0=s}.
\end{aligned}
\end{equation*}
Here $b_{r,\gamma}\colon \RB\to\RB$ is an affine function defined by $b_{r,\gamma}(x)=r+\gamma x$, and $f_\#\nu$ is the push forward measure of $\nu$ through any function $f\colon \RB\to\RB$, so that $f_\#\nu(A)=\nu(f^{-1}(A))$ for any Borel set $A$, where $f^{-1}(A):=\brc{x\colon f(x)\in A}$.
The operator $\gT^\pi\colon \Delta\prn{\brk{0,(1-\gamma)^{-1})}}^\gS\to \Delta\prn{\brk{0,(1-\gamma)^{-1}}}^\gS$ is known as the distributional Bellman operator, and ${\bm{\eta}}^\pi$ is a fixed point of $\gT^\pi$. 
For ease of notation, from now on we  denote $\Delta\prn{\brk{0,(1-\gamma)^{-1}}}$ by $\sP$.
\subsection{\texorpdfstring{$\gT^\pi$}{Tpi} as Contraction in \texorpdfstring{$\sP^\gS$}{PS}}
A key property of the distributional Bellman operator ${\gT}^\pi$ is that it is a contraction w.r.t.\ common probability metrics.
Letting $d$ be a metric on $\sP$, we denote by $\bar{d}$  the corresponding supreme metric on $\sP^\gS$, \ie, $\bar{d}\prn{{\bm{\eta}},{\bm{\eta}}^\prime}:=\max_{s\in\gS}d\prn{\eta(s),\eta^\prime(s)}$ for any ${\bm{\eta}},{\bm{\eta}}^\prime\in\sP^\gS$. 

Suppose $\nu_1$ and $\nu_2$ are two probability distributions on $\RB$ with finite $p$-moments for $p\in[1,\infty]$.
The $p$-Wasserstein metric between $\nu_1$ and $\nu_2$ is defined as 
\[
W_p(\nu_1, \nu_2)=\left(\inf _{\kappa \in \Gamma(\nu_1, \nu_2)} \int_{\RB^2}\abs{x-y}^p \kappa(dx, dy)\right)^{1 / p}.
\]
Each element $\kappa \in \Gamma(\nu_1, \nu_2)$ is a coupling of $\nu_1$ and $\nu_2$, \ie, a joint distribution on $\RB^2$ with prescribed marginals $\nu_1$ and $\nu_2$ on each ``axis.''
When $p=1$ we have $W_1(\nu_1, \nu_2)=\int_\RB |F_{\nu_1}(x)-F_{\nu_2}(x)| dx$,
where $F_{\nu_1}$ and $F_{\nu_2}$ are the cumulative distribution functions of $\nu_1$ and $\nu_2$, respectively.

It can be shown that $\gT^\pi$ is a $\gamma$-contraction w.r.t.\ the supreme $p$-Wasserstein metric $\bar{W}_p$.
\begin{proposition} \cite[Propositions~4.15]{bdr2022} \label{Proposition_contraction_Wp}
    The distributional Bellman operator is a $\gamma$-contraction on $\sP^\gS$ w.r.t.\ the supreme $p$-Wasserstein metric for $p\in[1,\infty]$. That is, for any ${\bm{\eta}},{\bm{\eta}}^\prime\in\sP^\gS$, we have ${\bar{W}_p\prn{\gT^\pi{\bm{\eta}},\gT^\pi{\bm{\eta}}^\prime}\leq \gamma\bar{W}_p({\bm{\eta}},{\bm{\eta}}^\prime)}.$
\end{proposition}
The $\ell_p$ metric for $p\in[1,\infty)$ between $\nu_1$ and $\nu_2$ is defined as $\ell_p(\nu_1, \nu_2)=\prn{\int_\RB \abs{F_{\nu_1}(x)-F_{\nu_2}(x)}^p dx}^{1/p}$,
and $\gT^\pi$ is $\gamma^{1/p}$-contraction w.r.t.\ the supreme $\ell_p$ metric $\bar{\ell}_p$.
\begin{proposition} \cite[Propositions~4.20]{bdr2022}\label{Proposition_contraction_lp}
    The distributional Bellman operator is a $\gamma^{1/p}$-contraction on $\sP^\gS$ w.r.t.\ the supreme $\ell_p$ metric for $p\in[1,\infty)$. That is, for any ${\bm{\eta}},{\bm{\eta}}^\prime\in\sP^\gS$, we have ${\bar{\ell}_p\prn{\gT^\pi{\bm{\eta}},\gT^\pi{\bm{\eta}}^\prime}\leq \gamma^{1/p}\bar{\ell}_p({\bm{\eta}},{\bm{\eta}}        ^\prime)}.$
\end{proposition}
Note that the $\ell_1$ metric coincides with the $1$-Wasserstein metric.
For $p=2$, the $\ell_2$ metric is also called the Cram\'er metric.
It plays an important role in subsequent analysis because the zero-mass signed measure space equipped with this metric, $\prn{\gM, \norm{\cdot}_{\ell_2}}$ defined in Section~\ref{subsection:signed_measure}, is a Hilbert space\footnote{In fact, the space $\prn{\gM, \norm{\cdot}_{\ell_2}}$ is not complete. However, the lack of completeness does not affect the non-asymptotic analysis, see Section~\ref{subsection:signed_measure} for more details.}. 

\section{Distributional Temporal Difference Learning}\label{Section_dtd}

If the MDP $M=\<\gS,\gA,\gP_R,P,\gamma\>$ is known, the value functions ${\bm{V}}^\pi$, as the fixed point of the contraction ${\bm{T}}^\pi$,  can be evaluated via the celebrated dynamic programming algorithm (DP).
To be concrete, for any initialization ${\bm{V}}^{(0)}\in\RB^\gS$, if we define the iteration sequence ${\bm{V}}^{(k+1)}={\bm{T}}^\pi({\bm{V}}^{(k)})$ for $k\in \NB$, we have $\lim_{k\to\infty}\norm{{\bm{V}}^{(k)}-{\bm{V}}^\pi}_\infty=0$ by the contraction mapping theorem \citep[Proposition~4.7 in ][]{bdr2022}. 

Similarly, the distributional dynamic programming algorithm (DP) defines the iteration sequence as ${\bm{\eta}}^{(k+1)}=\gT^\pi{\bm{\eta}}^{(k)}$ for any initialization ${\bm{\eta}}^{(0)}\in\sP^\gS$.
And in the same way, we have $\lim_{k\to\infty}\bar{W}_p({\bm{\eta}}^{(k)},{\bm{\eta}}^\pi)=0$ for $p\in [1,\infty]$ and $\lim_{k\to\infty}\bar{\ell}_p({\bm{\eta}}^{(k)},{\bm{\eta}}^\pi)=0$  for $p\in [1,\infty)$.

However, in most application scenarios, the transition dynamic $P$ and reward distribution $\gP_R$ are unknown, instead we can only access samples of $P$ and $\gP_R$ by online interactions with the environments.
Similar to TD \citep{sutton1988learning} in classic RL, distributional TD employs the stochastic approximation \citep{robbins1951stochastic} technique to address the aforementioned problem and can be viewed as an approximate version of the distributional DP. 
In this paper, we first analyze distributional TD when a generative model \citep{kakade2003sample, kearns2002sparse} is accessible, which can generate independent samples in each iteration.
Then, we will extend the theoretical results to the more challenging and practical setting where the data is generated in a Markovian manner.

\subsection{Distributional TD on a Generative Model}\label{subsection:dtd_with_generative_model}

In this part, we assume that a generative model with a fixed distribution $\mu\in\Delta\prn{\gS}$ with $\mu_{\min}:=\min_{s\in\gS}\mu(s)>0$ is available. That is, in the $t$-th iteration, we collect samples $s_t\sim\mu(\cdot), a_t\sim\pi(\cdot|s_t), r_t\sim \gP_R(\cdot|s_t,a_t), s_t^\prime\sim P(\cdot|s_t,a_t)$.
Let $\bLamb_t=\operatorname{diag}\brc{\prn{\delta_{s,s_{t}}}_{s\in\gS}}$ be a diagonal matrix, where $\delta_{s,s_{t}}=1$ if $s_t=s$, and $0$ otherwise.
And we define the empirical Bellman operator ${\gT}_t^\pi\colon\sP^\gS\to\sP^\gS$ at the $t$-th iteration, which satisfies
\begin{equation}\label{eq:empirical_op}
    \brk{\gT_t^\pi{\bm{\eta}}}(s_{t})=(b_{r_{t},\gamma})_\#(\eta(s_{t}^\prime)),\quad\forall\bm{\eta}\in\sP^\gS.
\end{equation}
It is an unbiased estimator of $\gT^\pi$, namely $\EB\brk{\brk{\gT_t^\pi{\bm{\eta}}}(s_t)\mid s_t}=\brk{\gT^\pi{\bm{\eta}}}(s_t)$. 
We denote by $\gI$ the identity operator in $\sP^\gS$.
\subsubsection{Non-parametric Distributional TD}
We first introduce non-parametric distributional TD learning ({\NTD}), which is helpful in the theoretical understanding of distributional TD.
In the {\NTD} case, we assume the infinite-dimensional return distributions can be precisely updated without any finite-dimensional parametrization. 
For any initialization ${\bm{\eta}}^\pi_0\in\sP^\gS$, the updating scheme is
\[
    {\bm{\eta}}^\pi_t=\prn{\gI-\alpha_t\bLamb_t}{\bm{\eta}}^\pi_{t-1}+\alpha_t\bLamb_t\gT^\pi_t{\bm{\eta}}^\pi_{t-1}
\]
for any $t\geq 1$, where $\alpha_t$ is the step size.
We can find that {\NTD} is a stochastic approximation modification of distributional DP.
Furthermore, if we take expectation of the random measures on both the sides, we obtain the classic TD updating scheme. Hence, {\NTD} is a lift of classic TD to distributional RL.
\subsubsection{Categorical Distributional TD}
Now, we revisit the more practical {\CTD}.
In this case, the updates are computationally tractable, due to the following categorical parametrization of probability distributions.
\[
    \sP_K := \left\{ \sum_{k=0}^K p_k \delta_{x_k} : p_0, \ldots, p_K \geq 0 \, , \sum_{k=0}^K p_k = 1 \right\}, 
\]
where $K\in \NB$, $0\leq x_0<\cdots<x_K\leq \prn{1-\gamma}^{-1}$ are fixed points of the support, and $\delta_{x}$ is the point mass (Dirac measure) at $x\in\RB$.
For simplicity, we assume $\brc{x_k}_{k=0}^K$ are equally-spaced, \ie, $x_k=k \iota_K$, where $\iota_K=\brk{K(1-\gamma)}^{-1}$ is the gap between two points.
When updating the return distributions, we need to evaluate the $\ell_2$-projection of $\sP_K$, $\bm{\Pi}_K\colon\sP\to\sP_K$, $\bm{\Pi}_K(\mu):=\argmin_{\hat{\mu}\in\sP_K}\ell_2(\mu,\hat{\mu})$.
% $$\Pi_K(\mu):=\argmin_{\hat{\mu}\in\sP_K}\ell_2(\mu,\hat{\mu}).$$
It can be shown \citep[see Proposition~5.14 in][]{bdr2022} that the projection is uniquely given by
\begin{equation*}
    \bm{\Pi}_K(\mu)=\sum_{k=0}^K p_k(\mu)\delta_{x_k},
\end{equation*}
where
\begin{equation*}
     p_k(\mu)=\EB_{X\sim \mu}\brk{\prn{1-\abs{\frac{X-x_k}{\iota_K}}}_+},
\end{equation*}
and $(x)_+:=\max\brc{x,0}$ for any $x\in\RB$.
It is known that $\bm{\Pi}_K$ is non-expansive w.r.t.\ the Cram\'er metric \citep[see Lemma~5.23 in][]{bdr2022}, namely $\ell_2(\bm{\Pi}_K(\mu),\bm{\Pi}_K(\nu))\leq\ell_2(\mu,\nu)$ for any $\mu,\nu\in\sP$.
For any ${\bm{\eta}}\in\sP^\gS$, $s\in\gS$, we slightly abuse the notation and define $\brk{\bm{\Pi}_K{\bm{\eta}}}(s) := \bm{\Pi}_K\prn{\eta(s)}$.
$\bm{\Pi}_K$ is also non-expansive w.r.t.\ $\bar{\ell}_2$. Hence $\gT^{\pi,K}:=\bm{\Pi}_K\gT^\pi$ is a $\sqrt{\gamma}$-contraction w.r.t.\ $\bar{\ell}_2$. We denote its unique fixed point by ${\bm{\eta}}^{\pi,K}\in\sP_K^\gS$.
The approximation error incurred by categorical parametrization is given by (see \citep[Proposition~3][]{rowland2018analysis})
\begin{equation}\label{eq:CTD_approx_err}
    \bar{\ell}_2({\bm{\eta}}^{\pi},{\bm{\eta}}^{\pi,K})\leq \frac{1}{\sqrt{K}(1-\gamma)}.
\end{equation}
Now, we are ready to give the updating scheme of {\CTD}, given any initialization ${\bm{\eta}}^\pi_0\in\sP^\gS_K$, at the $t$-th iteration,
\begin{equation*}
    {\bm{\eta}}^\pi_t=\prn{\gI-\alpha_t\bLamb_t}{\bm{\eta}}^\pi_{t-1}+\alpha_t\bLamb_t\bm{\Pi}_K\gT^\pi_t{\bm{\eta}}^\pi_{t-1}.
\end{equation*}
The only difference between {\CTD} and {\NTD} lies in the additional projection operator $\bm{\Pi}_K$ at each iteration in {\CTD}.
In this paper, we consider the last iterate estimator ${\bm{\eta}}^\pi_T$, where $T\in\NB$ is the total number of updates.

\subsection{Distributional TD with Markovian Data}
In this part, we consider the more practical Markovian setting, where the generative model is no longer available and we only has access to a random trajectory $\brc{\prn{s_t,a_t,r_t,s_{t+1}}}_{t=0}^\infty$ generated by the policy $\pi$.
Specifically, $s_0$ is sampled from the initial distribution $\rho\in\Delta\prn{\gS}$, and $a_t\sim\pi(\cdot| s_t)$, $r_t\sim \gP_R({\cdot}| s_t,a_t)$, ${s_{t+1}}\sim P({\cdot}| {s_t,a_t})$ for any $t\in\NB$.
The trajectory $\brc{s_t}_{t=0}^\infty$ is a Markov chain in $\gS$ with Markov transition kernel $P^\pi$, where $P^\pi(s^\prime|s):=\quad\sum_{a\in\gA}\pi(a| s)P(s^\prime| s,a)$.
Without loss of generality, we assume that the Markov chain is irreducible and aperiodic.
Hence, there is a unique stationary distribution $\mu_\pi$, and we denote it by $\mu_{\pi,\min}:=\min_{s\in\gS}\mu_{\pi}(s)$, which satisfies $\mu_{\pi,\min}>0$.
According to \citep[Remark 1.2][]{paulin2015concentration}, an irreducible and aperiodic Markov chain with finite states is uniformly ergodic. That is, there exists $\tmix\in\NB$ such that for all $t\in\NB$,
\begin{equation*}
    \max_{s\in\gS}\TV(P^t(\cdot\mid s),\mu_{\pi}(\cdot))\leq\prn{\frac{1}{4}}^{\lfloor\frac{t}{\tmix} \rfloor},
\end{equation*}
where $P^t(\cdot|s)$ is the distribution of $s_t$ given $s_0=s$, and $\TV$ is total variation distance defined as
\begin{equation*}
    \TV(\mu_1,\mu_2)=\frac{1}{2}\sum_{s\in\gS}\abs{\mu_1(s)-\mu_2(s)},
\end{equation*}
for any $\mu_1,\mu_2\in\Delta\prn{\gS}$.
We can define $\bLamb_t$ and $\gT^\pi_t$ in the same way,
once we replace $s_t^\prime$ with $s_{t+1}$ in Eqn.~\eqref{eq:empirical_op}.
\subsubsection{Distributional TD with Data-drop and Burn-in}
To simplify the analysis, we first consider the distributional TD with data-drop and burn-in, which is also used in the analysis of classic TD with Markovian data \citep{samsonov2024improved}.
Namely, given a burn-in time $T_0\in\NB$ and a updating interval $q\in\NB$,  which are  scaled by the mixing time $\tmix$, we only update ${\bm{\eta}}^\pi_t$ using the samples at time $t$ when $t=T_0+mq$ for some $m\in\NB$. 
Hence, the updating schemes of {\NTD} and {\CTD} are respectively given by
\begin{equation*}
    {\bm{\eta}}^\pi_{m}=\prn{\gI-\alpha_{m}\bLamb_{T_0+mq}}{\bm{\eta}}^\pi_{m-1}+\alpha_{m}\bLamb_{T_0+mq}\gT^\pi_{T_0+mq}{\bm{\eta}}^\pi_{m-1},
\end{equation*}
\begin{equation*}
    {\bm{\eta}}^\pi_{m}=\prn{\gI-\alpha_{m}\bLamb_{T_0+mq}}{\bm{\eta}}^\pi_{m-1}+\alpha_{m}\bLamb_{T_0+mq}\bm{\Pi}_K\gT^\pi_{T_0+mq}{\bm{\eta}}^\pi_{m-1},
\end{equation*}
for any $m\geq1$.
In the Markovian setting, we also consider the last iterate estimator ${\bm{\eta}}^\pi_{T_0+T^\star q}$, where $T^\star\in\NB$ is the total number of updates.
Here, we use different notation $T^\star$ to distinguish it from $T$ in the setting without data-drop. 
We want to emphasize that, in this setting, the actual number of samples required is $T_0+T^\star q$, rather than $T^\star$.

\textbf{Remark \theremark:}\stepcounter{remark}
The data-drop technique is not practical, which not only wastes the collected samples at $t\neq T_0+mq$, but requires utilizing the unknown $\tmix$ as well (although $\tmix$ could be estimated with non-asymptotic theoretical guarantees, see \citep{wolfer2019estimating}).
We employ the data-drop merely to simplify the theoretical analysis. 
For distributional TD without data-drop, one can combine the more sophisticated techniques from \citep{li2021sample} to derive a different sample complexity bound (see Section~\ref{Section_non_asymp_asyn} for more discussions).

\subsubsection{Distributional TD with Variance Reduction}

To achieve sharp sample complexity bound matching that of the classic policy evaluation task in the Markovian setting, we resort to variance reduction techniques \citep{johnson2013accelerating}, giving rise of variance-reduced nonparametric distributional TD learning ({\VRNTD}) and variance-reduced  categorical distributional TD learning ({\VRCTD}). 
{\VRNTD} and {\VRCTD} can be viewed as extensions of variance-reduced TD learning \citep{wainwright2019variance, li2021sample} to the distributional policy evaluation task.
We give some intuitions for variance reduction in distributional TD learning.
We first consider a reference distributional Bellman operator $\wtilde\gT^\pi$ and a reference point $\bar{\bm{\eta}}^\pi$ that are supposed to be good estimates of $\gT^\pi$ and $\bm{\eta}^\pi$, respectively.
As long as $\wtilde\gT^\pi$ is close to $\gT^\pi$ and $\bar{\bm{\eta}}^\pi$ is close to $\bm{\eta}^\pi$, we simply replace $\gT^\pi_t{\bm{\eta}}^\pi_{t-1}$ with $\brc{\gT^\pi_t{\bm{\eta}}^\pi_{t-1}-\gT^\pi_t\bar{\bm{\eta}}^\pi}+\wtilde\gT^\pi\bar{\bm{\eta}}^\pi$, which should have a smaller variance.
Now, we give the complete procedure of {\VRNTD} and {\VRCTD} in Algorithm~\ref{alg:var_reduced_dtd}.
{\VRNTD} and {\VRCTD} are multi-epoch algorithms. 
In the $e$-th epoch for $e\in[E]$, given $\bar{\bm{\eta}}^\pi_{e-1}$, our objective is to obtain a new reference point $\bar{\bm{\eta}}^\pi_e$, which serves as a better estimate of $\bm{\eta}^\pi$.
Within each epoch, we first use $N$ new samples to construct the reference operator $\wtilde\gT^\pi_e$, initialize $\bm{\eta}_{0,e}^\pi$ with $\bar{\bm{\eta}}^\pi_{e-1}$, and update it for the $\tepoch$ steps using $\tepoch$ new samples.
The updating schemes of {\VRNTD} and {\VRCTD} are respectively given by
\begin{equation}\label{eq:vrntd}
    {\bm{\eta}}^\pi_{t,e}=\prn{\gI-\alpha_t\bLamb_{t,e}}{\bm{\eta}}^\pi_{t-1,e}+\alpha_t\bLamb_{t,e}\prn{\gT^\pi_{t,e}{\bm{\eta}}^\pi_{t-1,e}-\gT^\pi_{t,e}\bar{\bm{\eta}}^\pi_{e-1} +\wtilde\gT^\pi_e\bar{\bm{\eta}}^\pi_{e-1}  },
\end{equation}
\begin{equation}\label{eq:vrctd}
    {\bm{\eta}}^\pi_{t,e}=\prn{\gI-\alpha_t\bLamb_{t,e}}{\bm{\eta}}^\pi_{t-1,e}+\alpha_t\bLamb_{t,e}\bm{\Pi}_K\prn{\gT^\pi_{t,e}{\bm{\eta}}^\pi_{t-1,e}-\gT^\pi_{t,e}\bar{\bm{\eta}}^\pi_{e-1} +\wtilde\gT^\pi_e\bar{\bm{\eta}}^\pi_{e-1}  }.
\end{equation}
Then $\bm{\eta}_{\tepoch,e}^\pi$ is treated as the new reference point $\bar{\bm{\eta}}^\pi_{e}$.
After $E$ epochs, we take the last reference point $\bar{\bm{\eta}}^\pi_{E}$ as the final estimator of $\bm{\eta}^\pi$.
Note that the actual number of samples required is $E\prn{N+\tepoch}$ in this setting.

We now consider 
the construction of the reference operator $\wtilde\gT^\pi$ (we omit $e$ for simplicity).
$\wtilde\gT^\pi$ is an unbiased estimator of $\gT^\pi$ based on $N$ consecutive samples $\brc{\prn{\tilde{s}_i,\tilde a_i,\tilde r_i,\tilde s_{i+1}}}_{i=0}^{N-1}$. Concretely, 
\begin{equation}\label{eq:reference_operator}
    \wtilde\gT^\pi{\bm{\eta}} =\prn{\sum_{i=0}^{N-1}\wtilde{\bLamb}_i}^{-1}\sum_{i=0}^{N-1}\wtilde{\bLamb}_i\wtilde\gT^\pi_i {\bm{\eta}}
\end{equation}
for any $\bm{\eta}\in\sP^\gS$.
Here  $\wtilde{\bLamb}_i$ and $\wtilde\gT^\pi_i$ take the same forms as ${\bLamb}_i$ and $\gT^\pi_i$ but use the sample $\prn{\tilde{s}_i,\tilde a_i,\tilde r_i,\tilde s_{i+1}}$ instead.
We remark that $\sum_{i=0}^{N-1}\wtilde{\bLamb}_i$ is invertible with high probability when $N$ is sufficiently large according to the concentration inequality of empirical distributions of Markov chains \citep[Lemma~8][]{li2021sample}.
\begin{algorithm}[H]
\caption{Variance-Reduced Distributional TD Learning with Markovian Data}
\label{alg:var_reduced_dtd}
\begin{algorithmic}[1]
\State \textbf{Input parameters:} initialization $\bar{\bm{\eta}}^\pi_{0}\in\sP^\gS$ ($\bar{\bm{\eta}}^\pi_{0}\in{\sP}_K^\gS$ in {\VRCTD}), initial distribution $\rho\in\Delta\prn{\gS}$, number of epochs $E\in\NB$, epoch length $\tepoch\in\NB$, number of collected samples $N\in\NB$, step sizes $\brc{\alpha_t}_{t=1}^{\tepoch}\subset \prn{0,1}$.
\State \textbf{Initialization:} draw $s_{\text{current}}\sim \rho$.
\For{each epoch $e = 1, \ldots, E$}
    \State Draw $N$ new consecutive samples started from $s_{\text{current}}$, compute $\wtilde{\gT}_e\prn{\bar{\bm{\eta}}^\pi_{e-1}}$ according to Eqn.~\eqref{eq:reference_operator}, and obtain the new $s_{\text{current}}\in\gS$.
    \State Set $\bm{\eta}_{0,e}^{\pi} \gets \bar{\bm{\eta}}^\pi_{e-1}$, and $s_{0,e} \gets s_{\text{current}}$.
    \For{$t = 1, 2, \ldots, t_{\text{epoch}}$}
    
        \State Draw action $a_{t-1,e} \sim \pi(\cdot\mid s_{t-1,e})$, and observe reward $r_{t-1,e}\sim\gP_R\prn{\cdot\mid s_{t-1,e},a_{t-1,e}}$ and next state $s_{t,e} \sim P(\cdot\mid s_{t-1,e}, a_{t-1,e})$.
        \State Update ${\bm{\eta}}^\pi_{t,e}$ according to Eqn.~\eqref{eq:vrntd} in {\VRNTD} or Eqn.~\eqref{eq:vrctd} in {\VRCTD}.
    \EndFor
    \State Set $\bar{\bm{\eta}}^\pi_{e} \gets \bm{\eta}_{\tepoch,e}^{\pi}$, and $s_{\text{current}} \gets s_{\tepoch,e}$.
\EndFor
\State \Return $\bar{\bm{\eta}}^\pi_{E}$.
\end{algorithmic}
\end{algorithm}

\section{Statistical Efficiency} \label{Section_analysis}
In this section we state our main theoretical results.
For both the generative model and Markovian data settings, we give the non-asymptotic convergence rates of the $1$-Wasserstein error for both {\NTD} and {\CTD}.
Before that, we first state the minimax lower bound of the DPE task.

\subsection{Minimax Lower Bound of Distributional Policy Evaluation}

Here we consider the synchronous setting.
In the t-th iteration, for each state $s \in \mathcal{S}$, we independently sample an action $a_t(s)$ from the policy $\pi(\cdot | s)$, a reward $r_t(s)$ from the reward distribution $\mathcal{P}_R(\cdot | s, a_t(s))$, and a next state $s^{\prime}_t(s)$ from the transition dynamic $P(\cdot | s, a_t(s))$ using a generative model for later updates.
The main difference between this setting and the generative model setting is that here we can simultaneously visit all states at one iteration.
For any positive integer $D$, we define $\sM\prn{D}$ as the set of all MDPs with state space size $\abs{\gS}=D$.
For any MDP $M$ and policy $\pi$, we denote by ${\bm{V}}^\pi_M$  the corresponding value function, and by ${\bm{\eta}}^\pi_M$ the corresponding return distribution.

Now, we can state the minimax lower bound of the distributional policy evaluation task in the $1$-Wasserstein metric, whose proof can be found in Appendix~\ref{Appendix_minimax}.
\begin{theorem}[Minimax lower bound of distributional policy evaluation w.r.t.\ the $1$-Wasserstein metric]\label{thm:minimax_lower_bound_w1}
For any positive integer $D\geq 3$, and sample size $T\geq {C}\prn{1-\gamma}^{-1}\log\prn{D/2}$, the following result holds
\begin{equation*}
    \inf_{\hat{{\bm{\eta}}}}\sup_{M\in\sM\prn{D}}\sup_{\pi}\EB\brk{\bar{W}_1\prn{\hat{{\bm{\eta}}},{\bm{\eta}}^\pi_M}}\geq \frac{c}{(1-\gamma)^{3/2}}\sqrt{\frac{\log\frac{D}{2}}{T}}.
\end{equation*}
Here, $c, C>0$ are universal constants, and the infimum $\hat{{\bm{\eta}}}\in\sP^{D}$ ranges over all measurable functions of $T$ samples from the generative model.
\end{theorem}
The theorem says that for any DPE algorithm working in the synchronous setting, there exist a MDP $M$ and a policy $\pi$ such that we need at least $\tilde{\Omega}\prn{{\varepsilon^{-2}(1-\gamma)^{-3}}}$ samples to ensure $\EB\brk{\bar{W}_1\prn{\hat{{\bm{\eta}}},{\bm{\eta}}^\pi_M}}\leq\varepsilon$ for some $\varepsilon>0$.
In our main result (Theorem~\ref{thm:samplecomplex_syn_w1}), we consider the generative model setting where at each iteration only one state $s_t$ sampled from $\mu$ is accessible.
Hence an additional $1/{\mu_{\min}}$ appears in the sample complexity bound accordingly.
According to \citep[Example~G.1][]{zhang2023estimation}, the factor $1/{\mu_{\min}}$ is inevitable, and the minimax lower bound becomes $\tilde{\Omega}\prn{{\varepsilon^{-2}\mu_{\min}^{-1}(1-\gamma)^{-3}}}$.

\subsection{Non-Asymptotic Analysis with a Generative Model}\label{subsection:nonasym_generative_model}
We present a non-asymptotic convergence rate of $\bar{W}_1({\bm{\eta}}^\pi_T,{\bm{\eta}}^{\pi})$ in the generative model setting for both {\NTD} and {\CTD}, which is minimax optimal up to logarithmic factors.
\begin{theorem}[Sample Complexity of {\NTD} and {\CTD} for the Generative Model Setting]\label{thm:samplecomplex_syn_w1}
In the generative model setting, for any  $\delta \in (0,1)$ and $\varepsilon\in (0, 1)$,
suppose that $K> {4}{\varepsilon^{-2}(1-\gamma)^{-3}}$ in {\CTD}, the initialization is ${\bm{\eta}}^\pi_0\in\sP^\gS$ ($\sP^\gS_K$ in {\CTD}),  
the total update steps $T$ satisfies
\begin{equation*}
    T \geq  \frac{C_1\log^3 T}{\varepsilon^2 \mu_{\min}(1-\gamma)^3}\log \frac{\abs{\gS}T}{\delta}
\end{equation*}
for some large universal constant $C_1>1$ (\ie, $T=\widetilde{O}\left({\varepsilon^{-2} \mu_{\min}^{-1}(1-\gamma)^{-3}}\right)$), and the step size $\alpha_t$ is set as
\begin{equation*}
    \frac{1}{1+\frac{c_2\mu_{\min}(1-\sqrt\gamma)t}{\log t}}\leq\alpha_t\leq \frac{1}{1+\frac{c_3\mu_{\min}(1-\sqrt\gamma)t}{\log t}}
\end{equation*}
for some small universal constants $c_2>c_3>0$.
Then, for both {\NTD} and {\CTD}, with probability at least $1-\delta$, the last iterate estimator satisfies $\bar{W}_1\prn{{\bm{\eta}}^\pi_T,{\bm{\eta}}^{\pi}}\leq \varepsilon$.
\end{theorem}
The key idea of our proof is to first expand the error term $\bar{W}_1\prn{{\bm{\eta}}^\pi_T,{\bm{\eta}}^{\pi}}$ ($\bar{W}_1\prn{{\bm{\eta}}^\pi_T,{\bm{\eta}}^{\pi,K}}$ in {\CTD}) over  time steps.
Then it can be decomposed into an initial error term and a martingale term.
The initial error term becomes smaller as the iterations progress due to the contraction properties of $\gT^\pi$.
To control the martingale term, we first use the basic inequality (Lemma~\ref{lem:prob_basic_inequalities}) $W_1\prn{\mu,\nu}\leq\prn{{1-\gamma}}^{-1/2}\ell_2\prn{\mu,\nu}$, which allows us to analyze this error term in the Hilbert space $(\gM,\norm{\cdot}_{\ell_2})$ defined in Section~\ref{subsection:signed_measure}.
Consequently, we can bound it by Freedman's inequality in Hilbert spaces (Theorem~\ref{thm:freedman_ineq_bounded_W}).
A more detailed outline of the proof is given in Section~\ref{Subsection_analysis_nonasymp_ntd_w1}.

We comment that the derived sample complexity of {\CTD} is better than the previous results \texttt{SCPE} \citep{speedy}. 
Moreover, the additional acceleration term introduced in the updating scheme of \texttt{SCPE} is not needed.
In addition, our theoretical result matches the sample complexity bound of the model-based method DCFP \citep{rowland2024near}.
If we take expectation of ${\bm{\eta}}^\pi_T$ in {\NTD}, we can recover the sample complexity result of classic TD in the generative model setting \citep[see Theorem~1 in ][]{li2024q}.

\textbf{Remark \theremark \ (The Polyak-Ruppert averaging):}\stepcounter{remark}
In Theorem~\ref{thm:samplecomplex_syn_w1} we show a last-iterate convergence; we can show the same sample complexity bound when we use the Polyak-Ruppert averaging $\prn{T-T_0}^{-1}\sum_{t=T_0+1}^T{\bm{\eta}}^\pi_t$, where $T_0\in\NB$ is the burn-in period. See Appendix~\ref{subsection:Polyak} for more details.

\textbf{Remark \theremark \  (The mean error bound):}\stepcounter{remark}
Since $\bar{W}_1\prn{{\bm{\eta}}^\pi_T,{\bm{\eta}}^{\pi}}\leq \prn{1-\gamma}^{-1}$ always holds, we can translate the high probability bound to a bound of expectations, namely
\begin{equation*}
    \EB\brk{\bar{W}_1\prn{{\bm{\eta}}^\pi_T,{\bm{\eta}}^{\pi}}}\leq \varepsilon(1-\delta)+\frac{\delta}{1-\gamma}\leq 2\varepsilon,
\end{equation*}
if taking $\delta\leq\varepsilon(1-\gamma)$.

\textbf{Remark \theremark \ (The $W_p$ error bound):}\stepcounter{remark}
Combining Theorem~\ref{thm:samplecomplex_syn_w1} with the basic inequalities $\bar{W}_p({\bm{\eta}},{\bm{\eta}}^\prime)\leq\prn{1-\gamma}^{-1+1/p}\bar{W}_1^{1/p}({\bm{\eta}},{\bm{\eta}}^\prime)$ for any ${\bm{\eta}}, {\bm{\eta}}^\prime\in\sP^\gS$ (Lemma~\ref{lem:prob_basic_inequalities}), we can derive that $T=\widetilde{O}\left({\varepsilon^{-2p}\mu_{\min}^{-p} (1-\gamma)^{-2p-1}}\right)$ iterations are sufficient to ensure $\bar{W}_p({\bm{\eta}}^\pi_T,{\bm{\eta}}^\pi)\leq\varepsilon$.
As pointed out in the \citep[Example after Corollary 3.1 in][]{zhang2023estimation}, when $p>1$, the slow rate in terms of $\varepsilon$ is inevitable without additional regularity conditions.

\textbf{Remark \theremark \ (The $\ell_2$ error bound):}\stepcounter{remark}
Generally speaking, the Cram\'er distance between two probability distributions with bounded support can be bounded with the the $1$-Wasserstein distance.
However, naively translating the $1$-Wasserstein error bounds into the $\ell_2$ error bounds with this fact would yield a loose rate.
Instead, we find that by making slight modifications to the proof, we can show that the non-asymptotic convergence rate of $\bar{\ell}_2({\bm{\eta}}^\pi_T,{\bm{\eta}}^{\pi})$ is $\widetilde{O}\left({\varepsilon^{-2} \mu_{\min}^{-1} (1-\gamma)^{-5/2}}\right)$.
See Appendix~\ref{Subsection_analysis_coro_ntd_l2} for more details.

\textbf{Remark \theremark \ (Reduce to the synchronous setting):}\stepcounter{remark}
If we replace $\bLamb_t$ defined in Section~\ref{subsection:dtd_with_generative_model} with the identity matrix $\bI$, the conclusion and the proofs still remain with $\mu_{\min}=1$.
This is in fact the synchronous setting, in which at the $t$-th iteration we sample ${a}_t(s)\sim\pi(\cdot|s), {r}_t(s)\sim \gP_R(\cdot|s,{a}_t(s)), {s}_t^\prime(s)\sim P(\cdot|s,{a}_t(s))$ for each $s\in\gS$ independently to update.

In subsequent discussion, we will not state the content in these remarks for the sake of brevity.
\subsection{Non-Asymptotic Analysis with Markovian Data}\label{Section_non_asymp_asyn}
We first state a parallel result to Theorem~\ref{thm:samplecomplex_syn_w1} for both {\NTD} and {\CTD} with data-drop and burn-in in the Markovian setting.
\begin{theorem}[Sample Complexity of {\NTD} and {\CTD} with Data-drop and Burn-in in the Markovian Setting]\label{thm:samplecomplex_asyn_w1}
For any  $\delta \in (0,1)$ and $\varepsilon\in (0,1)$,
suppose that $K> {4}{\varepsilon^{-2}(1-\gamma)^{-3}}$ in {\CTD},
the initialization is ${\bm{\eta}}^\pi_0\in\sP^\gS$ ($\sP_K^\gS$ in {\CTD}), burn-in time is set as $T_0=\lceil\tmix \log\prn{12/\delta}\rceil$,
the total update steps $T^\star$ satisfies
\begin{equation*}
    T^\star \geq  \frac{C_1 \log^3 T^\star}{\varepsilon^2 \mu_{\pi,\min}(1-\gamma)^3}\log \frac{\abs{\gS}T^\star}{\delta}
\end{equation*}
for some large universal constant $C_1>1$, 
the updating interval is set as $q=\lceil\tmix\log\prn{{3T^\star}/{\delta}}\rceil$ (\ie, $T_0+qT^{\star}=\widetilde{O}\left({\tmix}{\varepsilon^{-2} \mu_{\pi,\min}^{-1}(1-\gamma)^{-3}}\right)$ samples are required), and the step size $\alpha_m$ is set as
\begin{equation*}
    \frac{1}{1+\frac{c_2\mu_{\pi,\min}(1-\sqrt\gamma)m}{\log m}}\leq\alpha_m\leq \frac{1}{1+\frac{c_3\mu_{\pi,\min}(1-\sqrt\gamma)m}{\log m}}
\end{equation*}
for some small universal constants $c_2>c_3>0$.
Then, for both {\NTD} and {\CTD}, with probability at least $1-\delta$, the last iterate estimator satisfies $\bar{W}_1\prn{{\bm{\eta}}^\pi_{T_0+T^\star q},{\bm{\eta}}^{\pi}}\leq \varepsilon$.
\end{theorem}
The key idea of the proof is to reduce the Markovian setting to the generative model setting based on Berbee’s coupling lemma \citep{berbee1979}, which is also used in the proof of \citep[Theorem 6][]{samsonov2024improved}.
A detailed outline of proof is given in Section~\ref{Subsection_analysis_nonasymp_markov_w1}.

If we consider value functions induced by ${\bm{\eta}}^\pi_T$ in {\NTD}, 
we  have a new $\widetilde{O}\left({\tmix}{\varepsilon^{-2} \mu_{\pi,\min}^{-1}(1-\gamma)^{-3}}\right)$ sample complexity result of classic TD in the Markovian setting.
Compared to the state-of-the-art $\widetilde{O}\left({\varepsilon^{-2} \mu_{\pi,\min}^{-1}(1-\gamma)^{-4}}+{\tmix}{\mu_{\pi,\min}^{-1}(1-\gamma)^{-1}}\right)$  non-asymptotic result of classic TD \citep[Theorem~4][]{li2024q}, our result has a better dependency on the effective horizon $\prn{1-\gamma}^{-1}$, but a worse dependency on the mixing time $\tmix$ (as $\tmix$ is almost independent of $\varepsilon$ in \citep{li2024q}).
However, a variance-reduced version of classic TD can achieve $\widetilde{O}\left({\varepsilon^{-2} \mu_{\pi,\min}^{-1}(1-\gamma)^{-3}}+{\tmix}{\mu_{\pi,\min}^{-1}(1-\gamma)^{-1}}\right)$ sample complexity bound \citep[Theorem~4][]{li2021sample}, which has the best dependency on $\prn{1-\gamma}^{-1}$ and $\tmix$ simultaneously.

Our proposed variance-reduced distributional TD ({\VRNTD} and {\VRCTD}) can also achieve the best sample complexity bound known for the classic policy evaluation task in the Markovian setting as mentioned earlier.  
This result is summarized in the following theorem.
\begin{theorem}[Sample Complexity for {\VRNTD} and {\VRCTD} in the Markovian Setting]\label{thm:samplecomplex_asyn_w1_variance_reduction}
%In the Markovian setting, 
For any  $\delta \in (0,1)$ and $\varepsilon\in (0,1)$,
suppose that $K> {4}{\varepsilon^{-2}(1-\gamma)^{-3}}$ in {\VRCTD},
the initialization is ${\bm{\eta}}^\pi_0\in\sP^\gS$ ($\sP_K^\gS$ in {\VRCTD}),
the number of epochs $E$ satisfies 
\begin{equation*}
    E\geq C_1\log\frac{1}{\varepsilon(1-\gamma)^2},
\end{equation*}
the recentering length $N$ satisfies
\begin{equation*}
    N\geq \frac{C_2}{\mu_{\pi,\min}}\prn{\frac{1}{\varepsilon^2(1-\gamma)^3}+\tmix}\log\frac{\abs{\gS}N}{\delta},
\end{equation*}
the epoch length $\tepoch$ satisfies
\begin{equation*}
    \tepoch\geq\frac{C_3}{\mu_{\pi,\min}}\prn{\frac{1}{(1-\gamma)^3}+\tmix}\prn{\log\frac{1}{\varepsilon(1-\gamma)^2}}\prn{\log\frac{\abs{\gS}\tepoch}{\delta}},
\end{equation*}
for some large universal constant $C_1, C_2, C_3>1$, namely, 
\[E\prn{N+\tepoch}=\widetilde{O}\left({\varepsilon^{-2} \mu_{\pi,\min}^{-1}(1-\gamma)^{-3}}+{\tmix}{\mu_{\pi,\min}^{-1}(1-\gamma)^{-1}}\right)\] samples are required, 
and the step size $\alpha_t$ is set as
\begin{equation*}
    \alpha_t\equiv\alpha=\frac{c_4}{\log\frac{\abs{\gS}\tepoch}{\delta}}\min\brc{\prn{1-\sqrt{\gamma}}^2,\frac{1}{\tmix}},
\end{equation*}
for some small universal constants $c_4>0$.
Then, for both {\VRNTD} and {\VRCTD}, with probability at least $1-\delta$, the last epoch estimator satisfies $\bar{W}_1\prn{\bar{\bm{\eta}}^\pi_{E},{\bm{\eta}}^{\pi}}\leq \varepsilon$.
\end{theorem}
The key idea of the proof is to combine the techniques in the proof of Theorem~\ref{thm:samplecomplex_syn_w1} with the proof of the variance-reduced Q-learning \citep[Theorem~4][]{li2021sample}.
A detailed outline of proof is given in Section~\ref{Subsection_analysis_nonasymp_markov_w1_variance_reduction}.
We comment that, using the argument from \citep[Theorem~3][]{li2021sample}, we can replace the constant step size with an adaptive step size that does not depend on the unknown $\tmix$, without affecting the sample complexity. 
For brevity, we omit the details.

We wonder if it is possible that the vanilla (instead of variance-reduced) distributional TD in the Markovian setting could achieve such a tight sample complexity bound.
We leave this question for future work.
\section{Freedman's Inequality in Hilbert Space}\label{Section_Freedman}
To prove our main results, 
we present a Freedman's inequality in Hilbert spaces.
First, we generalize the original Freedman's inequality \citep[Theorem~1.6][]{freedman1975tail} to Hilbert space.
Then we will generalize a more useful version of Freedman's inequality \citep[Theorem~6][]{li2024q} to Hilbert spaces, where we have a finer concentration bounds depending on the quadratic variation process itself instead of a crude bound of the process.
We will give self-contained proofs in Appendix~\ref{Appendix_freedman}.
Our proof techniques are mainly inspired by \citep[Theorem~3.2][]{pinelis1994optimum}.

Let $\sX$ be a Hilbert space, $n\in\NB\cup\brc{\infty}$, $\brc{X_i}_{i=1}^n$ be an $\sX$-valued martingale difference sequence adapted to the filtration $\brc{\gF_i}_{i=1}^n$ for $i\in[n]:=\brc{1, 2, \ldots, n}$ ($[\infty]:=\NB$), $Y_i:=\sum_{j=1}^i X_j$ be the corresponding martingale,
and $W_i:=\sum_{j=1}^i\sigma_j^2$ be the corresponding previsible quadratic variation process.
For a more concrete introduction to martingales one may refer to \citep{durrett2019probability}.
Here $\sigma_j^2:=\EB_{j-1}\norm{X_j}^2$, and $\EB_{i}\brk{\cdot}:=\EB\brk{\cdot|\gF_{i}}$ is the conditional expectation.
\begin{theorem}[Freedman's inequality in Hilbert spaces]\label{thm:freedman_ineq}
Suppose $\sup_{i\in[n]}\norm{X_i}\leq b$ almost surely for some constant $b>0$.
Then, for any $\varepsilon, \sigma>0$, it holds that
\begin{equation*}
\begin{aligned}
    \PB\big(\exists k\in[n],  \: & \norm{Y_k}\geq \varepsilon  \mbox{ and } W_k\leq \sigma^2\big)\leq 2\exp\brc{-\frac{\varepsilon^2/2}{\sigma^2 {+} b\varepsilon/3}}.
\end{aligned}
\end{equation*}
\end{theorem}
Now, we state the generalization of \citep[Theorem~6][]{li2024q} to Hilbert spaces, which is a key technical tool in our non-asymptotic analysis.
The proof can be found in Appendix~\ref{Subsection_freedman_proof_2}.
\begin{theorem}[Freedman's inequality in Hilbert spaces: $W_k$-dependent upper bound]\label{thm:freedman_ineq_bounded_W}
Suppose $\sup_{i\in[n]}\norm{X_i}\leq b$ and $W_n\leq \sigma^2$ almost surely for some constants $b, \sigma>0$. 
Then, for any $\delta\in(0,1)$, and any positive integer $H\geq 1$, with probability at least $1-\delta$, for all $k\in [n]$, the following inequality holds 
\begin{equation}\label{eq:freedman_2}
    \norm{Y_k}\leq\sqrt{8\max\Big\{ W_{k},\frac{\sigma^{2}}{2^{H}}\Big\}\log\frac{2H}{\delta}}+\frac{4}{3}b\log\frac{2H}{\delta}.
\end{equation}
\end{theorem}
Compared to the original version of Freedman's inequality (Theorem~\ref{thm:freedman_ineq}), this version is more convenient for applications, and more closely resembles Bernstein's inequality.
To see this, if we assume $\brc{X_i}_{i=1}^n$ are i.i.d. and $n\in\NB$, Bernstein's inequality reads that with probability at least $1-\delta$
\begin{equation*}
    \max_{k\in[n]}\norm{Y_k}\leq\sqrt{4n\EB\brk{\norm{X_1}^2}\log\frac{2}{\delta}}+\frac{4}{3}b\log\frac{2}{\delta}.
\end{equation*}
Here, the second term $\frac{4}{3}b\log\prn{{2H}/{\delta}}$ is not critical because it is independent of $n$, while the first term scales with $\sqrt{n}$. As long as we take $H$ sufficiently large (but still keep the $\log\prn{2H/\delta}$ term small), we have ${\sigma^2}/{2^H}$ sufficiently small. 
Therefore, this upper bound is of order $\sqrt{W_n\log\prn{{2H}/{\delta}}}$.
Our concentration bounds can be much tighter that the $\sqrt{2nb^2\log\prn{2/\delta}}$ bound in the Azuma-Hoeffding inequality \citep[Theorem~3.1][]{luo2022azuma} and the $\sqrt{4\sigma^2\log\prn{2/\delta}}$ bound in the Bernstein's inequality provided in \citep{tarres2014online,cutkosky2021high,martinez2024empirical}, because $W_n$ can be much smaller than $nb^2$ and $\sigma^2$ with high probability.

\section{Proof Outlines}\label{Section_proof_outlines}
In this section, we  outline the proofs of our main theoretical results (Theorem~\ref{thm:samplecomplex_syn_w1} and Theorem~\ref{thm:samplecomplex_asyn_w1}).
Before diving into the details of the proofs, we first define some notation. 
\subsection{Zero-mass Signed Measure Space}\label{subsection:signed_measure}
To analyze the distance between the estimator and the ground-truth ${\bm{\eta}}^\pi$, we will work with the zero-mass signed measure space $\gM$ defined as follows
\begin{equation*}
    \gM:= \left\{\nu\colon\abs{\nu}(\RB)< \infty ,\nu(\RB)=0,\text{supp}(\nu)\subseteq \left[0,\frac{1}{1{-}\gamma} \right] \right\},
\end{equation*} 
where $\abs{\nu}$ is the total variation measure of the signed measure $\nu$, and $\text{supp}(\nu)$ is the support of $\nu$. 
See \citep{bogachev2007measure} for more details about signed measures.

For any $\nu\in\gM$, we define its cumulative function as $F_\nu(x):=\nu[0,x)$. 
It is easily checked that $F_\nu$ is linear in $\nu$, \ie, $F_{\alpha\nu_1+\beta\nu_2}=\alpha F_{\nu_1}+\beta F_{\nu_2}$ for any $\alpha,\beta\in\RB$, $\nu_1,\nu_2\in\gM$.

To analyze the case with the  Cram\'er metric, we define the following Cram\'er inner product on $\gM$,
\begin{equation*}
    \inner{\nu_1}{\nu_2}_{\ell_2}:=\int_{0}^{\frac{1}{1-\gamma}}F_{\nu_1}(x)F_{\nu_2}(x)d x.
\end{equation*} 
It is easy to verify that $\inner{\cdot}{\cdot}_{\ell_2}$ is indeed an inner product on $\gM$. The corresponding norm, called Cram\'er norm, is given by $\norm{\nu}_{\ell_2}=\sqrt{\inner{\nu}{\nu}_{\ell_2}}=\sqrt{\int_{0}^{\prn{1{-}\gamma}^{-1}}\prn{F_{\nu}(x)}^2 d x}$, which is exactly the Cram\'er metric if $\nu$ is the difference of two probability distributions.
Let
\begin{equation*}
    \gM_K := \left\{ \sum_{k=0}^K p_k \delta_{x_k} : p_0, \ldots, p_K \geq 0 \, , \sum_{k=0}^K p_k = 0 \right\},
\end{equation*} 
then $\gM_K$ is a linear subspace of the Hilbert space $\prn{\gM,\inner{\cdot}{\cdot}}$, and the Cram\'er projection operator $\bm{\Pi}_K$ is the orthogonal projection onto $\gM_K$.
Similarly, we define the $W_1$ norm on $\gM$: $\norm{\nu}_{W_1}:=\int_0^{\prn{1{-}\gamma}^{-1}}\abs{F_\nu(x)}dx$. 

We can extend the distributional Bellman operator $\gT^\pi$ and the Cram\'er projection operator $\bm{\Pi}_K$ naturally to $\gM^\gS$ without modifying its definition.
Here, the product space $\gM^\gS$ is a Banach space equipped with the supreme norm: $\norm{{\bm{\eta}}}_{\bar{\ell}_2}:=\max_{s\in\gS}\norm{\eta(s)}_{\ell_2}$, or $\norm{{\bm{\eta}}}_{\bar{W}_1}:=\max_{s\in\gS}\norm{\eta(s)}_{W_1}$ for any ${\bm{\eta}}\in\gM^\gS$.
And we denote $\gI$ as the identity operator in $\gM^\gS$.

\begin{proposition}\label{Proposition_extention_of_Bellman_operator}
    $\gT^\pi$ and $\bm{\Pi}_K$ are linear operators in $\gM^\gS$. Furthermore, $\norm{\gT^\pi}_{\bar{\ell}_2}\leq\sqrt{\gamma}$, $\norm{\gT^\pi}_{\bar{W}_1}\leq\gamma$, 
    % $\norm{\gT^\pi}_{\bar{\KS}}=1$,
    $\norm{\bm{\Pi}_K}_{\bar{\ell}_2}=1$, and $\norm{\bm{\Pi}_K}_{\bar{W}_1}\leq 1$.
\end{proposition}
The proof of the last inequality can be found in the proof of Lemma~\ref{lem:sigma_fine_upper_bound}, while the remaining results are trivial, we omit the proofs for brevity.

Moreover, we have the following matrix (of operators) representations of $\gT^\pi$ and $\bm{\Pi}_K$.
$\gT^\pi\in \sL(\gM)^{\gS\times\gS}$:  for any ${\bm{\eta}}\in\gM^\gS$
\begin{equation*}
\begin{aligned}
        \brk{\gT^\pi{\bm{\eta}}}(s)
=\sum_{s^\prime\in\gS}\gT^\pi(s,s^\prime)\eta(s^\prime),
\end{aligned}
\end{equation*}
where $\gT^\pi(s,s^\prime)\in\sL(\gM)$, for any $\nu\in\gM$
\begin{equation*}
\begin{aligned}
        \gT^\pi(s,s^\prime)\nu=\sum_{a\in\gA}\pi(a| s)P(s^\prime| s,a)\int_0^1 \prn{b_{r,\gamma}}_\#\nu\gP_R(dr| s,a).
\end{aligned}
\end{equation*}
It can be verified that $\norm{\gT(s,s^\prime)}_{\ell_2}\leq \sqrt{\gamma}P^\pi(s^\prime|s)$. Similarly, $\norm{\gT(s,s^\prime)}_{W_1}\leq\gamma P^\pi(s^\prime|s)$, where $P^\pi(s^\prime|s):=\quad\sum_{a\in\gA}\pi(a| s)P(s^\prime| s,a)$.
These upper bounds will be proved in place  where they are used.
And we denote $\bm{\Pi}_K=\text{diag}{\prn{\bm{\Pi}_K\big|_\gM}_{s\in\gS}}\in \sL(\gM)^{\gS\times\gS}$.
Accordingly, $\bm{\Pi}_K\gT^\pi\in \sL(\gM)^{\gS\times\gS}$ can be interpreted as matrix multiplication, where the scalar multiplication is replaced by the composition of operators.
It can be verified that $\prn{\bm{\Pi}_K\gT^\pi}(s,s^\prime)=\bm{\Pi}_K \gT^\pi(s,s^\prime)$, and $\norm{\prn{\bm{\Pi}_K\gT^\pi}(s,s^\prime)}_{\ell_2}\leq\sqrt{\gamma}P^\pi(s^\prime|s)$.

\textbf{Remark \theremark:}\stepcounter{remark}
In Lemma~\ref{lem:separable}, we show that both $\prn{\gM, \norm{\cdot}_{\ell_2}}$ and $\prn{\gM, \norm{\cdot}_{W_1}}$ are separable.
And in Lemma~\ref{lem:incomplete}, we show that $\prn{\gM, \norm{\cdot}_{W_1}}$ is not complete.
To resolve this problem, we will use their completions to replace them without loss of generality, because the completeness property does not affect the separability.
For simplicity, we still use $\gM$ to denote the completion space.
According to the BLT theorem  \citep[Theorem 5.19][]{hunter2001applied}, any bounded linear operator can be extended to the completion space, and still preserves its operator norm.
\subsection{Analysis of Theorem~\ref{thm:samplecomplex_syn_w1}}\label{Subsection_analysis_nonasymp_ntd_w1}
For simplicity, we abbreviate both $\norm{\cdot}_{\bar{\ell}_2}$ and $\norm{\cdot}_{\ell_2}$ as $\norm{\cdot}$ in this part.
For all $t\in[T]$, we denote $\gT_t:=\gT^\pi_t$, $\gT:=\gT^\pi$, ${\bm{\eta}}:={\bm{\eta}}^{\pi}$ for {\NTD}; $\gT_t:=\bm{\Pi}_K \gT^\pi_t$, $\gT:=\bm{\Pi}_K\gT^\pi$, ${\bm{\eta}}:={\bm{\eta}}^{\pi,K}$ for {\CTD};
and ${\bm{\eta}}_t:={\bm{\eta}}^\pi_t$, ${\bm{\Delta}}_t:={\bm{\eta}}_t-{\bm{\eta}}\in\gM^\gS$ for both {\NTD} and {\CTD}.
According to Lemma~\ref{lem:range_eta}, ${\bm{\eta}}_t\in\sP^\gS$ for {\NTD} and ${\bm{\eta}}_t\in\sP^\gS_K$ for {\CTD}.
Our goal is to bound the $\bar{W}_1$ norm of the error term $\norm{{\bm{\Delta}}_T}_{\bar{W}_1}$.
This can be achieved by bounding $\norm{{\bm{\Delta}}_T}$, as $\norm{{\bm{\Delta}}_T}_{\bar{W}_1} \leq\prn{1-\gamma}^{-1/2}\norm{{\bm{\Delta}}_T}$ by Lemma~\ref{lem:prob_basic_inequalities}.

Letting $   \bLamb:=\EB\brk{\bLamb_t}=\operatorname{diag}\brc{\prn{\mu(s)}_{s\in\gS}}$, we can also regard $\bLamb, \bLamb_t\in\sL(\gM)^{\gS\times\gS}$.
According to the updating rule, we have the error decomposition
\begin{equation*}
\begin{aligned}
    {\bm{\Delta}}_t&={\bm{\eta}}_t-{\bm{\eta}}\\
    &={\bm{\eta}}_{t-1}-\alpha_t\bLamb_t\prn{\gI-\gT_t}{\bm{\eta}}_{t-1}-{\bm{\eta}} \\
    &=\brk{\gI-\alpha_t\bLamb_t\prn{\gI-\gT_t}}{\bm{\eta}}_{t-1}-{\bm{\eta}}\\
    &=\brk{\gI-\alpha_t\bLamb_t\prn{\gI-\gT}}{\bm{\eta}}_{t-1}+\alpha_t\bLamb_t\prn{\gT_t-\gT}{\bm{\eta}}_{t-1}-{\bm{\eta}}\\
    &=\brk{\gI-\alpha_t\bLamb\prn{\gI-\gT}}{\bm{\eta}}_{t-1}+\alpha_t\prn{\bLamb-\bLamb_t}\prn{\gI-\gT}{\bm{\eta}}_{t-1}+\alpha_t\bLamb_t\prn{\gT_t-\gT}{\bm{\eta}}_{t-1}-{\bm{\eta}}\\
    &=\brk{\gI-\alpha_t\bLamb\prn{\gI-\gT}}{\bm{\Delta}}_{t-1}+\alpha_t\prn{\bLamb-\bLamb_t}\prn{\gI-\gT}{\bm{\eta}}_{t-1}+\alpha_t\bLamb_t\prn{\gT_t-\gT}{\bm{\eta}}_{t-1}.
\end{aligned}
\end{equation*}
Applying it recursively, we can further decompose the error into two terms:
\begin{align*}
    {\bm{\Delta}}_t=&\brc{\prod_{k=1}^t \brk{\gI-\alpha_k\bLamb\prn{\gI-\gT}}{\bm{\Delta}}_0}\\
    &+\Bigg\{\sum_{k=1}^t \alpha_k \prn{\prod_{i=k+1}^t \brk{\bI-\alpha_i\bLamb\prn{\gI-\gT}}}\brk{\prn{\bLamb-\bLamb_k}\prn{\gI-\gT}{\bm{\eta}}_{k-1}+\bLamb_k\prn{\gT_k-\gT}{\bm{\eta}}_{k-1}}\Bigg\}\\
    =&\text{(I)}_t+\text{(II)}_t.
\end{align*}
Term (I) is an initial error term that becomes negligible when $t$ is large because $\gT$ is a contraction.
Term (II) is a martingale and can be bounded via Freedman's inequality in Hilbert space (Theorem~\ref{thm:freedman_ineq_bounded_W}).
Combining the two upper bounds, we can establish a recurrence relation.
Solving this relation will lead to the conclusion.

We first establish the conclusion for step sizes that depend on $T$. 
Specifically, we consider
\[
        T \geq  \frac{C_4\log^3 T}{\varepsilon^2\mu_{\min} (1-\gamma)^3}\log \frac{\abs{\gS}T}{\delta},
\]
\[
    \frac{1}{1+\frac{c_5\mu_{\min}(1-\sqrt\gamma)T}{\log^{2} T}}\leq\alpha_t\leq \frac{1}{1+\frac{c_6\mu_{\min}(1-\sqrt\gamma)t}{\log^{2} T}},
\]
where $c_5>c_6>0$ are small constants satisfying $c_5c_6\leq1/8$, and $C_4>1$ is a large constant depending only on $c_5$ and $c_6$. 
As shown in Appendix~\ref{subsection_remove_T}, once we have established the conclusion in this setting, we can recover the original conclusion stated in the theorem.

Now, we introduce the following useful quantities involving step sizes and $\gamma$
\begin{equation*}
\beta_{k}^{(t)}:=\begin{cases}
\prod_{i=1}^{t}\prn{1-\alpha_{i}\mu_{\min}(1-\sqrt\gamma)}, & \text{if }k=0,\\
\alpha_{k}\prod_{i=k+1}^{t} \prn{ 1 {-} \alpha_{i}\mu_{\min}(1 {-} \sqrt\gamma) }, & \text{if }0<k<t,\\
\alpha_{T}, & \text{if }k=t.
\end{cases}
\end{equation*}
The following lemma provides useful bounds for $\beta_{k}^{(t)}$.
\begin{lemma}\label{lem:step_size_range}
    If $c_5c_6\leq 1/8$, 
    then for all $t\geq T/\prn{c_6\log T}$, we have
    \begin{equation*}
            \beta_{k}^{(t)}\leq \begin{cases} \frac{1}{T^2}, & \text{ for}\ 0\leq k\leq\frac{t}{2};\\
            \frac{2\log^3 T}{\mu_{\min}(1-\sqrt{\gamma})T}, & \text{ for}\ \frac{t}{2}< k\leq t.
            \end{cases}
    \end{equation*}
\end{lemma}
The proof can be found in Appendix~\ref{appendix:proof_range_step_size}.
From now on, we only consider $t\geq {T}/\prn{c_6\log T}$.

Using the fact $\norm{{\bm{\Delta}}_0}\leq \sqrt{\int_0^{\prn{1-\gamma}^{-1}} dx}=\prn{1-\gamma}^{-1/2}$, the upper bound of Term (I) is given by:
\begin{align*}
\norm{\text{(I)}_t}&=\norm{\prod_{k=1}^t \brk{\gI-\alpha_k\bLamb\prn{\gI-\gT}}{\bm{\Delta}}_0}
\leq \norm{\prod_{k=1}^t \brk{\gI-\alpha_k\bLamb\prn{\gI-\gT}}}\norm{{\bm{\Delta}}_0}\\
&\leq \prn{\prod_{k=1}^t\brk{1-\alpha_k\mu_{\min}\prn{1-\sqrt\gamma}}}\norm{{\bm{\Delta}}_0}
\leq \frac{\beta_{0}^{(t)}}{\sqrt{1-\gamma}}\\
&\leq \frac{1}{\sqrt{1-\gamma}T^2},
\end{align*}
where in the second inequality, we used the following lemma.
\begin{lemma}\label{lem:effect_lr_bound} It holds that
    \[
        \norm{\gI-\alpha_k\bLamb\prn{\gI-\gT}}\leq 1-\alpha_k\mu_{\min}\prn{1-\sqrt{\gamma}}.
    \]
\end{lemma}
The proof of the lemma can be found in Appendix~\ref{subsection:proof_effect_lr_bound}.

As for Term (II), we can further decompose it into the sum of the following two terms:
\begin{align*}
    &\text{(i)}_t=\sum_{k=1}^t \alpha_k  \prn{\prod_{i=k+1}^t \brk{\gI {-} \alpha_i\bLamb\prn{\gI {-} \gT}}}\bLamb_k\prn{\gT_k {-}\gT}{\bm{\eta}}_{k-1},\\
    &\text{(ii)}_t=\sum_{k=1}^t \alpha_k \prn{\prod_{i=k+1}^t \brk{\gI {-} \alpha_i\bLamb\prn{\gI {-} \gT}}}\prn{\bLamb {-} \bLamb_k}\prn{\gI {-}\gT}{\bm{\eta}}_{k-1},
\end{align*}
and we will analyze them separately using Freedman's inequality (Theorem~\ref{thm:freedman_ineq_bounded_W}).
\begin{lemma}\label{lem_asyn_dtd_l2_term_i}
    For any $\delta\in(0,1)$, with probability at least $1-\delta$, we have for all $t\geq{T}/\prn{c_6\log T}$,
    in the case of {\NTD},
    \begin{equation*}
    \begin{aligned}
        &\norm{\text{(i)}_t}\leq 34\sqrt{\frac{\prn{\log^{3}T}\prn{\log\frac{|\gS|T}{\delta}}}{\mu_{\min}(1-\gamma)^{2}T}\prn{1+\max_{k:\,t/2< k\leq t}\norm{{\bm{\Delta}}_{k-1}}_{\bar{W}_1}}},
    \end{aligned}
\end{equation*}
the conclusion still holds for {\CTD} if we take $K>4\prn{1-\gamma}^{-1}$.
\end{lemma}
The proof can be found in Section~\ref{subsection:concentration_i}.
\begin{lemma}\label{lem_asyn_dtd_l2_term_ii}
    For any $\delta\in(0,1)$, with probability at least $1-\delta$, we have for all $t\geq{T}/\prn{c_6\log T}$,
    in both  the {\NTD} and {\CTD} cases,
    \begin{equation*}
    \begin{aligned}
        &\norm{\text{(ii)}_t}\leq 46\sqrt{\frac{\prn{\log^{3}T}\prn{\log\frac{|\gS|T}{\delta}}}{\mu_{\min}(1-\gamma)^{2}T}\prn{1+\max_{k:\,t/2< k\leq t}\norm{{\bm{\Delta}}_{k-1}}_{\bar{W}_1}}}.
        % &\leq 22\sqrt{\frac{\prn{\log^{3}T}\prn{\log\frac{|\gS|T}{\delta}}}{(1-\gamma)^{5/2}T}\prn{1+\max_{k:\,t/2< k\leq t}\norm{{\bm{\Delta}}_{k-1}}}}+\frac{12\prn{\log^{3}T}\prn{\log\frac{|\gS|T}{\delta}}}{(1-\gamma)^{3/2}T},
    \end{aligned}
\end{equation*}
\end{lemma}
The proof can be found in Section~\ref{subsection:concentration_ii}.

Combining the results, we find the following recurrence relation in terms of the $\bar{W}_1$ norm holds given the choice of $T$, with probability at least $1-\delta$, for all $t\geq{T}/\prn{c_6\log T}$
\begin{align*}
        &\norm{{\bm{\Delta}}_t}_{\bar{W}_1}\leq\frac{1}{\sqrt{1-\gamma}}\norm{{\bm{\Delta}}_t}\leq81\sqrt{\frac{\prn{\log^{3}T}\prn{\log\frac{|\gS|T}{\delta}}}{\mu_{\min}(1-\gamma)^{3}T}\prn{1+\max_{k:\,t/2< k\leq t}\norm{{\bm{\Delta}}_{k-1}}_{\bar{W}_1}}}.
\end{align*}

In Theorem~\ref{thm:solve_recurrence}, we solve the relation and obtain the error bound of the last iterate estimator:
% \begin{small}
\[ 
    \norm{{\bm{\Delta}}_T}_{\bar{W}_1}\leq C_7\prn{\sqrt{\frac{\prn{\log^{3}T}\prn{\log\frac{|\gS|T}{\delta}}}{
    \mu_{\min}(1-\gamma)^3 T}}+\frac{\prn{\log^{3}T}\prn{\log\frac{|\gS|T}{\delta}}}{\mu_{\min}(1-\gamma)^3 T}}.
\]
% \end{small}
Here $C_7>1$ is a large universal constant depending on $c_6$.
Now, we can obtain the conclusion if we take $C_4\geq 2C_7^2$ and $T \geq  {C_4}{\varepsilon^{-2} \mu_{\min}^{-1}(1-\gamma)^{-3}}\log^3 T\log \prn{{\abs{\gS}T}/{\delta}}$.

Until now, the proof has been done for {\NTD}.
For {\CTD}, we have only shown that $\bar{W}_1({\bm{\eta}}^\pi_T,{\bm{\eta}}^{\pi,K})\leq{\varepsilon}/{2}$ with high probability.
Furthermore, according to the upper bound (Eqn.~\eqref{eq:CTD_approx_err}) of the approximation error of ${\bm{\eta}}^{\pi,K}$ and Lemma~\ref{lem:prob_basic_inequalities}, if  taking $K> {4}{\varepsilon^{-2}(1 {-} \gamma)^{-3}}$, we have $\bar{W}_1\prn{{\bm{\eta}}^\pi_T,{\bm{\eta}}^{\pi}}\leq \bar{W}_1({\bm{\eta}}^\pi_T,{\bm{\eta}}^{\pi,K})+\bar{W}_1\prn{{\bm{\eta}}^{\pi,K},{\bm{\eta}}^{\pi}}\leq\varepsilon$.

\subsection{Analysis of Theorem~\ref{thm:samplecomplex_asyn_w1}}\label{Subsection_analysis_nonasymp_markov_w1}
The following proof is based on the proof of \citep[Theorem~6][]{samsonov2024improved}. 
For brevity, we will not rigorously write out the proof based on the theory of Markov chains, but rather aim to clearly articulate the idea of how to reduce the Markovian setting to the generative model setting with the stationary distribution $\mu_{\pi}$. 
For a rigorous treatment, please refer to \citep[Appendix~D][]{samsonov2024improved}.

First, we need to remove the dependency on the initial distribution $\rho\in\Delta(\gS)$.
This can be done by a standard maximal exact coupling argument; see \citep[Section 19.3][]{douc:moulines:priouret:soulier:2018} for a rigorous definition.
Specifically, we define a new Markov chain $\brc{\tilde{s}_t}_{t=0}^\infty$ in the same probability space as the Markov chain $\brc{s_t}_{t=0}^\infty$, with the initial state $\tilde{s}_0\sim\mu_\pi$ and the same Markov kernel $P^\pi$, and we define a coupling (joint distribution of the two Markov chains) between these two Markov chains by forcing them identical once they first meet.
The time of the first meet is a stopping time, denoted  $\tau$.
And we can find a ``maximal exact coupling'', so that $\tau$ can be bounded by the mixing time $\tmix$ properly.
Then, with probability at least $1-{\delta}/{3}$, $\tau\leq\tmix \log\prn{{12}/{\delta}}$ \citep[see Theorem 19.3.9 in][]{douc:moulines:priouret:soulier:2018}.
Consequently, with high probability, for all $t\geq \tmix \log\prn{{12}/{\delta}}$, the distribution of $s_t$ is exactly the stationary distribution $\mu_\pi$ and the Markov kernel remains unchanged.

Afterwards, we can apply a version of Berbee's coupling lemma \citep[Lemma~4.1][]{dedecker2002maximal} to show that for any $t\geq \tmix \log\prn{{12}/{\delta}}$, if we take the updating interval $q$ scales with $\tmix$, it holds that with high probability $\brc{s_{t+kq}}_{k=0}^\infty$ can be regarded as i.i.d.\ samples from $\mu_\pi$.
To be concrete, by \citep[Lemma~9][]{douc:moulines:priouret:soulier:2018} and the union bound, we  consider $\brc{s_{t+kq}}_{k=0}^{m-1}$ as i.i.d.\ samples from $\mu_\pi$ with probability at least $1-m\prn{{1}/{4}}^{\lfloor{q}/{\tmix} \rfloor}$ conditioned on the event $\{\tau\leq\tmix \log\prn{{12}/{\delta}} \}$.
Now we manage to reduce the Markovian setting to the generative model setting with distribution $\mu_\pi$.
By Theorem~\ref{thm:samplecomplex_syn_w1}, if we use $T^\star$ i.i.d.\ samples from $\mu_\pi$ with $T^\star\geq{C_1}{\varepsilon^{-2} \mu_{\pi,\min}^{-1}(1-\gamma)^{-3}}\log^3 T^\star\log \prn{3\abs{\gS}T^\star/\delta}$, then the error is less than $\varepsilon$ with probability at least $1-{\delta}/{3}$.
Finally, to make the probability $T^\star \prn{{1}/{4}}^{\lfloor{q}/{\tmix} \rfloor}\leq{\delta}/{3}$, we  take $q=\lceil\tmix\log\prn{3T^\star/\delta}\rceil$.
To summarize, in the Markovian setting, we need $qT^\star=\widetilde{O}\left({\tmix}{\varepsilon^{-2} \mu_{\pi,\min}^{-1}(1-\gamma)^{-3}}\right)$ samples in total to achieve the desired error bound.

\subsection{Analysis of Theorem~\ref{thm:samplecomplex_asyn_w1_variance_reduction}}\label{Subsection_analysis_nonasymp_markov_w1_variance_reduction}
As in proof of \citep[Theorem~4][]{li2021sample}, we first conduct a per-epoch analysis in which we aim to analyze how much the reference point as an estimate of $\bm{\eta}$ improves after one epoch of updates. 
The per-epoch analysis is divided into two phases based on the quality of the current reference point estimate.
Subsequently, we analyze how many epochs are required to control the final error terms.

Let $\bar{\bDelta}_e:=\bar{\bm{\eta}}_e-\bm{\eta}$ be the error of the reference point after epoch $e$.
The following lemma gives a summary of the per-epoch analysis, whose proof is given in Appendix~\ref{subsection:proof_per_epoch_analysis}.
\begin{lemma}\label{thm_per_epoch_analysis}
    For any $\delta\in(0,1)$, suppose that we take $K>4\prn{1-\gamma}^{-1}$ in {\VRCTD}. Then with probability at least $1-\delta$, we have for all $e\in [E]$, in Phase 1 (\ie, when $\norm{\bar{\bDelta}_{e-1}}> 1$) it holds that
    \begin{equation*}
        \norm{\bar{\bDelta}_{e}}\leq\frac{1}{2}\max\brc{1,\norm{\bar{\bDelta}_{e-1}}},
    \end{equation*}
    otherwise in Phase 2 (\ie, when $\norm{\bar{\bDelta}_{e-1}}\leq 1$), it holds that
        \begin{equation*}
                \norm{\bar{\bDelta}_{e}}\leq\frac{1}{2}\max\brc{\sqrt{1-\gamma}\varepsilon,\norm{\bar{\bDelta}_{e-1}}},
    \end{equation*}
\end{lemma}
It follows from the lemma  that $-(1/2)\log_2\prn{1-\gamma}$ epochs are enough to make sure $\norm{\bar{\bDelta}_{e}}\leq 1$ (i.e., entering Phase 2 from Phase 1) because $\norm{\bar{\bDelta}_0}\leq\prn{1-\gamma}^{-1/2}$.
And for Phase 2, it can also be checked that $-(1/2)\log_2\prn{\varepsilon\sqrt{1-\gamma}}$ epochs are enough to make $\norm{\bar{\bDelta}_{e+k}}\leq \sqrt{1-\gamma}\varepsilon$ when $\norm{\bar{\bDelta}_{e}}\leq 1$.
Thus, we need $E\geq-\log_2\prn{\varepsilon\sqrt{1-\gamma}}$ epochs in total such that $\norm{\bar{\bm{\Delta}}_E}\leq\sqrt{1-\gamma}\varepsilon$ holds with probability at least $1-\delta$.
The proof is ready done for {\VRNTD} by combining the result with the  inequality $\norm{\bar{\bm{\Delta}}_E}_{\bar{W}_1} \leq\prn{1 {-}\gamma}^{-1/2}\norm{\bar{\bm{\Delta}}_E}$.
But for {\VRCTD}, we need an additional step to derive the conclusion,
namely using the upper bound (Eqn.~\eqref{eq:CTD_approx_err}) of the approximation error of ${\bm{\eta}}^{\pi,K}$.

\section{Conclusions}\label{Section_discussion}
In this paper, we have studied the statistical performance of the distributional temporal difference learning (TD) from a non-asymptotic perspective.
Specifically, we have considered two instances of distributional TD, namely the non-parametric distributional TD ({\NTD}) and the categorical distributional TD ({\CTD}).
For both {\NTD} and {\CTD}, we have shown that $\wtilde O\prn{{\varepsilon^{-2}\mu_{\min}^{-1}(1-\gamma)^{-3}}}$ online interactions are sufficient to achieve a $1$-Wasserstein $\varepsilon$-optimal estimator in the generative model setting, which is minimax optimal (up to logarithmic factors).
In addition, we have generalized the result to the more challenging Markovian setting. In particular, 
we have proposed variance-reduced variations of {\NTD} and {\CTD} that  called {\VRNTD} and {\VRCTD}.
We have derived an $\widetilde{O}\left({\varepsilon^{-2} \mu_{\pi,\min}^{-1}(1-\gamma)^{-3}}+{\tmix}{\mu_{\pi,\min}^{-1}(1-\gamma)^{-1}}\right)$ sample complexity bound of both {\VRNTD} and {\VRCTD}, matching the state-of-the-art sample complexity bounds for the classic policy evaluation in the Markovian setting.
To prove these theoretical results,
we have established a novel Freedman's inequality in Hilbert spaces, which has independent theoretical value beyond the current work. 

\newpage
\appendix
\section{Omitted Results and Proofs in Section~\ref{Section_Freedman}}\label{Appendix_freedman}
\subsection{Proof of Theorem~\ref{thm:freedman_ineq}}\label{Subsection_freedman_proof_1}
\begin{proof}
For any $\lambda>0$, $t\in[0,1]$ and $j\in[n]$, let $\phi(t)=\phi_{j,\lambda}(t):=\EB_{j-1}\cosh\prn{\lambda\norm{Y_{j-1}+tX_j}}=\EB_{j-1}\cosh\prn{\lambda u(t)}$, where
$u(t):=\norm{Y_{j-1}+tX_j}$.
We aim to use the Newton-Leibniz formula to establish the relationship between $\phi(1)=\EB_{j-1}\cosh\prn{\lambda\norm{Y_j}}$ and $\phi(0)=\cosh\prn{\lambda\norm{Y_{j-1}}}$. 
This will allow us to construct a positive supermartingale $\prn{B_i}_{i=0}^n$.

Firstly, we calculate the derivative of $\phi$. In particular, 
\[
    u^\prime(t)=\frac{\inner{Y_{j-1}+tX_j}{X_j}}{u(t)},
\]
\begin{equation*}
\begin{aligned}
    \phi^{\prime}(t)&=\lambda\EB_{j-1}\brk{\sinh\prn{\lambda u(t)}u^\prime(t)}=\lambda \EB_{j-1}\brk{\sinh\prn{\lambda u(t)}\frac{\inner{Y_{j-1}+tX_j}{X_j}}{u(t)}},
\end{aligned}
\end{equation*}
\begin{equation*}
\begin{aligned}
    \phi^{\prime}(0)&=\lambda \EB_{j-1}\brk{\sinh\prn{\lambda u(0)}\frac{\inner{Y_{j-1}}{X_j}}{u(0)}}=\lambda \sinh\prn{\lambda \norm{Y_{j-1}}}\frac{\inner{Y_{j-1}}{\EB_{j-1}\brk{X_j}}}{\norm{Y_{j-1}}}=0.
\end{aligned}
\end{equation*}
Utilizing the Newton-Leibniz formula, we have
\begin{equation*}
\begin{aligned}
        \phi(1)&=\phi(0)+\int_0^1 \phi^\prime(s) ds=\phi(0)+\int_0^1 \int_0^s\phi^{\prime\prime}(t)d t ds=\phi(0)+\int_0^1 (1-t)\phi^{\prime\prime}(t) d t.
\end{aligned}
\end{equation*}
We now calculate the second order derivative of $\phi$.
\begin{equation*}
\begin{aligned}
    \phi^{\prime\prime}(t)&=\lambda\EB_{j-1}\brc{\frac{d}{dt}\brk{\sinh\prn{\lambda u(t)}u^\prime(t)}}\\
    &=\lambda \EB_{j-1}\brk{\lambda\prn{u^\prime(t)}^2\cosh\prn{\lambda u(t)}+u^{\prime\prime}(t)\sinh\prn{\lambda u(t)}}\\
    &\leq \lambda^2 \EB_{j-1}\brk{\prn{\prn{u^\prime(t)}^2+u^{\prime\prime}(t)u(t)}\cosh\prn{\lambda u(t)}}\\
    &=\frac{\lambda^2}{2} \EB_{j-1}\brk{\prn{u^2}^{\prime\prime}(t)\cosh\prn{\lambda u(t)}}\\
    &=\lambda^2 \EB_{j-1}\brk{\norm{X_j}^2\cosh\prn{\lambda\norm{Y_{j-1}+tX_j}} }\\
    &\leq \lambda^2 \cosh\prn{\lambda\norm{Y_{j-1}}}\EB_{j-1}\brk{\norm{X_j}^2\exp\prn{\lambda t \norm{X_j}} },
\end{aligned}
\end{equation*}
where in the first inequality, we used $h(x)=x\cosh(x)-\sinh(x)\geq 0$ $h^\prime(x)=x\sinh(x) \geq 0$ for any $x\geq 0$, and $h(0) = 0$.
In the third equality, we used $\prn{u^2}^{\prime\prime}(t)=2\prn{\prn{u^\prime(t)}^2+u^{\prime\prime}(t)u(t)}$.
In the fourth equality, we used 
\begin{align*}
   \prn{u^2}^{\prime\prime}(t)&=\frac{d^2}{dt^2}\norm{Y_{j-1}+tX_{j}}^2=\frac{d}{dt}\prn{2\inner{Y_{j-1}+tX_j}{X_j}}=2\norm{X_j}^2 .
\end{align*}
In the last inequality, we used
\begin{equation*}
    \cosh\prn{\lambda\norm{Y_{j-1}+tX_j}}\leq \cosh\prn{\lambda\norm{Y_{j-1}}}\exp\prn{\lambda t \norm{X_j}},
\end{equation*}
which follows from that
\begin{equation*}
    \begin{aligned}
        \exp\prn{\lambda\norm{Y_{j-1}+tX_j}}&\leq\exp\brc{\lambda\prn{\norm{Y_{j-1}}+ t\norm{X_j}}}=\exp\prn{\lambda\norm{Y_{j-1}}} \exp\prn{\lambda t \norm{X_j}},
        \end{aligned}
\end{equation*}
\begin{equation*}
    \begin{aligned}
        \exp\prn{-\lambda\norm{Y_{j-1}+tX_j}}&\leq\exp\brc{-\lambda\prn{\norm{Y_{j-1}}- t\norm{-X_j}}}=\exp\prn{-\lambda\norm{Y_{j-1}}} \exp\prn{\lambda t \norm{X_j}}.
    \end{aligned}
\end{equation*}
Hence, we can derive the following inequality for all $j\in[n]$
\begin{small}
\begin{equation}\label{eq:supermartingale}
    \begin{aligned}
        &\EB_{j-1}\brk{\cosh\prn{\lambda\norm{Y_{j}}}}=\phi(1)\\
        &\qquad= \phi(0)+\int_0^1(1-t)\phi^{\prime\prime}(t)dt\\
        &\qquad\leq \cosh\prn{\lambda\norm{Y_{j-1}}}+\lambda^2 \cosh\prn{\lambda\norm{Y_{j-1}}}\EB_{j-1}\brk{\norm{X_j}^2\int_0^1 (1-t)\exp\prn{\lambda t \norm{X_j}} dt}\\
        &\qquad= \cosh\prn{\lambda\norm{Y_{j-1}}}+\lambda^2 \cosh\prn{\lambda\norm{Y_{j-1}}}\EB_{j-1}\brk{\norm{X_j}^2\frac{\exp\prn{\lambda\norm{X_j}}-\lambda\norm{X_j}-1}{\lambda^2\norm{X_j}^2}}\\
        &\qquad= \EB_{j-1}\brk{\exp\prn{\lambda\norm{X_j}}-\lambda\norm{X_j}}\cosh\prn{\lambda\norm{Y_{j-1}}}\\
        &\qquad= \EB_{j-1}\brk{1+\sum_{k=0}^\infty\frac{1}{(k+2)!}\prn{\lambda\norm{X_j}}^{k+2}}\cosh\prn{\lambda\norm{Y_{j-1}}}\\
        &\qquad \leq  \EB_{j-1}\brk{1+\frac{\lambda^2\norm{X_j}^2}{2}\sum_{k=0}^\infty \prn{\frac{\lambda b}{3}}^{k}}\cosh\prn{\lambda\norm{Y_{j-1}}}\\
        &\qquad= \prn{1+\frac{\lambda^2\sigma_j^2}{2(1-\lambda b/3)}}\cosh\prn{\lambda\norm{Y_{j-1}}}\\
        &\qquad \leq  \exp\brc{\frac{\lambda^2\sigma_j^2}{2(1-\lambda b/3)}}\cosh\prn{\lambda\norm{Y_{j-1}}},
    \end{aligned}
\end{equation}
\end{small}
which holds for any $\lambda\in(0,\frac{3}{b})$, where we used Taylor expansion $e^x=\sum_{k=0}^\infty \frac{x^k}{k!}$ for $x\in\RB$, and $\frac{1}{1-x}=\sum_{k=0}^\infty x^k$ for $x\in (-1,1)$.
And in the second inequality, we used $(k+2)!\geq 2(3^k)$ and $\norm{X_j}\leq b$.

Let $B_0:=1$, $B_i:=\exp\brc{-\frac{\lambda^2 W_i}{2(1-\lambda b/3)}}\cosh\prn{\lambda \norm{Y_i}}$. Then
\begin{equation*}
    \begin{aligned}
        \EB_{i-1}\brk{B_i}&=\exp\brc{-\frac{\lambda^2 W_{i-1}}{2(1-\lambda b/3)}}\exp\brc{-\frac{\lambda^2\sigma_i^2}{2(1-\lambda b/3)}}\EB_{i-1}\brk{\cosh\prn{\lambda\norm{Y_{i}}}}\\
        &\leq\exp\brc{-\frac{\lambda^2 W_{i-1}}{2(1-\lambda b/3)}}\cosh\prn{\lambda\norm{Y_{i-1}}}\\
        &=B_{i-1},
    \end{aligned}
\end{equation*}
showing that $\prn{B_i}_{i=0}^n$ is positive supermartingale.
By the optional stopping theorem \citep[Theorem 4.8.4][]{durrett2019probability}, for any stopping time $\tau$, we have $\EB\brk{B_\tau}\leq \EB\brk{B_0}=1$.

Let $\tau:=\inf\brc{k\in[n]:\norm{Y_k}\geq \varepsilon}$ be a stopping time, and $\inf \emptyset:=\infty$.
Define an event
\begin{equation*}
    A:=\brc{\exists k\in[n], \text{s.t.}\ \norm{Y_k}\geq \varepsilon \text{ and }W_k\leq \sigma^2}.
\end{equation*}
Then on $A$, we have $\tau<\infty$, $\norm{Y_\tau}\geq \varepsilon$ and $W_\tau\leq\sigma^2$, noting that $W_k$ is non-decreasing with $k$.
Our goal is to provide an upper bound for $\PB(A)$. Note that 
    \begin{align*}
        \PB(A)&=\EB\brk{\sqrt{B_\tau}\frac{1}{\sqrt{B_\tau}}\mathds{1}(A)}\\
        &\leq \sqrt{\EB\brk{B_\tau}\EB\brk{\frac{1}{B_\tau}\mathds{1}(A)} }\\
        &\leq\sqrt{\EB\brk{\frac{\exp\brc{\frac{\lambda^2 W_\tau}{2(1-\lambda b/3)}}}{\cosh\prn{\lambda \norm{Y_\tau}}}\mathds{1}(A)}}\\
        &\leq\sqrt{\EB\brk{\frac{\exp\brc{\frac{\lambda^2\sigma^2}{2(1-\lambda b/3)}}}{\cosh\prn{\lambda \varepsilon}}\mathds{1}(A)}}\\
        &\leq \sqrt{2\exp\brc{-\lambda\varepsilon+\frac{\lambda^2\sigma^2}{2(1-\lambda b/3)}}\PB(A)},
    \end{align*}
where in the first inequality, we used Cauchy-Schwartz inequality.
In the second inequality, we used $\EB\brk{B_\tau}\leq 1$.
In the third inequality, we used $\norm{Y_\tau}\geq \varepsilon$ and $W_\tau\leq\sigma^2$ on $A$, and $\cosh(x)$ is increasing when $x\geq 0$.
In the last inequality, we used $\cosh(x)\geq \frac{1}{2}e^{x}$.

Hence for any $\lambda\in(0,\frac{3}{b})$,
\begin{equation*}
    \PB(A)\leq 2\exp\brc{-\lambda\varepsilon+\frac{\lambda^2\sigma^2}{2\prn{1-\lambda b/3}}}.
\end{equation*}
We  choose $\lambda^\star=\frac{\varepsilon}{\sigma^2+\varepsilon b/3}\in(0,\frac{3}{b})$. Hence, 
\begin{align*}
        \PB(A)&\leq 2\exp\brc{-\lambda^\star\varepsilon+\frac{\prn{\lambda^\star}^2\sigma^2}{2\prn{1-\lambda^\star b/3}}}\\
        &= 2\exp\brc{-\frac{\varepsilon^2}{\sigma^2 {+} \varepsilon b/3} {+} \frac{\sigma^2}{2\prn{1{-}\frac{\varepsilon b/3}{\sigma^2 {+} \varepsilon b/3} }}\frac{\varepsilon^2}{\prn{\sigma^2 {+} \varepsilon b/3}^2}}\\
        &=2\exp\brc{-\frac{\varepsilon^2/2}{\sigma^2+\varepsilon b/3}},
\end{align*}
which is the desired conclusion.
\end{proof}

\subsection{Proof of Theorem~\ref{thm:freedman_ineq_bounded_W}}\label{Subsection_freedman_proof_2}
\begin{proof}
According to Theorem~\ref{thm:freedman_ineq}, for any $\varepsilon, \tilde{\sigma}>0$, we have
\begin{equation*}
\begin{aligned}
    \PB\big(\exists k\in[n],  \: & \norm{Y_k} \geq  \varepsilon  \mbox{ and } W_k \leq  \tilde{\sigma}^2\big) \leq  2\exp\brc{-\frac{\varepsilon^2/2}{\tilde{\sigma}^2 {+} b\varepsilon/3}}.
\end{aligned}
\end{equation*}
We can check that when $\varepsilon=\sqrt{4\tilde{\sigma}^2\log\frac{2}{\delta}}+\frac{4}{3}b\log \frac{2}{\delta}$,  the upper bound on RHS is less than $\delta$. Hence, 
\begin{small}
\begin{equation}\label{eq:freedman_ineq_for_proof}
    \begin{aligned}
        \PB \prn{\exists k\in[n],  \: \norm{Y_{k}} \geq \sqrt{4\tilde{\sigma}^2\log\frac{2}{\delta}} + \frac{4}{3}b\log \frac{2}{\delta}\text{ and }W_{k}\leq\tilde{\sigma}^{2}}\leq\delta.
    \end{aligned}
\end{equation}
\end{small}
For each $k\in[n]$, define the events
\begin{small}
\begin{equation*}
\begin{aligned}
    &\gH_H^{(k)}:=\brc{\norm{Y_{k}}\geq\sqrt{8\max\Big\{ W_{k},\frac{\sigma^{2}}{2^{H}}\Big\}\log\frac{2H}{\delta}}+\frac{4}{3}b\log\frac{2H}{\delta}},\\
    &\gB_{H,H}^{(k)}:= \brc{\norm{Y_{k}} \geq \sqrt{4\frac{\sigma^{2}}{2^{H - 1}} \log \frac{2H}{\delta}} + \frac{4}{3}b\log \frac{2H}{\delta}\ \text{and}\ W_{k} \leq \frac{\sigma^{2}}{2^{H-1}} },\\
    &\gB_{h,H}^{(k)}:=\Bigg\{\norm{ Y_{k}}\geq\sqrt{4\frac{\sigma^{2}}{2^{h-1}}\log\frac{2H}{\delta}}+\frac{4}{3}b\log\frac{2H}{\delta}\text{ and }\frac{\sigma^{2}}{2^{h}}\le W_{k}\leq\frac{\sigma^{2}}{2^{h-1}} \Bigg\} ,\qquad1\leq h\leq H-1.
\end{aligned}
\end{equation*}
\end{small}
By the definition, we only need to show $\PB\prn{\bigcup_{k\in[n]}\gH^{(k)}_H}\leq \delta$.
Since $W_k\leq W_n\leq \sigma^2$ almost surely, we can find that $\gH_{H}^{(k)}\subseteq\bigcup_{h\in[H]}\gB_{h,H}^{(k)}$ (we will justify this later).
Then $\bigcup_{k\in[n]}\gH_{H}^{(k)}\subseteq\bigcup_{h\in[H]}\bigcup_{k\in[n]}\gB_{h,H}^{(k)}$.
By the inequality \eqref{eq:freedman_ineq_for_proof} with $\tilde{\sigma}^{2}=\frac{\sigma^{2}}{2^{h-1}}$ and setting $\delta$ as $\frac{\delta}{H}$, 
we have $\PB\prn{\bigcup_{k\in[n]}\gB_{h,H}^{(k)}} \leq\frac{\delta}{H}$ for all $h\in[H]$. 
By the union bound, we can arrive at the conclusion:
\begin{equation*}
    \PB\prn{\bigcup_{k\in[n]}\gH^{(k)}_H}\leq\sum_{h=1}^{H}\PB\prn{\bigcup_{k\in[n]}\gB_{h,H}^{(k)}}\leq\delta.
\end{equation*}
To justify $\gH_{H}^{(k)}\subseteq\bigcup_{h\in[H]}\gB_{h,H}^{(k)}$, we can consider the decomposition $\gH_{H}^{(k)}=\bigcup_{h\in[H]}\prn{\gH_{H}^{(k)}\cap\gC_{h,H}^{(k)}}$,
where
\begin{small}
\begin{equation*}
\begin{aligned}
    &\gC_{H,H}^{(k)}:=\brc{W_{k}\leq\frac{\sigma^{2}}{2^{H-1}}},\qquad\gC_{h,H}^{(k)}:=\brc{\frac{\sigma^{2}}{2^{h}}\le W_{k}\leq\frac{\sigma^{2}}{2^{h-1}}},\quad 1\leq h\leq H-1.
\end{aligned}
\end{equation*}
\end{small}
The decomposition holds because $W_k\leq W_n\leq \sigma^2$ almost surely.
We only need to show that $\gH_{H}^{(k)}\cap\gC_{h,H}^{(k)}\subseteq \gB_{h,H}^{(k)}$ for each $h\in[H]$.
On the event $\gH_{H}^{(k)}\cap\gC_{h,H}^{(k)}$, we have
\begin{align*}
    \norm{Y_{k}}&\geq\sqrt{8\max\Big\{ W_{k},\frac{\sigma^{2}}{2^{H}}\Big\}\log\frac{2H}{\delta}}+\frac{4}{3}b\log\frac{2H}{\delta}\geq\sqrt{4\frac{\sigma^{2}}{2^{h-1}}\log\frac{2H}{\delta}}+\frac{4}{3}b\log\frac{2H}{\delta},
\end{align*}
hence $\gH_{H}^{(k)}\cap\gC_{h,H}^{(k)}\subseteq \gB_{h,H}^{(k)}$.
\end{proof}

\section{Omitted Results and Proofs in Section~\ref{Section_proof_outlines}}\label{Appendix_proof}
\subsection{Proof of Theorem~\ref{thm:minimax_lower_bound_w1}}\label{Appendix_minimax}
\begin{proof}
For any ${\bm{\eta}}\in\sP^{D}$, we define $\gV({\bm{\eta}})\in\RB^{D}$ as the entry-wise expectation of ${\bm{\eta}}$.
It is easy to check that $\gV({\bm{\eta}}^\pi_M)={\bm{V}}^\pi_M$.
Recall the dual representation of the $1$-Wasserstein metric \citep[Corollary 5.16][]{villani2009optimal}, for any $\nu_1,\nu_2\in\sP,$
\begin{equation*}
    W_1(\nu_1,\nu_2)=   \sup_{f\colon f\text{ is 1-Lipschitz}}\abs{\EB_{X\sim\nu_1}\brk{f(X)} -\EB_{Y\sim\nu_2}\brk{f(Y)}}, 
\end{equation*}
we have
$\bar{W}_1\prn{\hat{{\bm{\eta}}},{\bm{\eta}}^\pi_M}\geq \norm{\gV(\hat{{\bm{\eta}}})-V^\pi_M}_{\infty}$.
Hence
\begin{equation*}
    \begin{aligned}
    \inf_{\hat{{\bm{\eta}}}}\sup_{M\in\sM(D)}\sup_{\pi}\EB\brk{\bar{W}_1\prn{\hat{{\bm{\eta}}},{\bm{\eta}}^\pi_M}}&\geq\inf_{\hat{{\bm{\eta}}}}\sup_{M\in\sM(D)}\sup_{\pi}\EB\brk{\norm{\gV(\hat{{\bm{\eta}}})-{\bm{V}}^\pi_M}_{\infty}}\\
    &\geq \inf_{\hat{{\bm{V}}}}\sup_{M\in\sM(D)}\sup_{\pi}\EB\brk{\norm{\hat{{\bm{V}}}-{\bm{V}}^\pi_M}_{\infty}}\\
    &\geq \frac{c}{(1-\gamma)^{3/2}}\sqrt{\frac{\log\frac{D}{2}}{T}},
    \end{aligned}
\end{equation*}
where the second inequality holds because $\gV\prn{\hat{{\bm{\eta}}}}\in\RB^D$ is also a measurable function of $T$ samples from the generative model, and the infimum $\hat{{\bm{V}}}\in\RB^{D}$ ranges over all measurable functions of $T$ samples from the generative model.
The last inequality is due to \citep[Theorem 2(b)][]{pananjady2020instance}.
\end{proof}
In our main result (Theorem~\ref{thm:samplecomplex_syn_w1}), we assume only one sample from $\mu$ is accessible, and an additional $\frac{1}{\mu_{\min}}$ appears in the sample complexity bound accordingly.
According to \citep[Example~G.1][]{zhang2023estimation}, one can find that the factor $\frac{1}{\mu_{\min}}$ is inevitable, and the minimax lower bound becomes $\tilde{\Omega}\prn{\frac{1}{\varepsilon^{2}\mu_{\min}(1-\gamma)^{3}}}$.

\subsection{Range of Iterations}
\begin{lemma}[Range of ${\bm{\eta}}^\pi_t$]
\label{lem:range_eta}
Suppose that $\alpha_t\in[0, 1]$ for all $t\geq 0$, and
 ${\bm{\eta}}^\pi_0\in\sP^\gS$. Then we have, for all $t\geq 0$, ${\bm{\eta}}^\pi_t\in\sP^\gS$ in the case of {\NTD}.
Similarly, if ${\bm{\eta}}^\pi_0\in\sP^\gS_K$, then we have, for all $t\geq 0$, ${\bm{\eta}}^\pi_t\in\sP^\gS_K$ in the case of {\CTD}.
\end{lemma}
\begin{proof}
We will only prove the case of {\NTD}, and the proof for {\CTD} is similar by utilizing the property of the projection operator $\bm{\Pi}_K\colon \sP^\gS\to\sP^\gS_K$.
And we only consider the synchronous setting, the same conclusion holding for the asynchronous setting.  

We prove the result by induction. 
It is trivial that ${\bm{\eta}}^\pi_t\in\sP^\gS$ for $t=0$.
Suppose that ${\bm{\eta}}^\pi_{t-1}\in\sP^\gS$, and recall the updating scheme of {\NTD}
\begin{equation*}
    {\bm{\eta}}^\pi_t=\prn{\gI-\alpha_t\bLamb_t}{\bm{\eta}}^\pi_{t-1}+\alpha_t\bLamb_t\gT^\pi_t{\bm{\eta}}^\pi_{t-1}.
\end{equation*}
It is evident that $\sP^\gS$ is a convex set, considering that $\sP^\gS$ is a subset of the linear space $\gM^\gS$ introduced in Section~\ref{subsection:signed_measure}. 
Therefore, we only need to show that $\gT^\pi_t{\bm{\eta}}^\pi_{t-1}\in \sP^\gS$, which trivially holds because $\gT^\pi_t$ is a random operator mapping from $\sP^\gS$ to $\sP^\gS$, and ${\bm{\eta}}^\pi_{t-1}\in\sP^\gS$.
Applying the induction argument, we can arrive at the conclusion.
\end{proof}

\subsection{Remove the Dependence on T for Step Sizes}\label{subsection_remove_T}
We have shown that the conclusion holds for
\begin{equation*}
        T \geq  \frac{C_4\log^3 T}{\varepsilon^2 \mu_{\min}(1-\gamma)^3}\log \frac{\abs{\gS}T}{\delta},
\end{equation*}
\begin{equation*}
    \frac{1}{1+\frac{c_5\mu_{\min}(1-\sqrt\gamma)T}{\log^{2} T}}\leq\alpha_t\leq \frac{1}{1+\frac{c_6\mu_{\min}(1-\sqrt\gamma)t}{\log^{2} T}},
\end{equation*}
where $c_5c_6\leq\frac{1}{8}$, $c_5>c_6>0$ and $C_4>0$. 

Then for some $c_2>c_3>0$ to be determined,  we now assume
\begin{equation*}
    \frac{1}{1+\frac{c_2\mu_{\min}(1-\sqrt\gamma)t}{\log^{2} t}}\leq\alpha_t\leq \frac{1}{1+\frac{c_3\mu_{\min}(1-\sqrt\gamma)t}{\log^{2} t}}.
\end{equation*}
Next, we will show that if we consider the result of the $\frac{T}{2}$-th iteration with this step size scheme as the initialization of a new iteration process, then the step sizes in the subsequent $\frac{T}{2}$ iterations lie in the previously established range. 
If this is done, the conclusion still holds if we choose $T\geq\frac{2C_4\log^3 T}{\varepsilon^2 \mu_{\min}(1-\gamma)^3}\log \frac{\abs{\gS}T}{\delta}$, because the initialization ${\bm{\eta}}^\pi_{T/2}\in\sP^\gS$ (or $\sP^\gS_K$ in the case of {\CTD}) is independent of the samples obtained for $\frac{T}{2}<t\leq T$.

For any $\frac{T}{2}< t\leq T$, we denote $\tau:=t-\frac{T}{2}$. We can see that there exist $c_2>c_3>0$, such that the last inequality in both of the following lines hold simultaneously, which is expected. 
\begin{equation*}
\begin{aligned}
\tilde{\alpha}_\tau&:=\alpha_t\leq \frac{1}{1+\frac{c_3\mu_{\min}(1-\sqrt\gamma)(\tau+T/2)}{\log^{2}(\tau+T/2) }}\leq\frac{1}{1+\frac{c_3\mu_{\min}(1-\sqrt\gamma)\tau}{\log^{2} T}}\leq\frac{1}{1+\frac{c_6\mu_{\min}(1-\sqrt\gamma)\tau}{\log^{2} (T/2)}},
\end{aligned}
\end{equation*}
and
\begin{equation*}
\begin{aligned}
        \tilde{\alpha}_\tau&=\alpha_t\geq\frac{1}{1+\frac{c_2\mu_{\min}(1-\sqrt\gamma)(\tau+T/2)}{\log^{2} (\tau+T/2)}}\geq\frac{1}{1+\frac{2c_2\mu_{\min}(1-\sqrt\gamma)T/2}{\log^{2} (T/2)}}\geq \frac{1}{1+\frac{c_5\mu_{\min}(1-\sqrt\gamma)T/2}{\log^{2} (T/2)}}.
\end{aligned}
\end{equation*}

\subsection{Proof of Lemma~\ref{lem:step_size_range}}\label{appendix:proof_range_step_size}
\begin{proof} Note that
\begin{small}
    \begin{equation*}
(1-\sqrt{\gamma})\alpha_{t}\geq\frac{1-\sqrt{\gamma}}{1+\frac{c_{5}\mu_{\min}(1-\sqrt{\gamma})T}{\log^{2}T}}\geq\frac{1-\sqrt{\gamma}}{\frac{2c_{5}\mu_{\min}(1-\sqrt\gamma)T}{\log^{2}T}}=\frac{\log^{2}T}{2c_{5}\mu_{\min}T}.
\end{equation*}
\end{small}
For any $0\leq k\leq\frac{t}{2}$, then
\begin{equation*}
\begin{aligned}
\beta_{k}^{(t)} & \leq\brk{1-\alpha_{t/2}\mu_{\min}(1-\sqrt{\gamma})}^{t/2}\\
&\leq\prn{1-\frac{\log^{2}T}{2c_{5}T}}^{t/2}\\
&\leq\prn{1-\frac{\log^{2}T}{2c_{5}T}}^{\frac{T}{2c_{6}\log T}}\\
 & =\brc{ \prn{1-\frac{\log^{2}T}{2c_{5}T}}^{\frac{2c_{5}T}{\log^{2}T}}} ^{\frac{\log T}{4c_{5}c_{6}}}\\
 &\leq\frac{1}{T^{2}},
\end{aligned}
\end{equation*}
where in the last inequality we used $c_{5}c_{6}\leq \frac{1}{8}$. 

And for any $\frac{t}{2}<k\leq t$, we have 
\begin{equation*}
\beta_{k}^{(t)}\leq\alpha_{k}\leq\frac{1}{\frac{c_{6}\mu_{\min}(1-\sqrt\gamma)k}{\log^{2}T}}\leq\frac{2\log^{3}T}{\mu_{\min}(1-\sqrt\gamma)T}.
\end{equation*}
\end{proof}

\subsection{Proof of Lemma~\ref{lem:effect_lr_bound}}\label{subsection:proof_effect_lr_bound}
\begin{proof}
We denote $\bLamb=\operatorname{diag}\brc{\prn{\lambda_s}_{s\in\gS}}$, then $\mu_{\min}=\min_{s\in\gS}\lambda_s$. Accordingly, 
\begin{equation*}
\begin{aligned}
\norm{\gI-\alpha_k\bLamb\prn{\gI-\gT}}=&\sup_{{\bm{\xi}}\colon\norm{{\bm{\xi}}}=1}\norm{\brk{\gI-\alpha_k\bLamb\prn{\gI-\gT}}{\bm{\xi}}}\\
=&\sup_{{\bm{\xi}}\colon\norm{{\bm{\xi}}}=1}\max_{s\in\gS}\norm{\brc{\brk{\gI-\alpha_k\bLamb\prn{\gI-\gT}}{\bm{\xi}}}(s)}\\
=&  \sup_{{\bm{\xi}}\colon\norm{{\bm{\xi}}}=1}  \max_{s\in\gS}\norm{(1-\alpha_k\lambda_s)\xi(s)+\alpha_k\lambda_s\sum_{s^\prime\in\gS}\gT(s,s^\prime)\xi(s^\prime)}\\
\leq&   \sup_{{\bm{\xi}}\colon\norm{{\bm{\xi}}}=1}  \max_{s\in\gS} \brk{ (1 - \alpha_k\lambda_s) \norm{\xi(s)} + \alpha_k\lambda_s\sum_{s^\prime\in\gS} \norm{\gT(s,s^\prime)\xi(s^\prime)} }\\
\leq&   \sup_{{\bm{\xi}}\colon\norm{{\bm{\xi}}}=1} \max_{s\in\gS}  \brk{ (1 - \alpha_k\lambda_s) + \alpha_k\lambda_s\sqrt{\gamma}\sum_{s^\prime\in\gS}P(s,s^\prime)\norm{\xi(s^\prime)} }\\
\leq &\max_{s\in\gS}\brk{1-\alpha_k\lambda_s\prn{1-\sqrt{\gamma}}}\\
\leq &1-\alpha_k\mu_{\min}\prn{1-\sqrt{\gamma}}.
\end{aligned}
\end{equation*}
\end{proof}

\subsection{Concentration of the Martingale Term (i)}\label{subsection:concentration_i}
\begin{proof}[Proof of Lemma~\ref{lem_asyn_dtd_l2_term_i}]
We first show that the inequality holds for each $t\geq\frac{T}{c_6\log T}$ and then apply the union bound.
We denote 
\begin{equation*}
    {\bm{\zeta}}_k^{(1)}=\alpha_k  \prn{\prod_{i=k+1}^t \brk{\gI-\alpha_i\bLamb\prn{\gI {-} \gT}}}\bLamb_k\prn{\gT_k {-} \gT}{\bm{\eta}}_{k-1},
\end{equation*}
where we omit the dependency on $t$ in the notation for brevity, then LHS in the lemma equals $\norm{\sum_{k=1}^t{\bm{\zeta}}_k^{(1)}}$.
Let $\gF_k$ denote the $\sigma$-field that contains all information up to time step $k$. Then $\brc{\zeta_k^{(1)}(s)}_{k=1}^t$ is a $\brc{\gF_k}_{k=1}^t$-martingale difference sequence in the Hilbert space $\gM$:
\begin{small}
\begin{equation*}
\begin{aligned}
        \EB_{k-1} \brk{{\bm{\zeta}}_{k}^{(1)}} &= \alpha_k \prod_{i=k+1}^t  \brk{\gI - \alpha_i\bLamb\prn{\gI - \gT}}\EB_{k-1} \brk{\bLamb_k\prn{\gT_k - \gT}{\bm{\eta}}_{k-1}}=0,
\end{aligned}
\end{equation*}
\end{small}
the first equality holds because a Bochner integral can be exchanged with a bounded linear operator , and the second equality holds due to the definition of $\bLamb_k$ and the empirical distributional Bellman operator $\gT_k$.

We hope to use Freedman's inequality (Theorem~\ref{thm:freedman_ineq_bounded_W}) to bound this martingale.

The norm of the martingale difference $\norm{{\bm{\zeta}}_k^{(1)}}$ can be bounded as follow
\begin{small}
\begin{equation*}
\begin{aligned}
    \norm{{\bm{\zeta}}_k^{(1)}}&\leq   \alpha_k\norm{\prod_{i=k+1}^t \brk{\gI-\alpha_i\bLamb\prn{\gI-\gT}}}\norm{\bLamb_k} \norm{\prn{\gT_k-\gT}{\bm{\eta}}_{k-1}}\\
    &\leq \alpha_k\prn{\prod_{i=k+1}^t \brk{1-\alpha_i\mu_{\min}\prn{1-\sqrt\gamma}}} \frac{1}{\sqrt{1-\gamma}}\\
    &=\frac{\beta_{k}^{(t)}}{\sqrt{1-\gamma}}.
\end{aligned}
\end{equation*}
\end{small}
In the second inequality, we used Lemma~\ref{lem:effect_lr_bound} and $\norm{\bLamb_k}\leq 1$.
Hence, $\max_{k\in [t]}\norm{{\bm{\zeta}}_k^{(1)}}\leq \frac{\max_{k\in[t]}\beta_k^{(t)}}{\sqrt{1-\gamma}}\leq \frac{1}{\sqrt{1-\gamma}}\max\brc{\frac{1}{T^2},\frac{2\log^3 T}{\mu_{\min}(1-\sqrt\gamma)T}}\leq\frac{4\log^3 T}{\mu_{\min}(1-\gamma)^{3/2} T}=:b^{(1)}$, according to Lemma~\ref{lem:step_size_range}.

Now, let's calculate the quadratic variation.

We first introduce some notations.
For any $k\in\NB$, we denote 
% $\var(\nu):=\EB\brk{\norm{\nu}^2}$, $\var_t(\nu):=\EB_t\brk{\norm{\nu}^2}$ for any random $\nu\in\gM$, and
$\var(\bm{\xi}):=\prn{\EB\brk{\norm{\xi(s)}^2}}_{s\in\gS}\in\RB^\gS$, $\var_k(\bm{\xi}):=\prn{\EB_k\brk{\norm{\xi(s)}^2}}_{s\in\gS}\in\RB^\gS$ for any random element $\bm{\xi}$ in $\gM^\gS$.
We will see that, our core task is to bound $\var_{k-1}\prn{{\bm{\zeta}}_k^{(1)}}$.

To facilitate theoretical analysis, we assume in $t$-th iteration, we have the following virtual samples which are not used in the algorithms: for each $s\in\gS$, $\tilde{a}_t(s)\sim\pi(\cdot|s), \tilde{r}_t(s)\sim \gP_R(\cdot|s,\tilde{a}_t(s)), \tilde{s}_t^\prime(s)\sim P(\cdot|s,\tilde{a}_t(s))$.
Then, we can define $\brk{\gT_t{\bm{\xi}}}(s)$ using the tuple $(s, \tilde{a}_t(s), \tilde{r}_t(s), \tilde{s}_t^\prime(s))$ for any ${\bm{\xi}}\in\gM^\gS$ and $s\neq s_t$. 

For any ${\bm{\xi}}\in\gM^\gS$, we define its one-step update Cram\'er variation as $\bm{\sigma}({\bm{\xi}}):=\var\prn{(\what{\gT}-\gT){\bm{\xi}}}\in\RB^\gS$, where $\what{\gT}$ is a random operator and has the same distribution as $\gT_1$.

For any $\bm{x},\bm{y}\in\RB^\gS$, we say $\bm{x}\leq \bm{y}$ if $x(s)\leq y(s)$ for all $s\in\gS$.
In this part, $\norm{\bm{x}}:=\norm{\bm{x}}_\infty=\max_{s\in\gS}\abs{x(s)}$, $\sqrt{\bm{x}}:=\prn{\sqrt{x(s)}}_{s\in\gS}$.
And for any $\bU\in\RB^{\gS\times\gS}$, $\norm{\bU}:=\norm{\bU}_\infty=\sup_{\bm{x}\in\RB^\gS, \norm{\bm{x}}=1}\norm{\bU\bm{x}}=\max_{s\in\gS}\sum_{s^\prime\in\gS}\abs{U(s,s^\prime)}$.

For any $\brc{\bm{x}_k}_{k=1}^n\subset \RB^\gS$, we abuse the notations to denote $\max_{k\in[n]}\bm{x}_k$ as the element-wise maximum $\prn{\max_{k\in[n]}x_k(s)}_{s\in\gS}$.

We denote $\bI\in\RB^{\gS\times\gS}$ as the identity matrix, $\bm{1}\in\RB^\gS$ as the all-ones vector, and $\bP:=P^\pi\in\RB^{\gS\times\gS}$, \ie, $P(s,s^\prime):=P^\pi(s^\prime|s)=\sum_{a\in\gA}\pi(a|s)P(s^\prime|s,a)$.

Also, we let $\bP_t$ be the empirical version of $\bP$ corresponding to $\gT_t$, for any $k\in\NB$. Namely, $P_t\prn{s_t,s^\prime}=\delta_{s^\prime,s_t^\prime}$ and $P_t(s,s^\prime)=\delta_{s^\prime,\tilde{s}_t^\prime(s)}$ for any $s\neq s_t$ and $s^\prime\in\gS$.

With these notations, the quadratic variation is $\bW_t^{(1)}:=\sum_{k=1}^t\var_{k-1}\prn{{\bm{\zeta}}_k^{(1)}}$.
To bound the quadratic variation $\bW_t^{(1)}$, we need to bound $\var_{k-1}\prn{{\bm{\zeta}}_k^{(1)}}$ using the following lemma, whose proof can be found in Appendix~\ref{subsection:proof_zeta1_var_bound}.
\begin{lemma}\label{lem:zeta1_var_bound}
\begin{small}
\begin{equation*}
            \var_{k-1}\prn{{\bm{\zeta}}_k^{(1)}}\leq  \alpha_k \beta_k^{(t)}\prn{\prod_{i=k+1}^t \brk{\bI-\alpha_i\bLamb\prn{\bI-\sqrt\gamma\bP}}} \bLamb\bm{\sigma}\prn{{\bm{\eta}}_{k-1}}.
\end{equation*}
\end{small}
\end{lemma}
Hence, the quadratic variation $\bW_t^{(1)}:=\sum_{k=1}^t\var_{k-1}\prn{{\bm{\zeta}}_k^{(1)}}$ can be bounded as follow:
\begin{small}
\begin{equation*}\label{eq:asyn_td_upper_bound_of_Wt1}
    \begin{aligned}
        \bW_t^{(1)}=&\sum_{k=1}^{t}\var_{t-1}\prn{{\bm{\zeta}}_k^{(1)}}\\
        \leq&\sum_{k=1}^{t} \alpha_k \beta_k^{(t)}\prn{\prod_{i=k+1}^t \brk{\bI-\alpha_i\bLamb\prn{\bI-\sqrt\gamma\bP}}}\bLamb \bm{\sigma}\prn{{\bm{\eta}}_{k-1}}\\
\leq& \sum_{k=1}^{t/2}\alpha_k\beta_{k}^{(t)}\norm{\prod_{i=k+1}^t \brk{\bI-\alpha_i\bLamb\prn{\bI-\sqrt\gamma\bP}}}\norm{\bLamb}\norm{\bm{\sigma}\prn{{\bm{\eta}}_{k-1}}}  \bm{1}
	+\sum_{k=t/2+1}^{t}\alpha_k \beta_k^{(t)}\prn{\prod_{i=k+1}^t \brk{\bI-\alpha_i\bLamb\prn{\bI-\sqrt\gamma\bP}}}\bLamb \bm{\sigma}\prn{{\bm{\eta}}_{k-1}}\\
 \leq&\sum_{k=1}^{t/2}\prn{\beta_{k}^{(t)}}^{2}\frac{1}{1-\gamma}\bm{1} 
 + \prn{\max_{k:\,t/2<k\leq t} \beta_{k}^{(t)}}\sum_{k=t/2+1}^{t}\alpha_k \prn{\prod_{i=k+1}^t \brk{\bI-\alpha_i\bLamb\prn{\bI-\sqrt\gamma\bP}}} \bLamb\bm{\sigma}\prn{{\bm{\eta}}_{k-1}}\\
  \leq&\frac{1}{2(1-\gamma)T^{3}} \bm{1} +\frac{2 \log^{3}T}{\mu_{\min}(1-\sqrt\gamma)T}\brc{\sum_{k=t/2+1}^{t}\alpha_{k}\prod_{i=k+1}^t \brk{\bI-\alpha_i\bLamb\prn{\bI-\sqrt\gamma\bP}}}\bLamb \max_{k:\,t/2< k\leq t}\bm{\sigma}\prn{{\bm{\eta}}_{k-1}}\\
 \leq&\frac{1}{2(1-\gamma)T^{3}}\bm{1} 
	+\frac{4\log^{3}T}{\mu_{\min}(1-\gamma)T}(\bm{I} - \sqrt\gamma\bP)^{-1}  \max_{k:\,t/2< k\leq t}\bm{\sigma}\prn{{\bm{\eta}}_{k-1}},
    \end{aligned}
\end{equation*}
\end{small}
In the third inequality, we used Lemma~\ref{lem:effect_lr_bound} and
\begin{equation*}
\begin{aligned}
        \norm{\bm{\sigma}({\bm{\eta}}_{k-1})}\leq\int_0^\frac{1}{1-\gamma}dx=\frac{1}{1-\gamma}.
\end{aligned}
\end{equation*}
In the fourth inequality, we used Lemma~\ref{lem:step_size_range} to bound $\beta_k^{(t)}$.
And in the last inequality, note that $\max_{k:\,t/2\leq k<t}\bm{\sigma}\prn{{\bm{\eta}}_{k-1}}\geq \bm{0}$, we can use the following lemma.
\begin{lemma}\label{lem:sum_lr_matrix_ub}
    For any $t\in\NB$, the following inequality holds entry-wise:
    \begin{equation*}
    \begin{aligned}
        \sum_{k=t/2+1}^{t}\alpha_{k}\prod_{i=k+1}^t \brk{\bI-\alpha_i\bLamb\prn{\bI-\sqrt\gamma\bP}} \leq (\bI-\sqrt{\gamma}\bP)^{-1}\bLamb^{-1}.
    \end{aligned}
\end{equation*}
\end{lemma}
The proof of the lemma can be found in Appendix~\ref{subsection:proof_lem_sum_lr_matrix_ub}.

Now, we have a almost sure upper-bound for $\norm{\bW_t^{(1)}}$:
\begin{equation*}
\begin{aligned}
\norm{\bW_t^{(1)}} & \leq\frac{1}{2(1-\gamma)T^{3}}
	+\frac{4\log^{3}T}{\mu_{\min}(1-\gamma)T}\norm{(\bm{I}-\sqrt\gamma\bP)^{-1}}\max_{k:\,t/2< k\leq t}\norm{\bm{\sigma}\prn{{\bm{\eta}}_{k-1}}}\\
 & \leq\frac{1}{2(1-\gamma)T^{3}}
	+\frac{8\log^{3}T}{\mu_{\min}(1-\gamma)^3T}\\
 &\leq\frac{9\log^{3}T}{\mu_{\min}(1-\gamma)^3T}\\
 &=:\sigma^{2}_{(1)}.
\end{aligned}
\end{equation*}
Let $H=\big\lceil 2\log_2\frac{1}{1-\gamma}\big\rceil $, we have
\begin{equation*}
\frac{\sigma^{2}_{(1)}}{2^{H}}\leq\frac{9\log^{3}T}{\mu_{\min}(1-\gamma)T}.
\end{equation*}
By applying Freedman's inequality (Theorem~\ref{thm:freedman_ineq_bounded_W}) and utilizing the union bound over $s\in\gS$, we obtain with probability at least $1-\delta$, for all $t\geq\frac{T}{c_6\log T}$ and $s\in\gS$
\begin{small}
\begin{equation*}
    \begin{aligned}
        \abs{\sum_{k=1}^t{\bm{\zeta}}_k^{(1)}}\leq&\sqrt{8\prn{\bW_{t}^{(1)}+\frac{\sigma^{2}_{(1)}}{2^{H}}\bm{1}}\log\frac{8|\gS|T\log\frac{1}{1-\gamma}}{\delta}}+ \frac{4}{3}b^{(1)}\log\frac{8|\gS|T\log\frac{1}{1-\gamma}}{\delta} \bm{1}\\
  \leq&\sqrt{16\prn{\bW_{t}^{(1)}+\frac{9\log^{3}T}{\mu_{\min}(1-\gamma)T}\bm{1}}\log\frac{|\gS|T}{\delta}}+3b^{(1)}\log\frac{|\gS|T}{\delta} \bm{1}\\
	 \leq&8\sqrt{\frac{\prn{\log^{3}T}\prn{\log\frac{|\gS|T}{\delta}}}{\mu_{\min}(1-\gamma)T} \brk{(\bm{I} - \sqrt\gamma\bP)^{-1} \max_{k:\,t/2< k\leq t} \bm{\sigma} \prn{{\bm{\eta}}_{k-1}} + 3 \cdot \bm{1}}}+\frac{12\prn{\log^3 T}\prn{\log\frac{|\gS|T}{\delta}}}{\mu_{\min}(1-\gamma)^{3/2} T}\bm{1}.
    \end{aligned}
\end{equation*}
\end{small}
where we used $\log\frac{8|\gS|T\log\frac{1}{1-\gamma}}{\delta}\leq2\log\frac{\abs{\gS}T}{\delta}$ in the second inequality, which holds due to the choice of $T$.
The following lemmas are required for deriving the upper bound, which hold for both cases of {\NTD} and {\CTD}, whose proofs can be found in Appendix~\ref{subsection:proof_diff_sigma_bounded_by_Delta} and Appendix~\ref{subsection:proof_sigma_fine_upper_bound} respectively.
\begin{lemma}\label{lem:diff_sigma_bounded_by_Delta}
For any $t\in[T]$, then
\begin{equation*}
    \bm{\sigma}({\bm{\eta}}_{t})-\bm{\sigma}({\bm{\eta}})\leq 4\norm{{\bm{\Delta}}_t}_{\bar{W}_1}\bm{1}.
\end{equation*}
\end{lemma}
\begin{lemma}\label{lem:sigma_fine_upper_bound} It holds that
\begin{equation*}
    (\bm{I}-\sqrt{\gamma}\bP)^{-1}\bm{\sigma}({\bm{\eta}})\leq \frac{4}{1-\gamma}\bm{1}.
\end{equation*}
\end{lemma}
Combining the upper bound with the two lemmas, we get the desired conclusion.
\begin{small}
\begin{equation*}
    \begin{aligned}
    \abs{\sum_{k=1}^t{\bm{\zeta}}_k^{(1)}}\leq&8\Bigg\{\frac{\prn{\log^{3}T}\prn{\log\frac{|\gS|T}{\delta}}}{\mu_{\min}(1-\gamma)T}\Big[4\max_{k:\,t/2< k\leq t}\norm{{\bm{\Delta}}_{k-1}}_{\bar{W}_1}(\bm{I}-\sqrt{\gamma}\bP)^{-1}\bm{1}+\frac{8}{1-\gamma}\bm{1}\Big]\Bigg\}^{\frac{1}{2}}+\frac{12\prn{\log^{3}T}\prn{\log\frac{|\gS|T}{\delta}}}{\mu_{\min}(1-\gamma)^{3/2}T}\bm{1}\\
\leq& 22\sqrt{\frac{\prn{\log^{3}T}\prn{\log\frac{|\gS|T}{\delta}}}{\mu_{\min}(1-\gamma)^{2}T}\prn{1+\max_{k:\,t/2< k\leq t}\norm{{\bm{\Delta}}_{k-1}}_{\bar{W}_1}}}\bm{1}+\frac{12\prn{\log^{3}T}\prn{\log\frac{|\gS|T}{\delta}}}{\mu_{\min}(1-\gamma)^{3/2}T}\bm{1}\\
\leq&34\sqrt{\frac{\prn{\log^{3}T}\prn{\log\frac{|\gS|T}{\delta}}}{\mu_{\min}(1-\gamma)^{2}T}\prn{1+\max_{k:\,t/2< k\leq t}\norm{{\bm{\Delta}}_{k-1}}_{\bar{W}_1}}}\bm{1},
    \end{aligned}
\end{equation*}
\end{small}
where in the last line, we used that, excluding the constant term, the first term is larger than the second term, given the choice of $T\geq \frac{C_4\log^3 T}{\varepsilon^2 \mu_{\min}(1-\gamma)^3}\log\frac{\abs{\gS}T}{\delta}$.
\end{proof}

\subsection{Concentration of the Martingale Term (ii)}\label{subsection:concentration_ii}
\begin{proof}[Proof of Lemma~\ref{lem_asyn_dtd_l2_term_ii}]
The proof strategy is same as that of Lemma~\ref{lem_asyn_dtd_l2_term_i}, and we will omit some details which can be checked easily.
We denote 
\begin{small}
\begin{equation*}
    {\bm{\zeta}}_k^{(2)}=\alpha_k \prn{\prod_{i=k+1}^t \brk{\gI-\alpha_i\bLamb\prn{\gI-\gT}}}\prn{\bLamb-\bLamb_k}\prn{\gI-\gT}{\bm{\eta}}_{k-1},
\end{equation*}
\end{small}
then $\text{(ii)}=\sum_{k=1}^t{\bm{\zeta}}_k^{(2)}$, and $\brc{{\bm{\zeta}}_k^{(2)}}_{k=1}^t$ is a $\brc{\gF_k}_{k=1}^t$-martingale difference sequence.

The norm of the martingale difference $\norm{{\bm{\zeta}}_k^{(2)}}$ can be bounded as follow
\begin{small}
\begin{equation*}
\begin{aligned}
    \norm{{\bm{\zeta}}_k^{(2)}}&\leq  \alpha_k\norm{\prod_{i=k+1}^t \brk{\gI-\alpha_i\bLamb\prn{\gI-\gT}}}\norm{\bLamb-\bLamb_k} \norm{\prn{\gI-\gT}{\bm{\eta}}_{k-1}}\\
    &\leq \alpha_k\prn{\prod_{i=k+1}^t \brk{1-\alpha_i\mu_{\min}\prn{1-\sqrt\gamma}}} \frac{1}{\sqrt{1-\gamma}}\\
    &\leq\frac{\beta_{k}^{(t)}}{\sqrt{1-\gamma}}.
\end{aligned}
\end{equation*}
\end{small}
Hence, $\max_{k\in [t]}\norm{{\bm{\zeta}}_k^{(2)}}\leq \frac{\max_{k\in[t]}\beta_k^{(t)}}{\sqrt{1-\gamma}}\leq \frac{1}{\sqrt{1-\gamma}}\max\brc{\frac{1}{T^2},\frac{4\log^3 T}{\mu_{\min}(1-\sqrt\gamma)T}}\leq\frac{4\log^3 T}{\mu_{\min}(1-\gamma)^{3/2} T}=:b^{(2)}$.

To bound the quadratic variation, we need to bound $\var_{k-1}\prn{{\bm{\zeta}}_k^{(2)}}$ using the following lemma, whose proof can be found in Appendix~\ref{subsection:proof_zeta2_var_bound}.
\begin{lemma}\label{lem:zeta2_var_bound} We have that
\begin{small}
\begin{equation*}
            \var_{k - 1}  \prn{ {\bm{\zeta}}_k^{(2)} }  \leq  2\alpha_k \beta_k^{(t)}  \prn{\prod_{i=k+1}^t    \brk{\bI - \alpha_i\bLamb \prn{\bI - \sqrt{\gamma}\bP}} } \bLamb \abs{\prn{\gI - \gT}{\bm{\eta}}_{k - 1}}^2,
\end{equation*}
\end{small}
where $\abs{\prn{\gI-\gT}{\bm{\eta}}_{k-1}}^2$ is the entry-wise square of $\abs{\prn{\gI-\gT}{\bm{\eta}}_{k-1}}$.
\end{lemma}
Hence the quadratic variation $\bW_t^{(2)}:= \sum_{k=1}^t \var_{k-1}\prn{{\bm{\zeta}}_k^{(2)}}$ can be bounded as follows
\begin{small}
\begin{equation*}\label{eq:asyn_td_upper_bound_of_Wt2}
    \begin{aligned}
        \bW_t^{(2)}=&\sum_{k=1}^{t}\var_{t-1}\prn{{\bm{\zeta}}_k^{(2)}}\\
        \leq&2\sum_{k=1}^{t} \alpha_k \beta_k^{(t)} \prn{\prod_{i=k+1}^t \brk{\bI - \alpha_i\bLamb\prn{\bI - \sqrt\gamma\bP}}} \bLamb\abs{\prn{\gI - \gT}{\bm{\eta}}_{k-1}}^2\\
\leq& 2 \sum_{k=1}^{t/2}\alpha_k\beta_{k}^{(t)} \norm{\prod_{i=k+1}^t    \brk{\bI - \alpha_i\bLamb\prn{\bI - \sqrt\gamma\bP}}} \norm{\bLamb} \norm{\prn{\gI - \gT}{\bm{\eta}}_{k-1}}^2  \bm{1}\\
&+2\sum_{k=t/2+1}^{t}  \alpha_k \beta_k^{(t)} \prn{\prod_{i=k+1}^t  \brk{\bI - \alpha_i\bLamb\prn{\bI - \sqrt\gamma\bP}}}\bLamb\abs{\prn{\gI - \gT}{\bm{\eta}}_{k-1}}^2\\
\leq& 2\sum_{k=1}^{t/2}\prn{\beta_{k}^{(t)}}^{2}\frac{1}{1-\gamma}\bm{1} 
 + 2\prn{\max_{k:\,t/2<k\leq t} \beta_{k}^{(t)}}\sum_{k=t/2+1}^{t}\alpha_k \prn{\prod_{i=k+1}^t \brk{\bI-\alpha_i\bLamb\prn{\bI-\sqrt\gamma\bP}}}\bLamb \abs{\prn{\gI-\gT}{\bm{\eta}}_{k-1}}^2\\
\leq&\frac{1}{(1-\gamma)T^{3}} \bm{1} +\frac{4 \log^{3}T}{\mu_{\min}(1-\sqrt\gamma)T}\brc{ \sum_{k=t/2+1}^{t}   \alpha_{k} \prod_{i=k+1}^t    \brk{\bI - \alpha_i\bLamb\prn{\bI - \sqrt\gamma\bP}} }\bLamb \max_{k:\,t/2< k\leq t} \abs{\prn{\gI - \gT}{\bm{\eta}}_{k-1}}^2\\
\leq&\frac{1}{(1-\gamma)T^{3}}\bm{1} 
	+\frac{4\log^{3}T}{\mu_{\min}(1-\sqrt\gamma)T}(\bm{I}-\sqrt\gamma\bP)^{-1}\max_{k:\,t/2< k\leq t}\abs{\prn{\gI-\gT}{\bm{\eta}}_{k-1}}^2\\
\leq&\frac{1}{(1-\gamma)T^{3}}\bm{1} 
	+\frac{16\log^{3}T}{\mu_{\min}(1-\gamma)^2T}\max_{k:\,t/2< k\leq t}\abs{\prn{\gI-\gT}{\bm{\eta}}_{k-1}}^2,
    \end{aligned}
\end{equation*}
\end{small}
where in the fifth inequality, we used $\max_{k:\,t/2\leq k<t}\abs{\prn{\gI-\gT}{\bm{\eta}}_{k-1}}^2\geq \bm{0}$ entry-wise, and Lemma~\ref{lem:sum_lr_matrix_ub}.

Now, we have a almost sure upper-bound for $\norm{\bW_t^{(2)}}$:
\begin{small}
\begin{equation*}
\begin{aligned}
\norm{\bW_t^{(2)}} & \leq\frac{1}{(1-\gamma)T^{3}}
	+\frac{16\log^{3}T}{\mu_{\min}(1-\gamma)^2T}\max_{k:\,t/2< k\leq t}\norm{\prn{\gI-\gT}{\bm{\eta}}_{k-1}}^2\\
 & \leq\frac{1}{(1-\gamma)T^{3}}
	+\frac{16\log^{3}T}{\mu_{\min}(1-\gamma)^3T}\\
 &\leq\frac{17\log^{3}T}{\mu_{\min}(1-\gamma)^3T}\\
 &=:\sigma^{2}_{(2)}.
\end{aligned}
\end{equation*}
\end{small}
Again letting $H=\big\lceil 2\log_2\frac{1}{1-\gamma}\big\rceil $, we have
\begin{equation*}
\frac{\sigma^{2}_{(2)}}{2^{H}}\leq\frac{17\log^{3}T}{\mu_{\min}(1-\gamma)T}.
\end{equation*}
Applying Freedman's inequality (Theorem~\ref{thm:freedman_ineq_bounded_W}) and utilizing the union bound over $s\in\gS$, we obtain with probability at least $1-\delta$, for all $t\geq\frac{T}{c_6\log T}$ and $s\in\gS$
\begin{small}
\begin{equation*}
    \begin{aligned}
        \abs{\sum_{k=1}^t{\bm{\zeta}}_k^{(2)}}\leq&\sqrt{8\prn{\bW_{t}^{(2)}+\frac{\sigma^{2}_{(2)}}{2^{H}}\bm{1}}\log\frac{8|\gS|T\log\frac{1}{1-\gamma}}{\delta}}+\frac{4}{3}b^{(2)}\log\frac{8|\gS|T\log\frac{1}{1-\gamma}}{\delta} \bm{1}\\
  \leq&\sqrt{16\prn{\bW_{t}^{(2)}+\frac{17\log^{3}T}{\mu_{\min}(1-\gamma)T}\bm{1}}\log\frac{|\gS|T}{\delta}}+3b^{(2)}\log\frac{|\gS|T}{\delta} \bm{1}\\
	 \leq&17\sqrt{\frac{\prn{\log^{3}T}\prn{\log\frac{|\gS|T}{\delta}}}{\mu_{\min}(1-\gamma)^2T}\brk{\max_{k:\,t/2< k\leq t}\abs{\prn{\gI-\gT}{\bm{\eta}}_{k-1}}^2+\bm{1}}}+\quad\frac{12\prn{\log^3 T}\prn{\log\frac{|\gS|T}{\delta}}}{\mu_{\min}(1-\gamma)^{3/2} T}\bm{1}, 
    \end{aligned}
\end{equation*}
\end{small}
where we used $\log\frac{8|\gS|T\log\frac{1}{1-\gamma}}{\delta}\leq2\log\frac{\abs{\gS}T}{\delta}$ in the second inequality, which holds due to the choice of $T$.

While
\begin{equation}\label{eq:bound_term_ii_middle}
    \begin{aligned}
        \abs{\prn{\gI-\gT}{\bm{\eta}}_{k-1}}^2& =\abs{\prn{\gI-\gT}\prn{{\bm{\eta}}_{k-1}-{\bm{\eta}}}}^2\\
        &\leq \abs{\prn{\gI-\gT}{\bm{\Delta}}_{k-1}}^2\\
        &\leq 4\norm{{\bm{\Delta}}_{k-1}}^2\bm{1}\\
        &\leq 4\norm{{\bm{\Delta}}_{k-1}}_{\bar{W}_1}\bm{1},
    \end{aligned}
\end{equation}
where in the first line, we used $(\gI-\gT){\bm{\eta}}=0$, and in the last line, we used for any $s\in\gS$,
\begin{equation*}
    \begin{aligned}
        \norm{{\Delta}_{k-1}(s)}^2=&\int_0^{\frac{1}{1-\gamma}}\abs{F_{{\Delta}_{k-1}(s)}(x)}^2 dx\\
        \leq& \prn{\max_{x\in\brk{0,\frac{1}{1-\gamma}}}\abs{F_{{\Delta}_{k-1}(s)}(x)}}\int_0^{\frac{1}{1-\gamma}}\abs{F_{{\Delta}_{k-1}(s)}(x)} dx\\
        \leq& \norm{{\Delta}_{k-1}(s)}_{W_1}.
    \end{aligned}
\end{equation*}
Hence
\begin{small}
\begin{equation*}
    \begin{aligned}
        \abs{\sum_{k=1}^t{\bm{\zeta}}_k^{(2)}}\leq&17\sqrt{\frac{\prn{\log^{3}T}\prn{\log\frac{|\gS|T}{\delta}}}{\mu_{\min}(1-\gamma)^2T}\brk{\max_{k:\,t/2< k\leq t}\abs{\prn{\gI-\gT}{\bm{\eta}}_{k-1}}^2+\bm{1}}}+\frac{12\prn{\log^3 T}\prn{\log\frac{|\gS|T}{\delta}}}{\mu_{\min}(1-\gamma)^{3/2} T}\bm{1}\\
        \leq&34\sqrt{\frac{\prn{\log^{3}T}\prn{\log\frac{|\gS|T}{\delta}}}{\mu_{\min}(1-\gamma)^2T}\prn{1+\max_{k:\,t/2< k\leq t}\norm{{\bm{\Delta}}_{k-1}}_{\bar{W}_1}}}\bm{1}+\frac{12\prn{\log^3 T}\prn{\log\frac{|\gS|T}{\delta}}}{\mu_{\min}(1-\gamma)^{3/2} T}\bm{1}\\
        \leq&46\sqrt{\frac{\prn{\log^{3}T}\prn{\log\frac{|\gS|T}{\delta}}}{\mu_{\min}(1-\gamma)^2T}\prn{1+\max_{k:\,t/2< k\leq t}\norm{{\bm{\Delta}}_{k-1}}_{\bar{W}_1}}}\bm{1}.
    \end{aligned}
\end{equation*}
\end{small}
where in the last line, we used that, excluding the constant term, the first term is larger than the second term, given the choice of $T\geq \frac{C_4\log^3 T}{\varepsilon^2 \mu_{\min}(1-\gamma)^3}\log\frac{\abs{\gS}T}{\delta}$.
\end{proof}
\subsection{Solve the Recurrence Relation}
\begin{theorem}\label{thm:solve_recurrence}
    Suppose for all $t\geq\frac{T}{c_6\log T}$,
    \begin{small}
    \begin{equation*}
    \norm{{\bm{\Delta}}_t}_{\bar{W}_1}\leq 81\sqrt{\frac{\prn{\log^{3}T}\prn{\log\frac{|\gS|T}{\delta}}}{\mu_{\min}(1-\gamma)^{3}T}\prn{1+\max_{k:\,t/2< k\leq t}\norm{{\bm{\Delta}}_{k-1}}_{\bar{W}_1}}}.
\end{equation*}
\end{small}
    Then there exists some large universal constant $C_7>0$, such that
    \begin{small}
    \begin{equation*}
    \norm{{\bm{\Delta}}_T}_{\bar{W}_1}\leq C_7\prn{\sqrt{\frac{\prn{\log^{3}T}\prn{\log\frac{|\gS|T}{\delta}}}{\mu_{\min}(1-\gamma)^{3} T}}+\frac{\prn{\log^{3}T}\prn{\log\frac{|\gS|T}{\delta}}}{\mu_{\min}(1-\gamma)^{3} T}}.
\end{equation*}
\end{small}
\end{theorem}
\begin{proof}
    For any $k\geq 0$, we denote
    \begin{equation*}
 u_{k}:=\max\left\{ \norm{{\bm{\Delta}}_{t}}_{\bar{W}_1}\ \Big|\ 2^{k}\frac{T}{c_{6}\log T}\leq t\leq T\right\},
 \end{equation*}
 for $0\leq k\leq \log_2\prn{c_6 \log T}$.
 We can see that $\norm{{\bm{\Delta}}_T}_{\bar{W}_1}\leq u_k$ for any valid $k$.
 Hence, it suffices to show the upper bound holds for $u_k$ for any valid $k$.
 It can be verified that $u_0\leq\frac{1}{1-\gamma}$, and for $k\geq 0$
 \begin{equation*}
     u_{k+1}\leq 81\sqrt{\frac{\prn{\log^{3}T}\prn{\log\frac{|\gS|T}{\delta}}}{\mu_{\min}(1-\gamma)^3 T}\prn{1+u_{k}}}.
 \end{equation*}
We first show that once $u_k\leq 1$, the subsequent values of $u_{k+l}$ will also remain upper bounded by $1$.
Namely, if $u_k\leq 1$ for some $k\geq 1$, then
\begin{equation*}
         u_{k+1}\leq 81\sqrt{\frac{2\prn{\log^{3}T}\prn{\log\frac{|\gS|T}{\delta}}}{\mu_{\min}(1-\gamma)^3 T}}\leq 1,
\end{equation*}
if $T\geq \frac{10^4\log^{3}T\log\frac{|\gS|T}{\delta}}{\mu_{\min}(1-\gamma)^{3}}$.

Let $\tau:=\inf\brc{k:u_k\leq 1}$, then for any $k>\tau$, we have
\begin{equation*}
    u_k\leq 81\sqrt{\frac{2\prn{\log^{3}T}\prn{\log\frac{|\gS|T}{\delta}}}{\mu_{\min}(1-\gamma)^3 T}}=:a.
\end{equation*}
For $k\leq\tau$, we have $u_k\geq 1$ and thereby
\begin{equation*}
    u_{k+1}\leq 35\sqrt{\frac{2\prn{\log^{3}T}\prn{\log\frac{|\gS|T}{\delta}}}{\mu_{\min}(1-\gamma)^3 T}u_{k}}=a\sqrt{u_{k}},
\end{equation*}
\ie,
\begin{equation*}
    \log u_{k+1} - 2\log a\leq \frac{1}{2}\prn{\log u_{k}-2\log a}.
\end{equation*}
Apply it recursively, we have
\begin{equation*}
    \log u_{k+1}\leq 2\log a+\prn{\frac{1}{2}}^{k+1}\prn{\log u_{0}-2\log a},
\end{equation*}
\ie,
\begin{small}
\begin{equation*}
    u_{k+1}\leq a^2 \prn{\frac{u_0}{a^2}}^{1/2^{k}}=a^{2\prn{1-1/2^k}}u_0^{1/2^k}\leq \frac{a^{2\prn{1-1/2^k}}}{(1-\gamma)^{1/2^{k}}}.
\end{equation*}
\end{small}
To sum up, for any $k\geq 0$, $u_{k+1}$ is always less than the sum of the upper bounds in cases of $k>\tau$ and $k\leq\tau$, 
\begin{equation*}
    u_{k+1}\leq a +   a^{2\prn{1-1/2^k}}\frac{1}{(1-\gamma)^{1/2^{k}}}
\end{equation*}
Note that, $a^{2\prn{1-1/2^k}}\leq\max\brc{a,\sqrt{a}}$, and if we take $k\geq c_8\log\log\frac{1}{1-\gamma}$ for any constant $c_8$, we have $\frac{1}{(1-\gamma)^{1/2^{k}}}=O(1)$.
We can take the constant $c_8$ small enough such that $c_8\log\log\frac{1}{1-\gamma}<\log_2\prn{c_6\log T}$ (this can be done and $c_8$ is universal since $\frac{1}{1-\gamma}=o(T)$), and thereby we can find a valid $k^\star\geq c_8\log\log\frac{1}{1-\gamma}+1$.
Then
\begin{equation*}
\begin{aligned}
        &\norm{{\bm{\Delta}}_T}_{\bar{W}_1}\leq u_{k^\star}\leq C_7\prn{\sqrt{\frac{\prn{\log^{3}T}\prn{\log\frac{|\gS|T}{\delta}}}{\mu_{\min}(1-\gamma)^3 T}}+\frac{\prn{\log^{3}T}\prn{\log\frac{|\gS|T}{\delta}}}{\mu_{\min}(1-\gamma)^3 T}},
\end{aligned}
\end{equation*}
which is the desired conclusion, and $C_7$ is some large universal constant related to $c_8$.
\end{proof}

\subsection{The Cram\'er Error Bound}\label{Subsection_analysis_coro_ntd_l2}
\begin{corollary}[Sample Complexity for Cram\'er Error with a Generative Model]\label{coro:samplecomplex_syn_l2}
In the generative model setting, consider any given $\delta \in (0,1)$, $\varepsilon\in (0,1)$,
suppose $K> \frac{4}{\varepsilon^2(1 {-} \gamma)^2}$ in {\CTD}, the initialization is ${\bm{\eta}}^\pi_0\in\sP^\gS$ ($\sP^\gS_K$ in {\CTD}),  
the total update steps $T$ satisfies
\begin{equation*}
    T \geq  \frac{C_1\log^3 T}{\varepsilon^2 \mu_{\min} (1-\gamma)^{5/2}}\log \frac{\abs{\gS}T}{\delta},
\end{equation*}
for some large universal constant $C_1>1$, \ie, $T=\widetilde{O}\left(\frac{1}{\varepsilon^2 \mu_{\min} (1-\gamma)^{5/2}}\right)$, the step size $\alpha_t$ is set as
\begin{equation*}
    \frac{1}{1+\frac{c_2\mu_{\min}(1-\sqrt\gamma)t}{\log t}}\leq\alpha_t\leq \frac{1}{1+\frac{c_3\mu_{\min}(1-\sqrt\gamma)t}{\log t}}
\end{equation*}
for some small universal constants $c_2>c_3>0$.
Then, for both {\NTD} and {\CTD}, with probability at least $1-\delta$, the last iterate estimator satisfies $\bar{\ell}_2\prn{{\bm{\eta}}^\pi_T,{\bm{\eta}}^{\pi}}\leq \varepsilon$.
\end{corollary}
The similar conclusion holds in the Markovian setting, we omit it for brevity.

The difference in the proof compared to Section~\ref{Subsection_analysis_nonasymp_ntd_w1} arises in Lemma~\ref{lem_asyn_dtd_l2_term_i} and Lemma~\ref{lem_asyn_dtd_l2_term_ii} when we control term (II).
Now we further bound the result in Lemma~\ref{lem:diff_sigma_bounded_by_Delta} by the Cram\'er norm of the error term, which is used in Lemma~\ref{lem_asyn_dtd_l2_term_i},
\begin{equation*}
    \bm{\sigma}({\bm{\eta}}_{t})-\bm{\sigma}({\bm{\eta}})\leq 4\norm{{\bm{\Delta}}_t}_{\bar{W}_1}\bm{1}\leq\frac{4}{\sqrt{1-\gamma}}\norm{{\bm{\Delta}}_t}\bm{1}.
\end{equation*}
As for Lemma~\ref{lem_asyn_dtd_l2_term_ii}, we need to modify Eqn.~\eqref{eq:bound_term_ii_middle} to 
\begin{equation*}
    \begin{aligned}
        \abs{\prn{\gI-\gT}{\bm{\eta}}_{k-1}}^2& =\abs{\prn{\gI-\gT}\prn{{\bm{\eta}}_{k-1}-{\bm{\eta}}}}^2\\
        &\leq \abs{\prn{\gI-\gT}{\bm{\Delta}}_{k-1}}^2\\
        &\leq 4\norm{{\bm{\Delta}}_{k-1}}^2\bm{1}\\
        &\leq \frac{4}{\sqrt{1-\gamma}}\norm{{\bm{\Delta}}_{k-1}}\bm{1}.
    \end{aligned}
\end{equation*}
In the same way, we can derive the following recurrence relation: with probability at least $1-\delta$, for all $t\geq\frac{T}{c_6\log T}$
\begin{equation*}
    \norm{{\bm{\Delta}}_t}\leq 81\sqrt{\frac{\prn{\log^{3}T}\prn{\log\frac{|\gS|T}{\delta}}}{\mu_{\min}(1-\gamma)^{5/2}T}\prn{1+\max_{k:\,t/2< k\leq t}\norm{{\bm{\Delta}}_{k-1}}}}.
\end{equation*}
By repeating the reasoning of Theorem~\ref{thm:solve_recurrence}, we can obtain the desired conclusion, 
\begin{small}
\begin{equation*}
    \norm{{\bm{\Delta}}_T}\leq C_7\prn{\sqrt{\frac{\prn{\log^{3}T}\prn{\log\frac{|\gS|T}{\delta}}}{\mu_{\min}(1-\gamma)^{5/2} T}}+\frac{\prn{\log^{3}T}\prn{\log\frac{|\gS|T}{\delta}}}{\mu_{\min}(1-\gamma)^{5/2} T}},
\end{equation*}
\end{small}
which is less than $\varepsilon$ if we take $C_4\geq 2C_7^2$ and $T \geq  \frac{C_4\log^3 T}{\varepsilon^2 \mu_{\min}(1-\gamma)^{5/2}}\log \frac{\abs{\gS}T}{\delta}$.
Here, $C_7>1$ is a large universal constant depending on $c_6$.

\subsection{The Polyak-Ruppert Averaging}\label{subsection:Polyak}
By the derivation in Appendix~\ref{subsection_remove_T} and the upper bound of $\norm{{\bm{\Delta}}_t}_{\bar{W}_1}$ derived by solving the recursive relation, if we take $T_0\geq \frac{T}{2}$, we have with probability at least $1-\delta$, for all $T_0< t\leq T$
\begin{equation*}
\norm{{\bm{\Delta}}_t}_{\bar{W}_1}\leq C\sqrt{\frac{\prn{\log^{3}T}\prn{\log\frac{|\gS|T}{\delta}}}{
    \mu_{\min}(1-\gamma)^3 t}},    
\end{equation*}
for some universal constant $C>0$.
Hence
\begin{small}
\begin{equation*}
\begin{aligned}
    \norm{\frac{1}{T-T_0}\sum_{t=T_0+1}^T{\bm{\eta}}^\pi_t-{\bm{\eta}}^\pi}_{\bar{W}_1}\leq& \frac{1}{T-T_0}\sum_{t=T_0+1}^T\norm{{\bm{\Delta}}_t}_{\bar{W}_1}\\
    \leq& C\sqrt{\frac{\prn{\log^{3}T}\prn{\log\frac{|\gS|T}{\delta}}}{
    \mu_{\min}(1-\gamma)^3}}\frac{1}{T-T_0}\sum_{t=T_0+1}^T\frac{1}{\sqrt{t}}\\
    \leq &2C\sqrt{\frac{\prn{\log^{3}T}\prn{\log\frac{|\gS|T}{\delta}}}{
    \mu_{\min}(1-\gamma)^3}}\frac{\sqrt{T}-\sqrt{T_0}}{T-T_0}\\
    \leq &2C\sqrt{\frac{\prn{\log^{3}T}\prn{\log\frac{|\gS|T}{\delta}}}{
    \mu_{\min}(1-\gamma)^3(T-T_0)}},
\end{aligned}
\end{equation*}
\end{small}
which is desired.

\subsection{Proof of Lemma~\ref{lem:zeta1_var_bound}}\label{subsection:proof_zeta1_var_bound}
\begin{proof}
We first introduce some notations.
For any matrix of operators $\gU\in\gL\prn{\gM}^{\gS\times\gS}$, we denote $\gU(s)=\prn{\gU(s,s^\prime)}_{s^\prime\in\gS}\in\gL\prn{\gM}^\gS$ as the $s$-row of $\gU$. 
And for any ${\bm{\xi}}\in\gM^\gS$, we define the vector inner product operation $\gU(s){\bm{\xi}}:=\sum_{s^\prime\in\gS}\gU(s,s^\prime)\xi(s^\prime)\in\gM$.

We need the following lemma, which holds for both cases of {\NTD} and {\CTD}. 
\begin{lemma}\label{lem:Uv_norm_bound_new}
    For any $n\in\NB$, 
    let 
    \begin{equation*}
        \begin{aligned}
            \gU_n&=\prod_{i=1}^n \brk{\gI-\alpha_i\bLamb\prn{\gI-\gT}},\\
            \bU_{n}&=\prod_{i=1}^n \brk{\bI-\alpha_i \bLamb\prn{\bI-\sqrt{\gamma}\bP}},\\
            u_n&=\prod_{i=1}^n \brk{1-\alpha_i\mu_{\min}\prn{1-\sqrt{\gamma}}}.
        \end{aligned}
    \end{equation*}
    Then for any $\nu\in\gM$, and $s, s^\prime\in\gS$, we have
\begin{equation*}
    \begin{aligned}
        \norm{\gU_n(s,s^\prime)\nu}^2\leq u_nU_n(s,s^\prime)\norm{\nu}^2.
    \end{aligned}
\end{equation*}
\end{lemma}
The proof can be found in Appendix~\ref{subsection:proof_Uv_norm_bound}.

Recall that $\what{\gT}$ is a random operator and has the same distribution as $\gT_1$.
Similarly, we define $\what{\bLamb}=\operatorname{diag}\brc{\prn{\hat{\lambda}_s}_{s\in\gS}}$ is a random matrix independent of any other random variables, and has the same distribution as $\bLamb_1$.
Then $\lambda_s=\EB\brk{\hat{\lambda}_s}$.
By definition of $\bLamb_1$, there exist a unique $\hat{s}\in\gS$ such that $\hat{\lambda}_{\hat{s}}=1$, and $\hat{\lambda}_s=0$ for $s\neq \hat{s}$.
If we assume the generative model samples all states at time $t$, then $\what{\bLamb}=\bI$, and the following derivation also holds.

Utilizing Lemma~\ref{lem:Uv_norm_bound_new}, we get the following result. 
For any non-random ${\bm{\xi}}\in\gM^\gS$,  
\begin{small}
\begin{equation*}
    \begin{aligned}
        \EB\brk{\norm{\gU_n(s)\what\bLamb(\what{\gT}-\gT){\bm{\xi}}}^2}=&\EB\brk{\norm{\sum_{s^\prime\in\gS}\gU_n(s,s^\prime)\hat{\lambda}_{s^\prime}\brk{(\what{\gT}-\gT){\bm{\xi}}}(s^\prime)}^2}\\
        =&\sum_{s^\prime\in\gS}\EB\brk{\hat{\lambda}_{s^\prime}^2\norm{\gU_n(s,s^\prime)\brk{(\what{\gT}-\gT){\bm{\xi}}}(s^\prime)}^2}\\
        &+\sum_{s^\prime\neq s^{\prime\prime}}\EB\Bigg\{\hat{\lambda}_{s^\prime}\hat{\lambda}_{s^{\prime\prime}}\inner{\gU_n(s,s^\prime)\brk{(\what{\gT}-\gT){\bm{\xi}}}(s^\prime)}{\gU_n(s,s^{\prime\prime})\brk{(\what{\gT}-\gT){\bm{\xi}}}(s^{\prime\prime})}\Bigg\}\\
        =&\sum_{s^\prime\in\gS}\EB\brk{\hat{\lambda}_{s^\prime}}\EB\brk{\norm{\gU_n(s,s^\prime)\brk{(\what{\gT}-\gT){\bm{\xi}}}(s^\prime)}^2}\\
        \leq&u_n\sum_{s^\prime\in\gS}\lambda_{s^\prime}U_n(s,s^\prime)\EB\brk{\norm{\brk{(\what{\gT}-\gT){\bm{\xi}}}(s^\prime)}^2}\\
        =&u_n\sum_{s^\prime\in\gS}U_n(s,s^\prime)\lambda_{s^\prime}\bm{\sigma}({\bm{\xi}})(s^\prime)\\
        =&u_n \bU_n(s)\bLamb\bm{\sigma}({\bm{\xi}}).
    \end{aligned}
\end{equation*}
\end{small}
In the third equality, we used the fact that $\hat{\lambda}_{s^\prime}\hat{\lambda}_{s^{\prime\prime}}=0$ for $s^\prime\neq s^{\prime\prime}$ when the generative model samples only one state from $\mu$ at each iteration; and we used the fact that different rows of $\what{\gT}$ are independent, $\what{\gT}(s^\prime){\bm{\xi}}$ is an unbiased estimator of $\gT(s^\prime){\bm{\xi}}\in\gM$ when the generative model can sample all states at each iteration.
In the fourth equality, we used the fact that $\what{\gT}$ and $\what{\bLamb}$ are independent, and $\hat{\lambda}_{s^\prime}^2=\hat{\lambda}_{s^\prime}$.
In the first inequality, we used Lemma~\ref{lem:Uv_norm_bound_new}.
Hence, $\var\prn{\gU_n\what{\bLamb}(\what{\gT}-\gT){\bm{\xi}}}\leq u_n \bU_n\bLamb\bm{\sigma}({\bm{\xi}})$.

Now, we are ready to bound $\var_{k-1}\prn{{\bm{\zeta}}_k^{(1)}}$
\begin{small}
\begin{equation*}
\begin{aligned}
    \var_{k-1}\prn{{\bm{\zeta}}_k^{(1)}}=&\alpha_k^2\var_{k-1}\prn{\prod_{i=k+1}^t \brk{\gI-\alpha_i\bLamb\prn{\gI-\gT}}\bLamb_k\prn{\gT_k-\gT}{\bm{\eta}}_{k-1}}\\
    \leq&\alpha_k^2  \prod_{i=k+1}^t    \brk{1 - \alpha_i\mu_{\min}\prn{1 - \sqrt{\gamma}}}   \prod_{i=k+1}^t    \brk{\bI - \alpha_i\bLamb\prn{\bI - \sqrt{\gamma}\bP}}\bLamb\bm{\sigma}({\bm{\eta}}_{k-1})\\
    =&\alpha_k\beta_k^{(t)}\prod_{i=k+1}^t \brk{\bI-\alpha_i\bLamb\prn{\bI-\sqrt{\gamma}\bP}} \bLamb\bm{\sigma}({\bm{\eta}}_{k-1}).
\end{aligned}
\end{equation*}
\end{small}
\end{proof}

\subsection{Proof of Lemma~\ref{lem:sum_lr_matrix_ub}}\label{subsection:proof_lem_sum_lr_matrix_ub}
\begin{proof}
\begin{small}
\begin{equation*}
    \begin{aligned}
        &\sum_{k=t/2+1}^{t}\alpha_{k}\prod_{i=k+1}^t \brk{\bI-\alpha_i\bLamb\prn{\bI-\sqrt\gamma\bP}}=\sum_{k=t/2+1}^{t}\prod_{i=k+1}^{t}   \brk{\bI - \alpha_i\bLamb\prn{\bI - \sqrt\gamma\bP}}\alpha_{k}\bLamb(\bI - \sqrt\gamma\bP)(\bI - \sqrt\gamma\bP)^{-1} \bLamb^{-1}\\
&\qquad=\sum_{k=t/2+1}^{t}\brc{\prod_{i=k+1}^{t}\brk{\bI - \alpha_i\bLamb\prn{\bI - \sqrt\gamma\bP}}  -  \prod_{i=k}^{t}\brk{\bI - \alpha_i\bLamb\prn{\bI - \sqrt\gamma\bP}}}(\bI-\sqrt\gamma\bP)^{-1}\bLamb^{-1} \\
&\qquad= \brc{\bI-\prod_{i=t/2+1}^{t}\brk{\bI-\alpha_i\bLamb\prn{\bI-\sqrt\gamma\bP}}}(\bI-\sqrt\gamma\bP)^{-1}\bLamb^{-1}\\
&\qquad\leq (\bI-\sqrt\gamma\bP)^{-1}\bLamb^{-1},
    \end{aligned}
\end{equation*}
\end{small}
where the last inequality holds entry-wise since we can verify that all entries of $(\bI-\sqrt\gamma\bP)^{-1}=\sum_{k=0}^\infty\prn{\sqrt\gamma\bP}^k$ and $\bI-\alpha_i\bLamb\prn{\bI-\sqrt\gamma\bP}$ are non-negative. 
\end{proof}

\subsection{Proof of Lemma~\ref{lem:diff_sigma_bounded_by_Delta}}\label{subsection:proof_diff_sigma_bounded_by_Delta}
\begin{proof}
For any $s\in\gS$,
\begin{small}
    \begin{equation*}
        \begin{aligned}
            \bm{\sigma}({\bm{\eta}}_{t})(s)-\bm{\sigma}({\bm{\eta}})(s)=&\int_{0}^{\frac{1}{1-\gamma}}\Bigg\{
            \brk{\EB\brk{F_{\prn{\what{\gT}{\bm{\eta}}_t}(s)}^2(x)}-F_{\prn{\gT{\bm{\eta}}_t}(s)}^2(x)}-\brk{\EB\brk{F_{\prn{\what{\gT}{\bm{\eta}}}(s)}^2(x)}- F_{\prn{\gT{\bm{\eta}}}(s)}^2(x)}\Bigg\}dx\\
            =&\int_{0}^{\frac{1}{1-\gamma}}\Bigg\{\EB\brk{F_{\prn{\what{\gT}{\bm{\eta}}_t}(s)}^2(x)-F_{\prn{\what{\gT}{\bm{\eta}}}(s)}^2(x)}+\brk{F_{\prn{\gT{\bm{\eta}}}(s)}^2(x)- F_{\prn{\gT{\bm{\eta}}_t}(s)}^2(x)}\Bigg\}dx\\
            =&\int_{0}^{\frac{1}{1-\gamma}}\Bigg\{\EB\Big[\prn{F_{\prn{\what{\gT}{\bm{\eta}}_t}(s)}(x)-F_{\prn{\what{\gT}{\bm{\eta}}}(s)}(x)}\prn{ F_{\prn{\what{\gT}{\bm{\eta}}_t}(s)}(x) + F_{\prn{\what{\gT}{\bm{\eta}}}(s)}(x) } \Big] \\
            &+ \prn{F_{\prn{\gT{\bm{\eta}}}(s)}(x) - F_{\prn{\gT{\bm{\eta}}_t}(s)}(x)}\prn{F_{\prn{\gT{\bm{\eta}}}(s)}(x)+F_{\prn{\gT{\bm{\eta}}_t}(s)}(x)}\Bigg\}dx\\
            \leq&2\int_{0}^{\frac{1}{1-\gamma}}\Bigg\{\EB\brk{\abs{F_{\prn{\what{\gT}{\bm{\eta}}_t}(s)}(x)-F_{\prn{\what{\gT}{\bm{\eta}}}(s)}(x)}}+\abs{F_{\prn{\gT{\bm{\eta}}}(s)}(x)-F_{\prn{\gT{\bm{\eta}}_t}(s)}(x)}\Bigg\}dx\\
            =&2\prn{\EB\brk{\norm{\brk{\what{\gT}\prn{{\bm{\eta}}_t-{\bm{\eta}}}}(s)}_{W_1}}+\norm{\brk{\gT\prn{{\bm{\eta}}_t-{\bm{\eta}}}}(s)}_{W_1}}.
        \end{aligned}
    \end{equation*}
    \end{small}
    In the case of {\NTD}, $\gT$ and $\what{\gT}$ are $\gamma$-contraction w.r.t. the supreme $1$-Wasserstein metric, hence
    \begin{equation}
        \begin{aligned}
            \bm{\sigma}({\bm{\eta}}_{t})(s)-\bm{\sigma}({\bm{\eta}})(s)\leq& 2\prn{\EB\brk{\norm{\brk{\what{\gT}\prn{{\bm{\eta}}_t-{\bm{\eta}}}}(s)}_{W_1}}+\norm{\brk{\gT\prn{{\bm{\eta}}_t-{\bm{\eta}}}}(s)}_{W_1}}\\
            \leq& 4\gamma\norm{{\bm{\eta}}_t-{\bm{\eta}}}_{\bar{W}_1}\\
            \leq& 4\norm{{\bm{\Delta}}_t}_{\bar{W}_1}.
        \end{aligned}
    \end{equation}
    In the case of {\CTD}, if we can show $\bm{\Pi}_K$ is non-expansive w.r.t. $1$-Wasserstein metric, the conclusion still holds.
    For any $x, y\in\brk{0,\frac{1}{1-\gamma}}$ such that $x<y$, we denote $x\in[x_k,x_{k+1})$ and $y\in[x_l,x_{l+1})$, then $k\leq l$, by the definition of $\bm{\Pi}_K$, we have
        \begin{equation}
    \bm{\Pi}_K(\delta_x)=\frac{x_{k+1} - x}{\iota_K} \delta_{x_k} + \frac{x - x_k}{\iota_K} \delta_{x_{k+1}},
    \end{equation}
    \begin{equation}
    \bm{\Pi}_K(\delta_y)=\frac{x_{l+1} - y}{\iota_K} \delta_{x_l} + \frac{y - x_{l}}{\iota_K} \delta_{x_{l+1}}.
    \end{equation}
    If $k=l$, we can check that $W_1\prn{\bm{\Pi}_K\delta_x,\bm{\Pi}_K\delta_y}=\iota_K\frac{y-x}{\iota_K}=y-x$.
    If $k<l$, we have $W_1\prn{\bm{\Pi}_K\delta_x,\bm{\Pi}_K\delta_y}\leq W_1\prn{\bm{\Pi}_K\delta_x,\delta_{x_{k+1}}}+W_1\prn{\delta_{x_{k+1}}, \delta_{x_l}}+W_1\prn{\delta_{x_l},\bm{\Pi}_K\delta_y}=(x_{k+1}-x)+(x_l-x_{k+1})+(y-x_{x_{l}})=y-x$.
    Hence, for any $\nu_1,\nu_2\in\sP$ and for any transport plan $\kappa\in\Gamma(\nu_1,\nu_2)$, the previous results tell us the cost of the transport plan $\bm{\Pi}_K\kappa\in\Gamma\prn{\bm{\Pi}_K\nu_1,\bm{\Pi}_K\nu_2}$ induced by $\bm{\Pi}_K$ is no greater than the cost of $\kappa$.
    Consequently, $W_1\prn{\bm{\Pi}_K\nu_1,\bm{\Pi}_K\nu_2}\leq W_1(\nu_1,\nu_2)$, \ie, $\bm{\Pi}_K$ is non-expansive w.r.t. $1$-Wasserstein metric, which is desired.
\end{proof}

\subsection{Proof of Lemma~\ref{lem:sigma_fine_upper_bound}}\label{subsection:proof_sigma_fine_upper_bound}
\begin{proof}
Firstly, we show that for any $\bm{v}\geq \bm{0}$, we have $\norm{(\bm{I}-\sqrt{\gamma}\bP)^{-1}\bm{v}}\leq 2\norm{(\bm{I}-\gamma\bP)^{-1}\bm{v}}$.
\begin{small}
    \begin{equation*}
        \begin{aligned}
            \norm{(\bm{I}-\sqrt{\gamma}\bP)^{-1}\bm{v}}=&\norm{(\bm{I}-\sqrt{\gamma}\bP)^{-1}(\bm{I}-\gamma\bP)(\bm{I}-\gamma\bP)^{-1}\bm{v}}\\
            =&\norm{(\bm{I}-\sqrt{\gamma}\bP)^{-1}\brk{(1-\sqrt{\gamma})\bm{I}+\sqrt{\gamma}(\bI-\sqrt{\gamma}\bP)}(\bm{I}-\gamma\bP)^{-1}\bm{v}}\\
            =&\norm{\brk{(1-\sqrt{\gamma})(\bm{I}-\sqrt{\gamma}\bP)^{-1}+\sqrt{\gamma}\bI}(\bm{I}-\gamma\bP)^{-1}\bm{v}}\\
            \leq& (1-\sqrt{\gamma})\norm{(\bm{I}-\sqrt{\gamma}\bP)^{-1}(\bm{I}-\gamma\bP)^{-1}\bm{v}}+\sqrt{\gamma}\norm{(\bm{I}-\gamma\bP)^{-1}\bm{v}}\\
            \leq&\prn{\frac{1-\sqrt{\gamma}}{1-\sqrt{\gamma}}+\sqrt{\gamma}}\norm{(\bm{I}-\gamma\bP)^{-1}\bm{v}}\\
            \leq& 2\norm{(\bm{I}-\gamma\bP)^{-1}\bm{v}}.
        \end{aligned}
    \end{equation*}
    \end{small}
In the case of {\NTD}, by Corollary~\ref{corollary:tight_sigma_upper_bound}, we have
\begin{equation*}
    \norm{(\bm{I}-\gamma\bP)^{-1}\bm{\sigma}\prn{{\bm{\eta}}}}\leq\frac{1}{1-\gamma},
\end{equation*}
In the case of {\CTD}, by \citep[Corollary 5.12][]{rowland2024near}, we have
\begin{equation*}
    \norm{(\bm{I}-\gamma\bP)^{-1}\bm{\sigma}\prn{{\bm{\eta}}}}\leq\frac{2}{1-\gamma},
\end{equation*}
given $K> \frac{4}{1-\gamma}$.
\end{proof}
\subsection{Proof of Lemma~\ref{lem:zeta2_var_bound}}\label{subsection:proof_zeta2_var_bound}
\begin{proof}
The proof is similar to that of Lemma~\ref{lem:zeta1_var_bound} in Appendix~\ref{subsection:proof_zeta1_var_bound}, and we will use the same notations.

Utilizing Lemma~\ref{lem:Uv_norm_bound_new}, we get the following result. 
For any non-random ${\bm{\xi}}\in\gM^\gS$,  
\begin{small}
\begin{equation*}
    \begin{aligned}
        \EB\brk{\norm{\gU_n(s)(\bLamb-\what\bLamb)(\gI-\gT){\bm{\xi}}}^2}\leq&\EB\brk{\norm{\gU_n(s)\what\bLamb(\gI-\gT){\bm{\xi}}}^2}\\
        =&\EB\brk{\norm{\sum_{s^\prime\in\gS}\gU_n(s,s^\prime)\hat{\lambda}_{s^\prime}\brk{(\gI-\gT){\bm{\xi}}}(s^\prime)}^2}\\
        =&\sum_{s^\prime\in\gS}\EB\brk{\hat{\lambda}_{s^\prime}^2}\norm{\gU_n(s,s^\prime)\brk{(\gI-\gT){\bm{\xi}}}(s^\prime)}^2\\
        &+\sum_{s^\prime\neq s^{\prime\prime}}\EB\brk{\hat{\lambda}_{s^\prime}\hat{\lambda}_{s^{\prime\prime}}}\inner{\gU_n(s,s^\prime)\brk{(\gI-\gT){\bm{\xi}}}(s^\prime)}{\gU_n(s,s^{\prime\prime})\brk{(\gI-\gT){\bm{\xi}}}(s^{\prime\prime})}\\
        =&\sum_{s^\prime\in\gS}\EB\brk{\hat{\lambda}_{s^\prime}}\norm{\gU_n(s,s^\prime)\brk{(\gI-\gT){\bm{\xi}}}(s^\prime)}^2\\
        \leq&u_n\sum_{s^\prime\in\gS}\lambda_{s^\prime}U_n(s,s^\prime)\norm{\brk{(\gI-\gT){\bm{\xi}}}(s^\prime)}^2\\
        =&u_n\sum_{s^\prime\in\gS}U_n(s,s^\prime)\lambda_{s^\prime}\abs{\prn{\gI-\gT}{\bm{\xi}}}^2(s^\prime)\\
        =&u_n \bU_n(s)\bLamb\abs{\prn{\gI-\gT}{\bm{\xi}}}^2.
    \end{aligned}
\end{equation*}
\end{small}
In the first inequality, we used the fact that, for any random element $\bm{\nu}$ in Hilbert space $\gM$, $\EB\norm{\bm{\nu}-\EB[\bm{\nu}]}^2=\EB\norm{\bm{\nu}}^2-\norm{\EB[\bm{\nu}]}^2\leq \EB\norm{\bm{\nu}}^2$. 
Hence, $\var\prn{\gU_n(\bLamb-\what{\bLamb})(\gI-\gT){\bm{\xi}}}\leq u_n \bU_n\bLamb\abs{\prn{\gI-\gT}{\bm{\xi}}}^2$.

Now, we are ready to bound $\var_{k-1}\prn{{\bm{\zeta}}_k^{(2)}}$
\begin{small}
\begin{equation*}
\begin{aligned}
    \var_{k-1}\prn{{\bm{\zeta}}_k^{(2)}}=&\alpha_k^2\var_{k-1}\prn{\prod_{i=k+1}^t \brk{\gI-\alpha_i\bLamb\prn{\gI-\gT}}(\bLamb-\bLamb_k)\prn{\gI-\gT}{\bm{\eta}}_{k-1}}\\
    \leq&\alpha_k^2\prod_{i=k+1}^t \brk{1-\alpha_i\mu_{\min}\prn{1-\sqrt{\gamma}}}\prod_{i=k+1}^t \brk{\bI-\alpha_i\bLamb\prn{\bI-\sqrt{\gamma}\bP}}\bLamb\abs{\prn{\gI-\gT}{\bm{\eta}}_{k-1}}^2\\
    =&\alpha_k\beta_k^{(t)}\prod_{i=k+1}^t \brk{\bI-\alpha_i\bLamb\prn{\bI-\sqrt{\gamma}\bP}}\bLamb\abs{\prn{\gI-\gT}{\bm{\eta}}_{k-1}}^2.
\end{aligned}
\end{equation*}
\end{small}
\end{proof}
\subsection{Proof of Lemma~\ref{lem:Uv_norm_bound_new}}\label{subsection:proof_Uv_norm_bound}
\begin{proof}
We proof this result by induction.
For $n=0$, we have $\gU_0=\gI$, $\bU_0=\bI$, $u_0=1$, thereby the inequality holds trivially.
Suppose the inequality holds true for $n-1$. To prove that the inequality holds for $n$, it is sufficient to show that, for any $\nu\in\gM$, 
\begin{equation*}
\begin{aligned}
        &\norm{\brk{(1-\alpha_n\lambda_{s})\delta_{s,s^\prime}+\alpha_n\lambda_{s}\gT(s,s^\prime)}\nu}^2\leq     \brk{(1-\alpha_n\lambda_{\min})+\alpha_n\lambda_{\min} \sqrt{\gamma}}\brk{(1-\alpha_n\lambda_{s})\delta_{s,s^\prime}+\alpha_n \lambda_{s}\sqrt{\gamma}P(s,s^\prime)}\norm{\nu}^2,
\end{aligned}
\end{equation*}
where $\delta_{s,s^\prime}=1$ if $s=s^\prime$, and $0$ otherwise.

LHS can be bounded as follow
\begin{small}
\begin{equation*}
    \begin{aligned}
       \norm{\brk{(1-\alpha_n\lambda_{s})\delta_{s,s^\prime}+\alpha_n\lambda_{s}\gT(s,s^\prime)}\nu}^2=& (1-\alpha_n\lambda_{s})^2\delta_{s,s^\prime}\norm{\nu}^2+ 2(1-\alpha_n\lambda_{s})\alpha_n\lambda_{s}\delta_{s,s^\prime}\inner{\nu}{\gT(s,s^\prime)\nu}+\alpha_n^2\lambda_s^2\norm{\gT(s,s^\prime)\nu}^2  \\
       \leq&(1-\alpha_n\lambda_{s})^2\delta_{s,s^\prime}\norm{\nu}^2+ 2(1-\alpha_n\lambda_{s})\alpha_n\lambda_{s}\delta_{s,s^\prime}\norm{\nu}\norm{\gT(s,s^\prime)\nu}+\alpha_n^2\lambda_s^2\norm{\gT(s,s^\prime)\nu}^2,
    \end{aligned}
\end{equation*}
\end{small}
where we used Cauchy-Schwarz inequality.
We need to give an upper bound for $\norm{\gT(s,s^\prime)\nu}^2$.

Note that $\prn{\bm{\Pi}_K\gT^\pi}(s,s^\prime)=\bm{\Pi}_K\prn{\gT^\pi(s,s^\prime)}$ and $\norm{\bm{\Pi}_K}=1$, we only need to consider the case of {\NTD}, by the definition of $\gT(s,s^\prime)$, we have
\begin{small}
\begin{equation*}
    \begin{aligned}
       \norm{\gT(s,s^\prime)\nu}^2=&\int_0^{\frac{1}{1-\gamma}}\brk{\sum_{a\in\gA}\pi(a|s)P(s^\prime|s,a)\int_{0}^1F_{\nu}\prn{\frac{x-r}{\gamma}}\gP_R(dr|s,a)}^2dx\\
       =&P(s,s^\prime)^2\int_0^{\frac{1}{1-\gamma}}\brk{\sum_{a\in\gA}\frac{\pi(a|s)P(s^\prime|s,a)}{P(s,s^\prime)}\int_{0}^1F_{\nu}\prn{\frac{x-r}{\gamma}}\gP_R(dr|s,a)}^2dx\\
       =&P(s,s^\prime)^2\int_0^{\frac{1}{1-\gamma}}\brc{\EB_{a\sim\pi(\cdot\mid s), r\sim\gP_R(\cdot\mid s,a)}\brk{F_\nu\prn{\frac{x-r}{\gamma}}\Big|s^\prime}}^2dx\\
       \leq& P(s,s^\prime)^2\EB_{a\sim\pi(\cdot\mid s), r\sim\gP_R(\cdot\mid s,a)}\brc{\int_0^{\frac{1}{1-\gamma}}\brk{F_\nu\prn{\frac{x-r}{\gamma}}}^2 dx \Big| s^\prime}\\
       =&\gamma P(s,s^\prime)^2 \norm{\nu}^2,
    \end{aligned}
\end{equation*}
\end{small}
where we used Jensen's inequality and Fubini's theorem.
Substitute it back to the upper bound,
\begin{equation*}
    \begin{aligned}
       &\norm{\brk{(1-\alpha_n\lambda_{s})\delta_{s,s^\prime}+\alpha_n\lambda_{s}\gT(s,s^\prime)}\nu}^2\\
       &\qquad\leq(1-\alpha_n\lambda_{s})^2\delta_{s,s^\prime}\norm{\nu}^2+2(1-\alpha_n\lambda_{s})\alpha_n\lambda_{s}\delta_{s,s^\prime}\norm{\nu}\norm{\gT(s,s^\prime)\nu}+\alpha_n^2\lambda_s^2\norm{\gT(s,s^\prime)\nu}^2\\
       &\qquad\leq\big[(1-\alpha_n\lambda_{s})^2\delta_{s,s^\prime}+ 2(1-\alpha_n\lambda_{s})\alpha_n\lambda_{s}\delta_{s,s^\prime}\sqrt{\gamma}P(s,s^\prime)+\alpha_n^2\lambda_s^2\gamma P(s,s^\prime)^2\big]\norm{\nu}^2\\
       &\qquad=\brk{(1-\alpha_n\lambda_{s})^2\delta_{s,s^\prime}+\alpha_n  \lambda_{s}\sqrt{\gamma}P(s,s^\prime)}^2\norm{\nu}^2\\
       &\qquad\leq\brk{(1-\alpha_n\lambda_{s})+\alpha_n \lambda_{s}\sqrt{\gamma}}\brk{(1-\alpha_n\lambda_{s})\delta_{s,s^\prime}+\alpha_n \lambda_{s}\sqrt{\gamma}P(s,s^\prime)}\norm{\nu}^2\\
       &\qquad\leq\brk{(1-\alpha_n\lambda_{\min})+\alpha_n \lambda_{\min}\sqrt{\gamma}}\brk{(1-\alpha_n\lambda_{s})\delta_{s,s^\prime}+\alpha_n \lambda_{s}\sqrt{\gamma}P(s,s^\prime)}\norm{\nu}^2,  
    \end{aligned}
\end{equation*}
which is desired.
\end{proof}
\subsection{Proof of Lemma~\ref{thm_per_epoch_analysis}}\label{subsection:proof_per_epoch_analysis}
\begin{proof}
For simplicity, we use the same abbreviations as in Section~\ref{Subsection_analysis_nonasymp_ntd_w1}, and we abuse the notation to abbreviate $\mu_{\pi,\min}$ as $\mu_{\min}$.
Since we are conducting a per-epoch analysis, we will also omit $e$ in the notations.
As in \citep[Theorem~4][]{li2021sample}, we define
\begin{align*}
    \tframe:=&\frac{443\tmix}{\mu_{\min}}\log\frac{24\abs{\gS}\tepoch}{\delta},\\
    \muframe:=&\frac{1}{2}\mu_{\min}\tframe,\\
    \beta:=&\prn{1-\sqrt{\gamma}}\prn{1-\prn{1-\alpha}^{\muframe}}.
\end{align*}
\paragraph{Phase 1: when \texorpdfstring{$\norm{\bar{\bm{\eta}}-\bm{\eta}}>1$}{||bar eta-eta||>1}}
\begin{align*}
    \bm{\Delta}_t&={\bm{\eta}}_t-\bm{\eta}\\
    &=\prn{\gI-\alpha\bLamb_t}{\bm{\eta}}_{t-1}+\alpha\bLamb_t\prn{\gT_t{\bm{\eta}}_{t-1}-\gT_t\bar{\bm{\eta}}+\wtilde\gT\bar{\bm{\eta}}   }-\bm{\eta}\\
    &=\prn{\gI-\alpha\bLamb_t}\bm{\Delta}_{t-1}+\alpha\bLamb_t\prn{\gT_t\prn{{\bm{\eta}}_{t-1}-\bar{\bm{\eta}}}+\wtilde\gT\bar{\bm{\eta}}-\bm{\eta}}\\
    &=\prn{\gI-\alpha\bLamb_t}\bm{\Delta}_{t-1}+\alpha\bLamb_t\prn{\wtilde{\gT}-\gT}\bar{\bm{\eta}}+\alpha\bLamb_t\prn{\gT_t-\gT}\prn{\bm{\eta}-\bar{\bm{\eta}}}+\alpha\bLamb_t\gT_t\bm{\Delta}_{t-1}.
\end{align*}
Applying it recursively, we can further decompose the error into four terms:
\begin{align*}
    \bm{\Delta}_t=&\brc{\sum_{k=1}^t\alpha\prn{\prod_{i=k+1}^t\prn{\gI-\alpha\bLamb_i}}\bLamb_k\prn{\wtilde{\gT}-\gT}\bar{\bm{\eta}}}+\brc{\sum_{k=1}^t\alpha\prn{\prod_{i=k+1}^t\prn{\gI-\alpha\bLamb_i}}\bLamb_k\prn{\gT_k-\gT}\prn{\bm{\eta}-\bar{\bm{\eta}}}}\\
    &+\brc{\sum_{k=1}^t\alpha\prn{\prod_{i=k+1}^t\prn{\gI-\alpha\bLamb_i}}\bLamb_k\gT_k\bm{\Delta}_{k-1}}+\brc{\prod_{k=1}^t\prn{\gI-\alpha\bLamb_k}\bm{\Delta}_0}\\
    =&\text{(I)}_t+\text{(II)}_t+\text{(III)}_t+\text{(IV)}_t.
\end{align*}
For Term (I), we have the following upper bound whose proof can be found in Appendix~\ref{subsection:proof_variance_reduction_large_error_regime_first_term}.
\begin{lemma}\label{lem:variance_reduction_large_error_regime_first_term}
    For any $\delta\in(0,1)$, and $N\geq \tframe$, with probability at least $1-\frac{\delta}{3}$, we have for all $t\in[\tepoch]$, in both cases of {\VRNTD} and {\VRCTD},
\begin{equation*}
    \begin{aligned}
        &\norm{\text{(I)}_t}\leq 2\sqrt{\frac{\log\frac{12\abs{\gS}N}{\delta}}{N\mu_{\min}}}\prn{2\norm{\bar{\bm{\eta}}-\bm{\eta}} +\frac{1}{\sqrt{1-\gamma}}}.
    \end{aligned}
\end{equation*}
\end{lemma}
For Term (II), we have the following upper bound whose proof can be found in Appendix~\ref{subsection:proof_variance_reduction_large_error_regime_second_term}.
\begin{lemma}\label{lem:variance_reduction_large_error_regime_second_term}
    For any $\delta\in(0,1)$, with probability at least $1-\frac{\delta}{3}$, we have for all $t\in[\tepoch]$, in both cases of {\VRNTD} and {\VRCTD},
\begin{equation*}
    \begin{aligned}
            &\norm{\text{(II)}_t}\leq 3\sqrt{\alpha \log\frac{6\abs{\gS}\tepoch}{\delta}}\norm{\bar{\bm{\eta}}-\bm{\eta}}.
    \end{aligned}
\end{equation*}
\end{lemma}
For Term (III), we can check that
\begin{equation*}
    \begin{aligned}
            &\abs{\text{(III)}_t}\leq\sqrt{\gamma}\sum_{k=1}^t\norm{\bm{\Delta}_{k-1}}\alpha\prn{\prod_{i=k+1}^t\prn{\bI-\alpha\bLamb_i}}\bLamb_k\bm{1}.
    \end{aligned}
\end{equation*}
For Term (IV), when $t<\tframe$, we use a trivial upper bound $\norm{\text{(IV)}_t}\leq\norm{\bm{\Delta}_0}$.
And using a concentration inequality of empirical distributions of Markov chains \citep[Lemma~8][]{li2021sample}, one can show that, for any $\delta\in(0,1)$, we have with probability at least $1-\frac{\delta}{3}$, for all $\tframe\leq t\leq \tepoch$, $\norm{\text{(IV)}_t}\leq\prn{1-\alpha}^{\frac{1}{2}t\mu_{\min}}\norm{\bm{\Delta}_0}$.

To summary, we arrive at, for any $\delta\in(0,1)$, with probability at least $1-\delta$, the following results hold:
For all $t<\tframe$,
\begin{align*}
    \abs{\bDelta_t}\leq&\sqrt{\gamma}\sum_{k=1}^t\norm{\bm{\Delta}_{k-1}}\alpha\prn{\prod_{i=k+1}^t\prn{\bI-\alpha\bLamb_i}}\bLamb_k\bm{1}+2\sqrt{\frac{\log\frac{12\abs{\gS}N}{\delta}}{N\mu_{\min}}}\prn{2\norm{\bar{\bm{\eta}}-\bm{\eta}} +\frac{1}{\sqrt{1-\gamma}}}\bm{1}\\
    &+3\sqrt{\alpha \log\frac{6\abs{\gS}\tepoch}{\delta}}\norm{\bar{\bm{\eta}}-\bm{\eta}}\bm{1}+\norm{\bm{\Delta}_0}\bm{1},
\end{align*}
and for all $\tframe\leq t\leq \tepoch$, the above inequality still holds if we replace the last term $\norm{\bm{\Delta}_0}\bm{1}$ with $\prn{1-\alpha}^{\frac{1}{2}t\mu_{\min}}\norm{\bm{\Delta}_0}\bm{1}$.

To solve the recursive relation, we first have the crude upper bound whose proof is identical to \citep[Eqn.~(49)][]{li2021sample}
\begin{align*}
    \norm{\bDelta_t}\leq& \frac{1}{1-\sqrt{\gamma}}\brc{2\sqrt{\frac{\log\frac{12\abs{\gS}N}{\delta}}{N\mu_{\min}}}\prn{2\norm{\bar{\bm{\eta}}-\bm{\eta}} +\frac{1}{\sqrt{1-\gamma}}}+3\sqrt{\alpha \log\frac{6\abs{\gS}\tepoch}{\delta}}\norm{\bar{\bm{\eta}}-\bm{\eta}}+\norm{\bm{\Delta}_0}}\\
    =&\frac{1}{1-\sqrt{\gamma}}\brc{2\sqrt{\frac{\log\frac{12\abs{\gS}N}{\delta}}{(1-\gamma)N\mu_{\min}}}+\norm{\bar{\bm{\eta}}-\bm{\eta}}\prn{2\sqrt{\frac{\log\frac{12\abs{\gS}N}{\delta}}{N\mu_{\min}}}+3\sqrt{\alpha \log\frac{6\abs{\gS}\tepoch}{\delta}}}+\norm{\bm{\Delta}_0}}. 
\end{align*}
To refine the upper bound, we use the argument in \citep[Theorem~5][]{li2021sample}.
Let $\xi:=\frac{1}{8}$, $u_0:=\frac{\norm{\bDelta_0}}{1-\sqrt{\gamma}}$, $u_t:=\norm{\bm{v}_t}$, where
\begin{align*}
    \bm{v}_t=\begin{cases} 
\sqrt{\gamma}\sum_{k=1}^tu_{k-1}\alpha\prn{\prod_{i=k+1}^t\prn{\bI-\alpha\bLamb_i}}\bLamb_k\bm{1}+\norm{\Delta_0}\bm{1} & \text{for }1\leq t\leq t_{\operatorname{th},\xi}, \\
\sqrt{\gamma}\sum_{k=1}^tu_{k-1}\alpha\prn{\prod_{i=k+1}^t\prn{\bI-\alpha\bLamb_i}}\bLamb_k\bm{1} & \text{for } t> t_{\operatorname{th},\xi},
\end{cases}
\end{align*}
and $t_{\operatorname{th},\xi}$ is defined as
\begin{equation*}
    t_{\operatorname{th},\xi}:=\max\brc{\frac{2\log\frac{1}{\xi(1-\sqrt{\gamma})\sqrt{1-\gamma}}}{\alpha\mu_{\min}},\tframe}.
\end{equation*}
We can find that $\prn{1-\alpha}^{\frac{1}{2}t\mu_{\min}}\norm{\bm{\Delta}_0}\leq (1-\sqrt{\gamma})\xi$ for any $t>t_{\operatorname{th},\xi}$.
Following the proof of \citep[Lemma~3 and Lemma~4][]{li2021sample}, one can show that with probability at least $1-\delta$, for all $t\in[\tepoch]$
\begin{align*}
    \norm{\bDelta_t}\leq&u_t+L+\xi,\\
    \leq&(1-\beta)^{k}\frac{\norm{\bDelta_0}}{1-\sqrt \gamma}+L+\xi,
\end{align*}
where $k=\prn{\lfloor \frac{t- t_{\operatorname{th},\xi}}{\tframe} \rfloor}_{+}$, and
\begin{align*}
    L=\frac{1}{1-\sqrt{\gamma}}\brc{2\sqrt{\frac{\log\frac{24\abs{\gS}N}{\delta}}{(1-\gamma)N\mu_{\min}}}+\norm{\bar{\bm{\eta}}-\bm{\eta}}\prn{2\sqrt{\frac{\log\frac{24\abs{\gS}N}{\delta}}{N\mu_{\min}}}+3\sqrt{\alpha \log\frac{12\abs{\gS}\tepoch}{\delta}}}}.
\end{align*}
If we take $\alpha=\min\brc{\frac{(1-\sqrt\gamma)^2}{C\log\frac{\abs{\gS}\tepoch}{\delta}},\frac{1}{\muframe}}$, and $N\geq \max\brc{\tframe,\frac{C\log\frac{\abs{\gS}N}{\delta}}{(1-\sqrt\gamma)^2(1-\gamma)\mu_{\min}}}$ for some sufficiently large universal constant $C>0$, we have
\begin{align*}
    L\leq \frac{1}{8}+\frac{1}{8}\norm{\bar{\bm{\eta}}-\bm{\eta}}.
\end{align*}
If we take $\tepoch\geq \tframe+t_{\operatorname{th},\xi}+\frac{8\log\frac{2}{1-\sqrt\gamma}}{(1-\sqrt\gamma)\alpha \mu_{\min}}$, we have
\begin{align*}
    \prn{1-\beta}^k\leq\frac{1}{8}\prn{1-\sqrt \gamma}.
\end{align*}
Consequently,
\begin{align*}
    \norm{\bDelta_{\tepoch}}\leq&\frac{1}{8}\norm{\bDelta_0}+\prn{\frac{1}{8}+\frac{1}{8}\norm{\bar{\bm{\eta}}-\bm{\eta}}}+\frac{1}{8}\\
    \leq&\frac{1}{2}\max\brc{1,\norm{\bar{\bm{\eta}}-\bm{\eta}}},
\end{align*}
where we used $\bDelta_0=\bar{\bm{\eta}}-\bm{\eta}$.

\paragraph{Phase 2: when \texorpdfstring{$\norm{\bar{\bm{\eta}}-\bm{\eta}}\leq1$}{||bar eta-eta||<=1}}
Following the idea of \citep[Section B.2][]{wainwright2019variance}, we introduce a new operator $\widecheck{\gT}$, which is a perturbation of distributional Bellman operator $\gT$ and defined as
\begin{equation*}
    \widecheck{\gT}\prn{\bm{\xi}}:=\gT\bm{\xi}+\prn{\wtilde{\gT}-\gT}\bar{\bm{\eta}},
\end{equation*}
for any $\bm{\xi}\in\gM_1^\gS$.
And $\gM_1$ is defined as
\begin{equation*}
    \gM_1:= \left\{\nu\colon\abs{\nu}(\RB)< \infty ,\nu(\RB)=1,\text{supp}(\nu)\subseteq \left[0,\frac{1}{1{-}\gamma} \right] \right\}.
\end{equation*} 
One can find that the Cram\'er metric $\ell_2$ can be naturally extended to $\gM_1$, making $\prn{\gM_1, \ell_2}$ a separable but not complete metric space, and $\widecheck{\gT}$ is a $\sqrt{\gamma}$-contraction in the product space $\prn{\gM_1^\gS, \bar{\ell}_2}$.
To resolve the completeness problem, we will use the completion of ${\gM}_1$ instead. 
Using some basic analysis, one can show that $\widecheck{\gT}$ can be uniquely extended to be a $\sqrt{\gamma}$-contraction in the completion of ${\gM}_1^\gS$.
For simplicity, we abuse the notations to use $\gM_1$ to denote the completion
space, and $\widecheck{\gT}$ to denote the unique extension.
By the contraction mapping theorem \citep[Proposition~4.7 in ][]{bdr2022}, $\widecheck{\gT}$ admits a unique fixed point $\check{\bm{\eta}}=\prn{\check{\eta}(s)}_{s\in\gS}\in\gM_1^\gS$.
We remark that $\widecheck{\gT}$ and $\check{\bm{\eta}}$ are only used for proof.

By the triangle inequality,
\begin{align*}
    \norm{\bm{\Delta_t}}=&\norm{ {\bm{\eta}}_t-{\bm{\eta}}}\leq\norm{\check{\bm{\Delta}}_t}+\norm{\check{\bm{\eta}}-{\bm{\eta}}},
\end{align*}
where $\check{\bm{\Delta}}_t:={\bm{\eta}}_t-\check{\bm{\eta}}$.
First, we give an upper bound for the norm of $\check{\bm{\eta}}-{\bm{\eta}}\in\gM^\gS$, whose proof can be found in Appendix~\ref{subsection:proof_variance_reduction_small_error_regime_eta_hat_term}.
\begin{lemma}\label{lem:variance_reduction_small_error_regime_eta_hat_term}
    For any $\delta\in(0,1)$, when $N\geq \frac{C\log\frac{\abs{\gS}N}{\delta}}{\varepsilon^2(1-\sqrt\gamma)^2(1-\gamma)\mu_{\min}}$ with probability at least $1-\frac{\delta}{2}$, in both cases of {\VRNTD} and {\VRCTD},
\begin{equation*}
    \begin{aligned}
        &\norm{\check{\bm{\eta}}-{\bm{\eta}}}\leq 11\sqrt{\frac{\log\frac{4\abs{\gS}N}{\delta}}{(1-\sqrt\gamma)^{2}N\mu_{\min}}}\leq \frac{1}{16} \sqrt{1-\gamma}\varepsilon.
    \end{aligned}
\end{equation*}
\end{lemma}
Next, we have the following upper bound for $\norm{\check{\bm{\Delta}}_t}$, whose proof can be 
 found in Appendix~\ref{subsection:proof_variance_reduction_small_error_regime_check_delta_term}.
\begin{lemma}\label{lem:variance_reduction_small_error_regime_check_delta_term}
    For any $\delta\in(0,1)$, with probability at least $1-\frac{\delta}{2}$, in both cases of {\VRNTD} and {\VRCTD},
\begin{equation*}
    \begin{aligned}
        &\norm{\check{\bm{\Delta}}_{\tepoch}}\leq \frac{1}{4}\norm{\check{\bDelta}_0}+\frac{1}{8}\sqrt{1-\gamma}\varepsilon.
    \end{aligned}
\end{equation*}
\end{lemma}
To summary, we have
\begin{align*}
    \norm{\bm{\Delta}_{\tepoch}}\leq&\norm{\check{\bm{\Delta}}_{\tepoch}}+\norm{\check{\bm{\eta}}-{\bm{\eta}}}\\
    \leq&\frac{1}{4}\norm{\check{\bDelta}_0}+\frac{3}{16}\sqrt{1-\gamma}\varepsilon\\
    \leq&\frac{1}{4}\norm{{\bDelta}_0}+\frac{1}{4}\norm{\check{\bm{\eta}}-{\bm{\eta}}}+\frac{3}{16}\sqrt{1-\gamma}\varepsilon\\
    \leq&\frac{1}{2}\max\brc{\norm{\bar{\bm{\eta}}-\bm{\eta}},\sqrt{1-\gamma}\varepsilon},
\end{align*}
which is the desired conclusion.
\end{proof}

\subsection{Proof of Lemma~\ref{lem:variance_reduction_large_error_regime_first_term}}\label{subsection:proof_variance_reduction_large_error_regime_first_term}
\begin{proof}
We can check that
\begin{align*}
    \norm{\text{(I)}_t}\leq&\norm{\sum_{k=1}^t\alpha\prn{\prod_{i=k+1}^t\prn{\gI-\alpha\bLamb_i}}\bLamb_k}\norm{\prn{\wtilde{\gT}-\gT}\bar{\bm{\eta}}}\\
    \leq&\norm{\sum_{k=1}^t\alpha\prn{\prod_{i=k+1}^t\prn{\bI-\alpha\bLamb_i}}\bLamb_k}\norm{\prn{\wtilde{\gT}-\gT}\bar{\bm{\eta}}}\\
    \leq& \norm{\prn{\wtilde{\gT}-\gT}\bar{\bm{\eta}}},
\end{align*}
where the last inequality holds due to \citep[Eqn.~(51b)][]{li2021sample}.
We only need to bound $\norm{\prn{\wtilde{\gT}-\gT}\bar{\bm{\eta}}}$,
\begin{align*}
    \prn{\wtilde{\gT}-\gT}\bar{\bm{\eta}}=&\prn{\sum_{i=0}^{N-1}\wtilde{\bLamb}_i}^{-1}\sum_{i=0}^{N-1}\wtilde{\bLamb}_i\prn{\wtilde\gT_i-\gT} \bar{\bm{\eta}}.
\end{align*}
Note that, for each $s\in\gS$, $\brk{\prn{\wtilde\gT_i-\gT} \bar{\bm{\eta}}}(s)$ is a mean-zero $\gM$-valued random element bounded by
\begin{align*}
    \norm{\brk{\prn{\wtilde\gT_i-\gT} \bar{\bm{\eta}}}(s)}\leq&\norm{\prn{\wtilde\gT_i-\gT} \bar{\bm{\eta}}}\\
    \leq&\norm{\prn{\wtilde\gT_i-\gT} \prn{\bar{\bm{\eta}}-\bm{\eta}}}+\norm{\prn{\wtilde\gT_i-\gT}{\bm{\eta}}}\\
    \leq& 2\norm{\bar{\bm{\eta}}-\bm{\eta}} +\frac{1}{\sqrt{1-\gamma}}.
\end{align*}
We can use the argument in \citep[Lemma~7][]{li2021sample} and replace the scalar version Hoeffding's inequality with the Hilbert space version to show that, with probability at least $1-\frac{\delta}{6}$
\begin{align*}
    \norm{\prn{\wtilde{\gT}-\gT}\bar{\bm{\eta}}}\leq\max_{s\in\gS}\sqrt{\frac{2\log\frac{12\abs{\gS}N}{\delta}}{N_s}}\prn{2\norm{\bar{\bm{\eta}}-\bm{\eta}} +\frac{1}{\sqrt{1-\gamma}}},
\end{align*}
where $N_s=\sum_{i=0}^{N-1}\ind\brc{\tilde s_i=s}$.
By a concentration inequality of empirical distributions of Markov chains \citep[Lemma~8][]{li2021sample}, when $N\geq \tframe$, we have with probability at least $1-\frac{\delta}{6}$, $N_s\geq \frac{1}{2}N\mu_{\min}$ for all $s\in\gS$.
Hence, with probability at least $1-\frac{\delta}{3}$
\begin{align*}
    \norm{\prn{\wtilde{\gT}-\gT}\bar{\bm{\eta}}}\leq2\sqrt{\frac{\log\frac{12\abs{\gS}N}{\delta}}{N\mu_{\min}}}\prn{2\norm{\bar{\bm{\eta}}-\bm{\eta}} +\frac{1}{\sqrt{1-\gamma}}}.
\end{align*}
\end{proof}

\subsection{Proof of Lemma~\ref{lem:variance_reduction_large_error_regime_second_term}}\label{subsection:proof_variance_reduction_large_error_regime_second_term}
\begin{proof}
For each $s\in\gS$, let $t_k(s)$ be the time when $s$ is visited for the $k$-th time, and $K_t(s):=\max\brc{k\colon t_k(s)\leq t}$ be the total number of times that $s$ is visited, then
\begin{equation*}
    \begin{aligned}
        &\brk{\sum_{k=1}^t\alpha\prn{\prod_{i=k+1}^t\prn{\gI-\alpha\bLamb_i}}\bLamb_k\prn{\gT_k-\gT}\prn{\bm{\eta}-\bar{\bm{\eta}}}}(s)\\
        &\qquad =\sum_{k=1}^{K_t(s)}\alpha\prn{1-\alpha}^{K_t(s)-k}\prn{\gT_{t_k(s)}(s)-\gT(s)}\prn{\bm{\eta}-\bar{\bm{\eta}}},
    \end{aligned}
\end{equation*}
in which 
\begin{equation*}
    \norm{\alpha\prn{1-\alpha}^{K_t(s)-k}\prn{\gT_{t_k(s)}(s)-\gT(s)}\prn{\bm{\eta}-\bar{\bm{\eta}}}}\leq 2\alpha\prn{1-\alpha}^{K_t(s)-k}\norm{\bm{\eta}-\bar{\bm{\eta}}}.
\end{equation*}
Now, we can use the argument in \citep[Lemma~1][]{li2021sample} and replace the scalar version Hoeffding's inequality with the Hilbert space version to show that, with probability at least $1-\frac{\delta}{3}$, for all $t\in[\tepoch]$
\begin{equation*}
    \begin{aligned}
        &\norm{\sum_{k=1}^t\alpha\prn{\prod_{i=k+1}^t\prn{\gI-\alpha\bLamb_i}}\bLamb_k\prn{\gT_k-\gT}\prn{\bm{\eta}-\bar{\bm{\eta}}}}\leq 3\sqrt{\alpha \log\frac{6\abs{\gS}\tepoch}{\delta}}\norm{\bm{\eta}-\bar{\bm{\eta}}}.
    \end{aligned}
\end{equation*}
\end{proof}

\subsection{Proof of Lemma~\ref{lem:variance_reduction_small_error_regime_eta_hat_term}}\label{subsection:proof_variance_reduction_small_error_regime_eta_hat_term}
\begin{proof}
Using the fixed point property, we have
\begin{equation*}
    \begin{aligned}
        \abs{\check{\bm{\eta}}-{\bm{\eta}}}=&\abs{\gT\check{\bm{\eta}}+\prn{\wtilde{\gT}-\gT}\bar{\bm{\eta}}-\gT{\bm{\eta}}}\\
        \leq&\abs{\gT\prn{\check{\bm{\eta}}-{\bm{\eta}}}}+\abs{\prn{\wtilde{\gT}-\gT}\bar{\bm{\eta}}}\\
        \leq& \sqrt{\gamma}\bP \abs{\check{\bm{\eta}}-{\bm{\eta}}}+\abs{\prn{\wtilde{\gT}-\gT}\bar{\bm{\eta}}}\\
        \leq& \sum_{k=0}^\infty \prn{\sqrt{\gamma}\bP}^k\abs{\prn{\wtilde{\gT}-\gT}\bar{\bm{\eta}}}\\
        =&\prn{\bI-\sqrt{\gamma}\bP}^{-1}\abs{\prn{\wtilde{\gT}-\gT}\bar{\bm{\eta}}}.
    \end{aligned}
\end{equation*}
Hence, we only need to deal with $\abs{\prn{\wtilde{\gT}-\gT}\bar{\bm{\eta}}}$.
Using the triangle inequality, 
\begin{equation*}
    \begin{aligned}
       \abs{\prn{\wtilde{\gT}-\gT}\bar{\bm{\eta}}}\leq\abs{\prn{\wtilde{\gT}-\gT}\prn{\bar{\bm{\eta}}-\bm{\eta}}}+\abs{\prn{\wtilde{\gT}-\gT}{\bm{\eta}}},
    \end{aligned}
\end{equation*}
we need bound the above two terms.
Following the proof of Lemma~\ref{lem:variance_reduction_large_error_regime_first_term} in Appendix~\ref{subsection:proof_variance_reduction_large_error_regime_first_term}, we have the following upper bound for the first term, with probability at least $1-\frac{\delta}{2}$
\begin{equation*}
    \begin{aligned}
       \norm{\prn{\wtilde{\gT}-\gT}\prn{\bar{\bm{\eta}}-\bm{\eta}}}\leq& 2\sqrt{\frac{\log\frac{4\abs{\gS}N}{\delta}}{N\mu_{\min}}}\norm{\bar{\bm{\eta}}-\bm{\eta}}\leq 2\sqrt{\frac{\log\frac{4\abs{\gS}N}{\delta}}{N\mu_{\min}}},
    \end{aligned}
\end{equation*}
where we used the assumption that $\norm{\bar{\bm{\eta}}-\bm{\eta}}\leq 1$ holds in the Phase 2.
As for the second term, one can use the argument in \citep[Eqn.~(94-96)][]{li2021sample} and replace the scalar version Bernstein’s
inequality with the Hilbert space version to show that,  with probability at least $1-\frac{\delta}{2}$
\begin{equation*}
    \begin{aligned}
       \abs{\prn{\wtilde{\gT}-\gT}{\bm{\eta}}}\leq 2\sqrt{\frac{2\log\frac{4\abs{\gS}N}{\delta}}{N\mu_{\min}}\bm{\sigma}\prn{\bm{\eta}}}+\frac{8\log\frac{4\abs{\gS}N}{\delta}}{3\sqrt{1-\gamma}N\mu_{\min}}.
    \end{aligned}
\end{equation*}
To summarize, we have
\begin{equation*}
    \begin{aligned}
        \abs{\check{\bm{\eta}}-{\bm{\eta}}}\leq&\prn{\bI-\sqrt{\gamma}\bP}^{-1}\abs{\prn{\wtilde{\gT}-\gT}\bar{\bm{\eta}}}\\
        \leq&\prn{\bI-\sqrt{\gamma}\bP}^{-1}\prn{\norm{\prn{\wtilde{\gT}-\gT}\prn{\bar{\bm{\eta}}-\bm{\eta}}}\bm{1}+\abs{\prn{\wtilde{\gT}-\gT}{\bm{\eta}}}}\\
        \leq&2\sqrt{\frac{\log\frac{4\abs{\gS}N}{\delta}}{(1-\sqrt{\gamma})^2N\mu_{\min}}}+\frac{8\log\frac{4\abs{\gS}N}{\delta}}{3(1-\sqrt{\gamma})\sqrt{1-\gamma}N\mu_{\min}}+2\prn{\bI-\sqrt{\gamma}\bP}^{-1}\sqrt{\frac{2\log\frac{4\abs{\gS}N}{\delta}}{N\mu_{\min}}\bm{\sigma}\prn{\bm{\eta}}}\\
        \leq&5\sqrt{\frac{\log\frac{4\abs{\gS}N}{\delta}}{(1-\sqrt{\gamma})^2N\mu_{\min}}}+2\sqrt{\frac{2\log\frac{4\abs{\gS}N}{\delta}}{(1-\sqrt{\gamma})N\mu_{\min}}\prn{\bI-\sqrt{\gamma}\bP}^{-1}\bm{\sigma}\prn{\bm{\eta}}}\\
        \leq& 11 \sqrt{\frac{\log\frac{4\abs{\gS}N}{\delta}}{(1-\sqrt\gamma)^2N\mu_{\min}}},
    \end{aligned}
    \end{equation*}
where in the fourth inequality, we used $N\geq \frac{C\log\frac{\abs{\gS}N}{\delta}}{(1-\sqrt\gamma)^2(1-\gamma)\mu_{\min}}$ and Jensen's inequality; and in the last inequality, we used Lemma~\ref{lem:sigma_fine_upper_bound}.
And the conclusion follows if we substitute $N$ into the upper bound.
\end{proof}

\subsection{Proof of Lemma~\ref{lem:variance_reduction_small_error_regime_check_delta_term}}\label{subsection:proof_variance_reduction_small_error_regime_check_delta_term}
\begin{proof}
For any $t\in[\tepoch]$,
\begin{align*}
    \check{\bm{\Delta}}_t&={\bm{\eta}}_t-\check{\bm{\eta}}\\
    &=\prn{\gI-\alpha\bLamb_t}{\check{\bm{\Delta}}}_{t-1}+\alpha\bLamb_t\prn{\gT_t{\bm{\eta}}_{t-1}-\gT_t\bar{\bm{\eta}}+\wtilde\gT\bar{\bm{\eta}} - \check{\bm{\eta} }}\\
        &=\prn{\gI-\alpha\bLamb_t}{\check{\bm{\Delta}}}_{t-1}+\alpha\bLamb_t\prn{\gT_t{\bm{\eta}}_{t-1}-\gT_t\bar{\bm{\eta}}+\wtilde\gT\bar{\bm{\eta}}-\gT\check{\bm{\eta}}-\wtilde{\gT}\bar{\bm{\eta}}+\gT\bar{\bm{\eta}}}\\
    &=\prn{\gI-\alpha\bLamb_t}{\check{\bm{\Delta}}}_{t-1}+\alpha\bLamb_t\prn{\gT_t\prn{{\bm{\eta}}_{t-1}-\bar{\bm{\eta}}}-\gT\prn{\check{\bm{\eta}}-\bar{\bm{\eta}}}}\\
    &=\prn{\gI-\alpha\bLamb_t}{\check{\bm{\Delta}}}_{t-1}+\alpha\bLamb_t\prn{\prn{\gT_t-\gT}\prn{\check{\bm{\eta}}-\bar{\bm{\eta}}}+\gT_t\check{\bm{\Delta}}_{t-1}}.
\end{align*}
Applying it recursively, we can further decompose the error into three terms:
\begin{align*}
    \check{\bm{\Delta}}_t=&\brc{\sum_{k=1}^t\alpha\prn{\prod_{i=k+1}^t\prn{\gI-\alpha\bLamb_i}}\bLamb_k\prn{\gT_k-\gT}\prn{\check{\bm{\eta}}-\bar{\bm{\eta}}}}\\
    &+\brc{\sum_{k=1}^t\alpha\prn{\prod_{i=k+1}^t\prn{\gI-\alpha\bLamb_i}}\bLamb_k\gT_k\check{\bm{\Delta}}_{k-1}}+\brc{\prod_{k=1}^t\prn{\gI-\alpha\bLamb_k}\check{\bm{\Delta}}_0}\\
    =&\text{(I)}_t+\text{(II)}_t+\text{(III)}_t.
\end{align*}
For Term (I), following the proof of Lemma~\ref{lem:variance_reduction_large_error_regime_second_term} in Appendix~\ref{subsection:proof_variance_reduction_large_error_regime_second_term}, we have with probability at least $1-\frac{\delta}{2}$
\begin{equation*}
    \begin{aligned}
            &\norm{\text{(I)}_t}\leq 3\sqrt{\alpha \log\frac{4\abs{\gS}\tepoch}{\delta}}\norm{\check{\bm{\eta}}-\bar{\bm{\eta}}}.
    \end{aligned}
\end{equation*}
For Term (II), we can check that
\begin{equation*}
    \begin{aligned}
            &\abs{\text{(II)}_t}\leq\sqrt{\gamma}\sum_{k=1}^t\norm{\check{\bm{\Delta}}_{k-1}}\alpha\prn{\prod_{i=k+1}^t\prn{\bI-\alpha\bLamb_i}}\bLamb_k\bm{1}.
    \end{aligned}
\end{equation*}
For Term (III), when $t<\tframe$, we use a trivial upper bound $\norm{\text{(III)}_t}\leq\norm{\check{\bm{\Delta}}_0}$.
And using a concentration inequality of empirical distributions of Markov chains \citep[Lemma~8][]{li2021sample}, one can show that, for any $\delta\in(0,1)$, we have with probability at least $1-\frac{\delta}{2}$, for all $\tframe\leq t\leq \tepoch$, $\norm{\text{(III)}_t}\leq\prn{1-\alpha}^{\frac{1}{2}t\mu_{\min}}\norm{\check{\bm{\Delta}}_0}$.

Following the proof in Phase 1 with $\xi$ replaced with $\frac{1}{8}\sqrt{1-\gamma}\varepsilon$, we have with probability at least $1-\delta$, for all $t\in[\tepoch]$
\begin{align*}
    \norm{\check{\bDelta}_t}\leq&(1-\beta)^{k}\frac{\norm{\check{\bDelta}_0}}{1-\sqrt \gamma}+\frac{3}{1-\sqrt{\gamma}}\sqrt{\alpha \log\frac{8\abs{\gS}\tepoch}{\delta}}\norm{\check{\bm{\eta}}-\bar{\bm{\eta}}}+\frac{1}{8}\sqrt{1-\gamma}\varepsilon,
\end{align*}
where $k=\prn{\lfloor \frac{t- t_{\operatorname{th},\xi}}{\tframe} \rfloor}_{+}$, and $t_{\operatorname{th},\xi}$ is defined as
\begin{equation*}
    t_{\operatorname{th},\xi}:=\max\brc{\frac{2\log\frac{4}{\xi(1-\sqrt{\gamma})\sqrt{1-\gamma}}}{\alpha\mu_{\min}},\tframe}.
\end{equation*}
If we take $\alpha=c\min\brc{\frac{(1-\sqrt\gamma)^2}{\log\frac{\abs{\gS}\tepoch}{\delta}},\frac{1}{\muframe}}$ for some sufficiently small universal constant $c>0$, we have
\begin{align*}
    \frac{3}{1-\sqrt{\gamma}}\sqrt{\alpha \log\frac{8\abs{\gS}\tepoch}{\delta}}\norm{\check{\bm{\eta}}-\bar{\bm{\eta}}}\leq\frac{1}{8}\norm{\check{\bm{\eta}}-\bar{\bm{\eta}}}.
\end{align*}
And if we take $\tepoch\geq \tframe+t_{\operatorname{th},\xi}+\frac{C^\prime\log\frac{1}{1-\sqrt\gamma}}{(1-\sqrt\gamma)\alpha \mu_{\min}}$ for some sufficiently large universal constant $C^\prime>0$, we have
\begin{align*}
    \prn{1-\beta}^k\leq\frac{1}{8}\prn{1-\sqrt \gamma}.
\end{align*}
Consequently,
\begin{align*}
    \norm{\check{\bDelta}_{\tepoch}}\leq&\frac{1}{8}\norm{\check{\bDelta}_0}+\frac{1}{8}\norm{\check{\bm{\eta}}-\bar{\bm{\eta}}}+\frac{1}{8}\sqrt{1-\gamma}\varepsilon\\
    \leq&\frac{1}{4}\norm{\check{\bDelta}_0}+\frac{1}{8}\sqrt{1-\gamma}\varepsilon,
\end{align*}
where we used $\check{\bDelta}_0=\bar{\bm{\eta}}-\check{\bm{\eta}}$.
\end{proof}

\section{Stochastic Distributional Bellman Equation and Operator}
In this section, we use the same notations as in Appendix~\ref{Appendix_proof} and only consider the {\NTD} setting.
Inspired by stochastic categorical CDF Bellman operator introduced in \citep{rowland2024near} used in the categorical setting, we introduce stochastic distributional Bellman operator $\sT\colon\Delta\prn{\sP^\gS}\to\Delta\prn{\sP^\gS}$ to derive an upper bound for $\norm{(\bI-\gamma \bP)^{-1}\bm{\sigma}({\bm{\eta}})}$ in the case of {\NTD}.
For any $\bm{\phi}\in\Delta\prn{\sP^\gS}$, we denote $\bm{\eta}_{\bm{\phi}}$ be the random element in $\sP^\gS$ with law $\bm{\phi}$.  
\begin{equation*}
    \sT\bm{\phi}:=\operatorname{Law}\prn{\what{\gT}\bm{\eta}_{\bm{\phi}}},
\end{equation*}
where $(\what{\gT}\bm{\eta}_{\bm{\phi}})(\omega):=(\what{\gT})(\omega)(\bm{\eta}_{\bm{\phi}})(\omega)\in\sP^\gS$ for any $\omega\in\Omega$, $\Omega$ is the corresponding probability space, and $\what{\gT}$ is independent of $\bm{\eta}_{\bm{\phi}}$.
In this part, $\what{\gT}$ does not consist of $\Pi_K$ since we only consider the {\NTD} setting.

We consider the $1$-Wasserstein metric $W_1$ on $\Delta\prn{\sP^\gS}$, the space of all probability measures on the space $\prn{\sP^\gS,\bar{\ell}_2}$.
Since $\prn{\sP^\gS,\bar{\ell}_2}$ is Polish (complete and separable), the space $\prn{\Delta\prn{\sP^\gS},W_1}$ is also Polish \citep[see Theorem 6.18 in][]{villani2009optimal}.

\begin{proposition}\label{prop:SDBO_contraction}
The stochastic distributional Bellman operator $\sT$ is a $\sqrt{\gamma}$-contraction on $\Delta\prn{\sP^\gS}$, \ie, for any $\bm{\phi},\bm{\phi}^\prime\in\Delta\prn{\sP^\gS}$, we have
\begin{equation*}
    W_1\prn{\sT\bm{\phi},\sT\bm{\phi}^\prime}\leq\sqrt{\gamma} W_1\prn{\bm{\phi},\bm{\phi}^\prime}.
\end{equation*}
\end{proposition}
\begin{proof}
    Let $\bm{\kappa}^\star\in\Gamma\prn{\bm{\phi},\bm{\phi}^\prime}$ be the optimal coupling between $\bm{\phi}$ and $\bm{\phi}^\prime$.
    The existence of $\bm{\kappa}^\star$ is guaranteed by \citep[Theorem 4.1][]{villani2009optimal}.
    And let the random element $\bm{\xi}=(\bm{\xi}_1,\bm{\xi}_2)$ in $\prn{\sP^\gS}^2$ has the law $\bm{\kappa}^\star$, where $\bm{\xi}_1$ and $\bm{\xi}_2$ are both random elements in $\sP^\gS$.
    We denote $\sT\bm{\kappa}^\star:=\operatorname{Law}\brk{\prn{\what{\gT}\bm{\xi}_1,\what{\gT}\bm{\xi}_2}}\in\Gamma\prn{\sT\bm{\phi},\sT\bm{\phi}^\prime}$.
    \begin{equation*}
        \begin{aligned}
            W_1\prn{\sT\bm{\phi},\sT\bm{\phi}^\prime}&=\inf_{\bm{\kappa}\in\Gamma\prn{\sT\bm{\phi},\sT\bm{\phi}^\prime}}\int_{\prn{\sP^\gS}^2}\bar{\ell}_2\prn{\xi,\xi^\prime}\bm{\kappa}\prn{d\xi,d\xi^\prime}    \\
            &\leq \int_{\prn{\sP^\gS}^2}\bar{\ell}_2\prn{\xi,\xi^\prime}\sT\bm{\kappa}^\star\prn{d\xi,d\xi^\prime}\\
            &= \EB\brk{ \bar{\ell}_2\prn{\what{\gT}\bm{\xi}_1 ,\what{\gT}\bm{\xi}_2  }  }\\
            &\leq\sqrt{\gamma}\EB\brk{ \bar{\ell}_2\prn{\bm{\xi}_1 ,\bm{\xi}_2  }  }\\
            &=\sqrt{\gamma}\int_{\prn{\sP^\gS}^2}\bar{\ell}_2\prn{\xi,\xi^\prime}\bm{\kappa}^\star\prn{d\xi,d\xi^\prime}\\
            &=\sqrt{\gamma}\inf_{\bm{\kappa}\in\Gamma\prn{\bm{\phi},\bm{\phi}^\prime}}\int_{\prn{\sP^\gS}^2}\bar{\ell}_2\prn{\xi,\xi^\prime}\bm{\kappa}\prn{d\xi,d\xi^\prime}\\
            &=\sqrt{\gamma}W_1\prn{\bm{\phi},\bm{\phi}^\prime}.
        \end{aligned}
    \end{equation*}
\end{proof}
By the proposition and contraction mapping theorem, there exists a unique fixed point of $\sT$, we denote $\bm{\psi}\in\Delta\prn{\sP^\gS}$ as the fixed point.
Hence, the stochastic distributional Bellman equation reads
\begin{equation*}
    \bm{\psi}=\sT\bm{\psi}.
\end{equation*}
We denote by $\bm{\eta}_{\bm{\psi}}$  the random element in $\sP$ with law $\bm{\psi}$, then $\what{\gT}\bm{\eta}_{\bm{\psi}}$ and $\bm{\eta}_{\bm{\psi}}$ have the same law.
As shown in the following proposition, $\bm{\eta}_{\bm{\psi}}$ can be regarded as a noisy version of $\bm{\eta}$.
\begin{proposition} It holds that
\[
\EB\brk{\bm{\eta}_{\bm{\psi}}}=\bm{\eta},
\]
where the expectation is regarded as the Bochner integral in the space of all finite measures on $\sP^\gS$, which is a normed linear space equipped with the Cram\'er metric as its norm.
\end{proposition}
\begin{proof}
    \begin{equation*}
        \begin{aligned}
\EB\brk{\bm{\eta}_{\bm{\psi}}}&=\EB\brk{\what\gT\bm{\eta}_{\bm{\psi}}}=\EB\brc{\EB\brk{\what\gT\bm{\eta}_{\bm{\psi}}\Big|\bm{\eta}_{\bm{\psi}}}}=\EB\brk{\gT\bm{\eta}_{\bm{\psi}}}=\gT\EB\brk{\bm{\eta}_{\bm{\psi}}},
        \end{aligned}
    \end{equation*}
    where we used $\what{\gT}$ is independent of $\bm{\eta}_{\bm{\psi}}$.
    Since $\EB\brk{\bm{\eta}_{\bm{\psi}}}$ is the fixed point of $\gT$, we have $\EB\brk{\bm{\eta}_{\bm{\psi}}}=\bm{\eta}$.
\end{proof}
Based on the concepts of $\sT$ and $\bm{\psi}$, we can obtain the following second order distributional Bellman equation, which is similar to the classic second-order Bellman equation \citep[see Lemma 7 in][]{gheshlaghi2013minimax}.

Recall the one-step Cram\'er variation $\bm{\sigma}({\bm{\eta}})=\prn{\EB\brk{\norm{\prn{\what{\gT}{\bm{\eta}}}(s)-\eta(s) }^2}}_{s\in\gS}\in\RB^\gS$ used in the {\NTD} setting.
We denote $\bm{\sigma}:=\bm{\sigma}({\bm{\eta}})$ for simplicity, and $\bm{\Sigma}:=\prn{\EB\brk{\norm{\bm{\eta}_{\bm{\psi}}(s)-\eta(s)}^2}}_{s\in\gS}\in\RB^\gS$.
\begin{proposition}[Second order distributional Bellman equation]\label{prop:second_order_dist_Bellman_eq} The following relationship is satisfied: 
\begin{equation*}
\bm{\Sigma}=\bm{\sigma}+\gamma \bP \bm{\Sigma}.
\end{equation*}
\end{proposition}
\begin{proof}
For any $s\in\gS$,
    \begin{equation}\label{eq:expansion_Sigma}
        \begin{aligned}
\Sigma(s)=&\EB\brk{\norm{\bm{\eta}_{\bm{\psi}}(s)-\eta(s)}^2}\\
=&\EB\brk{\norm{\prn{\what{\gT}\bm{\eta}_{\bm{\psi}}}(s)-\eta(s)}^2}\\
=&\EB\brk{\norm{\prn{\what{\gT}\bm{\eta}_{\bm{\psi}}}(s)-\prn{\what{\gT}{\bm{\eta}}}(s)+\prn{\what{\gT}{\bm{\eta}}}(s)-\eta(s)}^2}\\
=&\EB\brk{\norm{\prn{\what{\gT}{\bm{\eta}}}(s)  {-}  \eta(s)}^2}  {+}  \EB\brk{\norm{\prn{\what{\gT}\bm{\eta}_{\bm{\psi}}}(s)  {-}  \prn{\what{\gT}{\bm{\eta}}}(s)}^2},
        \end{aligned}
    \end{equation}
where the last equality holds since the cross term is zero as below
\begin{equation*}
    \begin{aligned}
        &\EB\brk{\inner{\prn{\what{\gT}{\bm{\eta}}}(s)-\eta(s)}{\prn{\what{\gT}\bm{\eta}_{\bm{\psi}}}(s)-\prn{\what{\gT}{\bm{\eta}}}(s)}}\\
        &\qquad=\EB\brc{\EB\brk{\inner{\prn{\what{\gT}{\bm{\eta}}}(s)-\eta(s)}{\prn{\what{\gT}\bm{\eta}_{\bm{\psi}}}(s)-\prn{\what{\gT}{\bm{\eta}}}(s)}\Big|\what{\gT}}}\\
        &\qquad=\EB\brc{\inner{\prn{\what{\gT}{\bm{\eta}}}(s)-\eta(s)}{\EB\brk{\prn{\what{\gT}\bm{\eta}_{\bm{\psi}}}(s)\Big|\what{\gT}}-\prn{\what{\gT}{\bm{\eta}}}(s)}}\\
        &\qquad=\EB\brk{\inner{\prn{\what{\gT}{\bm{\eta}}}(s)-\eta(s)}{\bm{0}}}\\
        &\qquad=0.
    \end{aligned}
\end{equation*}
The first term in the last line of \eqref{eq:expansion_Sigma} is $\bm{\sigma}(s)$, we need to deal with the second term.
\begin{small}
\begin{equation*}
    \begin{aligned}
        &\EB\brk{\norm{\prn{\what{\gT}\bm{\eta}_{\bm{\psi}}}(s)-\prn{\what{\gT}{\bm{\eta}}}(s)}^2}\\
        &\qquad=\EB\brc{\EB\brk{\norm{\prn{\what{\gT}\bm{\eta}_{\bm{\psi}}}(s)-\prn{\what{\gT}{\bm{\eta}}}(s)}^2\Big|\bm{\eta}_{\bm{\psi}}}}\\
        &\qquad=\EB\Bigg\{\EB_{a(s)\sim\pi(\cdot|s),s^\prime(s)\sim P(\cdot|s,a(s)),r(s)\sim\gP_R(\cdot|s,a(s))}\brk{\int_0^{\frac{1}{1 {-} \gamma}} \prn{  F_{\prn{\bm{\eta}_{\bm{\psi}}}(s^\prime(s))}\prn{\frac{x {-} r}{\gamma}} - F_{\eta(s^\prime(s))}\prn{\frac{x-r}{\gamma}} }^2  dx\Bigg|\bm{\eta}_{\bm{\psi}}}\Bigg\}\\
        &\qquad=\gamma\sum_{s^\prime\in\gS}\EB\brk{\norm{\bm{\eta}_{\bm{\psi}}(s^\prime)-\eta(s^\prime)}^2}\sum_{a\in\gA}\pi(a|s)P(s^\prime|s,a)\\
        &\qquad=\gamma \sum_{s^\prime\in\gS}P(s,s^\prime){\Sigma}(s^\prime).
    \end{aligned}
\end{equation*}
\end{small}
Put these together, and we can arrive at the conclusion.
\end{proof}

Now, we can derive a tighter upper bound for $\norm{(\bI-\gamma\bP)^{-1}\bm{\sigma}({\bm{\eta}})}$.
\begin{corollary}\label{corollary:tight_sigma_upper_bound} We have that
    \begin{equation*}
        \norm{(\bI-\gamma\bP)^{-1}\bm{\sigma}(\eta)}=\norm{\bm{\Sigma}}\leq\frac{1}{1-\gamma}.
    \end{equation*}
\end{corollary}
\begin{proof}
Note that ${\Sigma}(s)=\EB\brk{\norm{\bm{\eta}_{\bm{\psi}}(s)-\eta(s)}^2}\leq\int_0^{\frac{1}{1-\gamma}}dx=\frac{1}{1-\gamma}$ for any $s\in\gS$.
\end{proof}

\section{Other Technical Lemmas}
\begin{lemma}[Basic Inequalities for Metrics on the Space of Probability Measures]
\label{lem:prob_basic_inequalities}
For any $\nu_1, \nu_2\in\sP$, we have $\prn{\ell_2(\nu_1,\nu_2)}^2\leq W_1(\nu_1,\nu_2)\leq\frac{1}{\sqrt{1-\gamma}}\ell_2(\nu_1,\nu_2)$ and $W_p(\nu_1,\nu_2)\leq\frac{1}{(1-\gamma)^{1-\frac{1}{p}}}W_1^{\frac{1}{p}}(\nu_1,\nu_2)$. 
\end{lemma}
\begin{proof}
By Cauchy-Schwarz inequality,
\begin{equation*}
    \begin{aligned}
        W_1(\nu_1,\nu_2)=&\int_0^{\frac{1}{1-\gamma}} |F_{\nu_1}(x)-F_{\nu_2}(x)| dx\\
        \leq&\sqrt{\int_0^{\frac{1}{1-\gamma}} 1^2 dx}\sqrt{\int_0^{\frac{1}{1-\gamma}} |F_{\nu_1}(x)-F_{\nu_2}(x)|^2 dx}\\
        =&\frac{1}{\sqrt{1-\gamma}}\ell_2(\nu_1,\nu_2).
    \end{aligned}
\end{equation*}
And
\begin{equation*}
    \begin{aligned}
        \ell_2(\nu_1,\nu_2)=&\sqrt{\int_0^{\frac{1}{1-\gamma}} |F_{\nu_1}(x)-F_{\nu_2}(x)|^2 dx}\\
        \leq&\sqrt{\int_0^{\frac{1}{1-\gamma}} |F_{\nu_1}(x)-F_{\nu_2}(x)| dx}\\
        =&\sqrt{W_1(\nu_1,\nu_2)}.
    \end{aligned}
\end{equation*}
And
\begin{equation*}
    \begin{aligned}
        W_p(\nu_1,\nu_2)=&\left(\inf _{\kappa \in \Gamma(\nu_1, \nu_2)} \int_{\brk{0,\frac{1}{1-\gamma}}^2}\abs{x-y}^p \kappa(dx, dy)\right)^{1 / p}\\
        \leq& \frac{1}{(1-\gamma)^{1-\frac{1}{p}}}\left(\inf _{\kappa \in \Gamma(\nu_1, \nu_2)} \int_{\brk{0,\frac{1}{1-\gamma}}^2}\abs{x-y} \kappa(dx, dy)\right)^{1 / p}\\
        =&\frac{1}{(1-\gamma)^{1-\frac{1}{p}}}W_1^{\frac{1}{p}}(\nu_1,\nu_2).
    \end{aligned}
\end{equation*}
\end{proof}

\begin{lemma}\label{lem:separable}
$\prn{\gM, \norm{\cdot}_{\ell_2}}$ and $\prn{\gM, \norm{\cdot}_{W_1}}$ are separable.
\end{lemma}
\begin{proof}
Recall that $\prn{\sP, W_1}$ is separable \citep[Theorem~6.18][]{villani2009optimal}, and by Lemma~\ref{lem:prob_basic_inequalities}, the Cram\'er distance $\ell_2$ can be bounded by $1$-Wasserstein distance $W_1$, hence $\prn{\sP, \ell_2}$ is also separable.
Let $d$ be either $W_1$ or $\ell_2$, and $A$ be the countable dense subset of $\prn{\sP, d}$. 
For any $\epsilon>0,\mu\in\gM$, we denote $\tilde{\mu}:=\frac{2\mu}{\abs{\mu}(\RB)}$.
Then we can find a $q\in\QB, \bar{\mu}_+,\bar{\mu}_-\in A, \st \abs{q-\frac{1}{2}\abs{\mu}(\RB)}\leq\frac{\epsilon\abs{\mu}(\RB)}{6\norm{\mu}_d}$, $d(\bar{\mu}_+,\tilde{\mu}_+)\leq\frac{\epsilon}{3q},d(\bar{\mu}_-,\tilde{\mu}_-)\leq\frac{\epsilon}{3q}$.
Note that the set of all possible $\bar{\mu}=\bar{\mu}_+-\bar{\mu}_-$ is countable.
Let $\hat{\mu}:= q\bar{\mu}$, then we have
\begin{equation*}
    \begin{aligned}
        \norm{\mu-\hat{\mu}}_{d}\leq& \norm{\mu-q\tilde{\mu}}_{d}+\norm{q\tilde{\mu}-\hat{\mu}}_{d}\\
        =&\norm{(\frac{1}{2}\abs{\mu}(\RB)-q)\tilde{\mu}}_{d}+q\norm{(\tilde{\mu}_+-\bar{\mu}_+)-(\tilde{\mu}_--\bar{\mu}_-)}_{d}\\
        \leq& \abs{q-\frac{1}{2}\abs{\mu}(\RB)}\frac{2\norm{\mu}_{d}}{\abs{\mu}(\RB)}+q\norm{\tilde{\mu}_+-\bar{\mu}_+}_{d}+q\norm{\tilde{\mu}_--\bar{\mu}_-}_{d}\\
        \leq&\frac{\epsilon\abs{\mu}(\RB)}{6\norm{\mu}_{d}}\frac{2\norm{\mu}_{d}}{\abs{\mu}(\RB)}+q\frac{\epsilon}{3q}+q\frac{\epsilon}{3q}\\
        =&\epsilon.
    \end{aligned}
\end{equation*}
Hence we have found a countable dense subset for $\prn{\gM, \norm{\cdot}_d}$ for $d=W_1$ or $\ell_2$.
Therefore, $\prn{\gM, \norm{\cdot}_{\ell_2}}$ and $\prn{\gM, \norm{\cdot}_{W_1}}$ are separable.
\end{proof}

\begin{lemma}\label{lem:incomplete}
$\prn{\gM, \norm{\cdot}_{W_1}}$ is not complete.
\end{lemma}
\begin{proof}
Consider the Cauchy sequence in $\prn{\gM, \norm{\cdot}_{W_1}}$
$\mu_n= \sum_{i=1}^n\left(\delta_{\frac{1}{i}+\frac{1}{2^i}}- \delta_{\frac{1}{i}-\frac{1}{2^i}}\right)$, which satisfies $\norm{\mu_n-\mu_{n+k}}_{W_1}\leq\sum_{i=1}^k \frac{1}{2^{n+i-1}}\leq\frac{1}{2^{n-1}}\to 0$, but its limit $\sum_{i=1}^\infty \left(\delta_{\frac{1}{i}+\frac{1}{2^i}}- \delta_{\frac{1}{i}-\frac{1}{2^i}} \right)\notin \gM$ because its total variation is infinity.
Hence, this is a Cauchy sequence without a limit in $\prn{\gM, \norm{\cdot}_{W_1}}$, implying $\prn{\gM, \norm{\cdot}_{W_1}}$ is not complete.
\end{proof}

\bibliography{ref}
\bibliographystyle{abbrvnat}
\end{document}